\documentclass{article}
\usepackage{microtype}
\usepackage{graphicx}
\usepackage{subcaption}
\usepackage{booktabs} 
\usepackage{enumitem}

\def\est{\mathrm{est}}
\def\E{\mathbb{E}}

\usepackage{hyperref}

\usepackage{bbm}

\usepackage{pgffor}
\foreach \x in {A,...,Z}{
	\expandafter\xdef\csname m\x\endcsname{\noexpand\mathbf{\x}}
}
\foreach \x in {A,...,Z}{
	\expandafter\xdef\csname om\x\endcsname{\noexpand\overline{\noexpand\mathbf{\x}}}
}
\foreach \x in {A,...,Z}{
	\expandafter\xdef\csname c\x\endcsname{\noexpand\mathcal{\x}}
}

\usepackage[noend]{algorithmic}

\usepackage[preprint]{icml2026}

\usepackage{amsmath}
\usepackage{amssymb}
\usepackage{mathtools}
\usepackage{amsthm}

\usepackage[capitalize,noabbrev]{cleveref}

\theoremstyle{plain}
\newtheorem{theorem}{Theorem}[section]
\newtheorem{proposition}[theorem]{Proposition}
\newtheorem{lemma}[theorem]{Lemma}
\newtheorem{corollary}[theorem]{Corollary}

\theoremstyle{definition}
\newtheorem{definition}[theorem]{Definition}
\newtheorem{assumption}[theorem]{Assumption}
\theoremstyle{definition}
\newtheorem{remark}[theorem]{Remark}

\usepackage[textsize=tiny]{todonotes}

\makeatletter

\DeclareRobustCommand\widecheck[1]{{\mathpalette\@widecheck{#1}}}
\def\@widecheck#1#2{
    \setbox\z@\hbox{\m@th$#1#2$}
    \setbox\tw@\hbox{\m@th$#1
       \widehat{
          \vrule\@width\z@\@height\ht\z@
          \vrule\@height\z@\@width\wd\z@}$}
    \dp\tw@-\ht\z@
    \@tempdima\ht\z@ \advance\@tempdima2\ht\tw@ \divide\@tempdima\thr@@
    \setbox\tw@\hbox{
       \raise\@tempdima\hbox{\scalebox{1}[-1]{\lower\@tempdima\box
\tw@}}}
    {\ooalign{\box\tw@ \cr \box\z@}}}

\icmltitlerunning{Graphon Mean-Field Subsampling for Cooperative Heterogeneous Multi-Agent Reinforcement Learning}

\begin{document}
    
\twocolumn[
  \icmltitle{Graphon Mean-Field Subsampling for Cooperative Heterogeneous Multi-Agent Reinforcement Learning}

  \icmlsetsymbol{equal}{*}

  \begin{icmlauthorlist}
    \icmlauthor{Emile Anand}{gatech}
    \icmlauthor{Richard Hoffmann}{caltech}
    \icmlauthor{Sarah Liaw}{harvard}
    \icmlauthor{Adam Wierman}{caltech}
  \end{icmlauthorlist}

  \icmlaffiliation{caltech}{Computing and Mathematical Sciences, California Institute of Technology, Pasadena, CA, USA}
  \icmlaffiliation{harvard}{Kempner Institute, Harvard University, Cambridge, MA, USA}
  \icmlaffiliation{gatech}{School of Computer Science, Georgia Institute of Technology, Atlanta, GA, USA}

  \icmlcorrespondingauthor{Emile Anand}{emiletimothy@outlook.com}

  \icmlkeywords{Machine Learning, ICML}

  \vskip 0.3in
]
 \printAffiliationsAndNotice{}

\begin{abstract}
    Coordinating large populations of interacting agents is a central challenge in multi-agent reinforcement learning (MARL), where the size of the joint state-action space scales exponentially with the number of agents. Mean-field methods alleviate this burden by aggregating agent interactions, but these approaches assume homogeneous interactions.  Recent graphon-based frameworks capture heterogeneity, but are computationally expensive as the number of agents grows. Therefore, we introduce \texttt{GMFS}, a \textbf{G}raphon \textbf{M}ean-\textbf{F}ield \textbf{S}ubsampling framework for scalable cooperative MARL with heterogeneous agent interactions. By subsampling $\kappa$ agents according to interaction strength, we approximate the graphon-weighted mean-field and learn a policy with sample complexity $\mathrm{poly}(\kappa)$ and optimality gap $\tilde{O}(1/\sqrt{\kappa})$. We verify our theory with numerical simulations in robotic coordination, showing that \texttt{GMFS} achieves near-optimal performance. \looseness=-1
\end{abstract}

\vspace{-0.5cm}
\section{Introduction}

Multi-agent reinforcement learning (MARL) has emerged as a powerful framework for modeling complex, large-scale networked systems whose dynamics stem from the interactions between many individual agents. These systems are now ubiquitous, appearing in domains as varied as robotic swarms \citep{lv2024localinformationaggregationbased, aina2025deepreinforcementlearningmultiagent, chiun2025marvelmultiagentreinforcementlearning, 7989376}, autonomous driving \citep{kiran2021deepreinforcementlearningautonomous, 9351818, shalevshwartz2016safemultiagentreinforcementlearning, deweese2024locally}, ride-sharing \citep{10.1145/3308558.3313433}, real-time bidding \citep{10.1145/3269206.3272021}, stochastic games \citep{pmlr-v125-jin20a, Chaudhari_2025}, and efficient wireless networks \citep{8807386, zhang2025multiagentreinforcementlearningwireless}. In these settings, the objective is to derive optimal policies that maximize collective reward across the entire system. Drawing on the success of reinforcement learning (RL) in tasks like the game of Go~\citep{Silver_Huang_Maddison_Guez_Sifre_van_den_Driessche_Schrittwieser_Antonoglou_Panneershelvam_Lanctot_et_al._2016} and robotic control \citep{doi:10.1177/0278364913495721}, MARL extends these ideas to environments with multiple interacting agents. However, quickly solving such systems exactly becomes infeasible as the number of agents grows, creating a barrier to large-scale deployment.\looseness=-1

The primary challenge in scaling MARL lies in the size of the joint state-action space $(|\mathcal S||\mathcal A|)^n$, which grows exponentially with the population size $n$ and renders exhaustive sampling over all agents computationally intractable \citep{Blondel_Tsitsiklis_2000}. To mitigate this ``curse of dimensionality'', previous works proposed using mean-field theory to approximate agent populations by replacing explicit interactions with population aggregates~\cite{yang2020meanfieldmultiagentreinforcement, carmona2021modelfreemeanfieldreinforcementlearning, li2021permutationinvariantpolicyoptimization, guo2021learningmeanfieldgames, chen2021pessimism, cui2024majorminormeanfieldmultiagent,gu2022meanfieldmultiagentreinforcementlearning}. However, these frameworks typically require iterating over the entire population at each time step to compute aggregates. More recently, \citet{anand2025meanfieldsamplingcooperativemultiagent} mitigate this cost by sampling a subset of $k$ agents to achieve polynomial complexity. However, these standard mean-field methods often fail at capturing the heterogeneous behaviors inherent in complex systems like traffic management \citep{zheng2025enhancedmeanfieldgame}. As noted by \citet{carmona2021modelfreemeanfieldreinforcementlearning}, ignoring agent diversity can lead to systemic failures, including communication delays~\citep{deep_iot, tolerable_compute}, inefficient information dissemination \citep{4118472}, and unstable robotic control \citep{Yi2016ConsensusIS}. \looseness-1

To model these non-uniformities,~\citet{lovasz2012large, 8263796, Caines_2021} represent network systems via graphons, where a limit function $W(x, y)$ defines interaction strengths between agents on a latent index $x \in [0,1]$. Although \citet{hu2022graphonmeanfieldcontrolcooperative} established an approximation error bound of $O(1/\sqrt{n})$ for such systems, no existing framework unifies the computational efficiency of neighborhood subsampling with the expressiveness of graphon-level heterogeneity. Therefore, we ask: 
\begin{center}
    \emph{Can a mean-field MARL algorithm handle heterogeneity through non-uniform agent interactions, while reducing computation to a polylogarithmic dependence on the number of agents?}
\end{center}

\subsection{Contributions}  Our key contributions are outlined below.
\begin{itemize}[noitemsep, topsep=0pt, parsep=6pt, partopsep=6pt]
    \item  \texttt{GMFS} \textbf{Algorithm (\S\ref{sec:gmfs_algorithm}).}
    We introduce a subsampling framework for cooperative MARL with heterogeneous agents, where graphon mean-field interactions are approximated using $\kappa\ll n$ sampled neighbors. \looseness-1
    \item \textbf{Theoretical Guarantees (\S\ref{sec:theoretical guarantees and analysis_main}, \S\ref{sec:proof outline_main}).} We bound the optimality gap between the $\kappa$-sampled policy and the optimal graphon mean-field policy. For finite state-action spaces, we show that \texttt{GMFS} reduces the sample complexity from exponential in $n$,  ($|\mathcal{S}|^n|\mathcal{A}|^{n}$), to polynomial in $\kappa$, ($\kappa^{|\mathcal{S}||\mathcal{A}|}$). \looseness-1
    \item \textbf{Numerical Simulations (\S\ref{sec: numerical simulations}).}
We observe monotonic improvements in average discounted return as $\kappa$ increases, approaching the exhaustive mean-field limit in a heterogeneous robot-warehouse environment. \looseness=-1

\end{itemize}

A crucial challenge in heterogeneous systems is the position-dependent variance that makes uniform sampling ineffective. We show how to resolve this issue via graphon-weighted subsampling, where neighbors are sampled directly according to normalized graphon weights. By showing that the resulting sampled Bellman operator remains a $\gamma$-contraction, we prove that the algorithm converges to an approximately optimal policy. We further show that the optimality gap concentrates at an $O(1/\sqrt{\kappa})$ rate (Theorem~\ref{Lemma: lipschitz cont of operators}), unifying uniform mean-field and exhaustive population methods.

\subsection{Related Literature}
MARL has a rich history, beginning with early work on Markov games for multi-agent decision-making \citep{littman,Sutton_McAllester_Singh_Mansour_1999}, which can be viewed as multi-agent extensions of Markov Decision Processes (MDPs). Since then, MARL has been studied across a wide range of settings \citep{zhang2021multiagent}.  It is closely related to ``succinctly described’’ MDPs \citep{Blondel_Tsitsiklis_2000}, where the joint state-action space is a product of individual agent spaces and agents optimize a collective objective. A promising recent line of work constrains such problems to sparse networked instances to enforce local interactions \citep{10.5555/3495724.3495899,DBLP:journals/corr/abs-2006-06555,10.5555/3586589.3586718}.In this formulation, agents correspond to graph vertices and interact only with immediate neighbors. By exploiting correlation decay \citep{gamarnik2009correlation}, these methods mitigate the curse of dimensionality by searching over policies defined on truncated graphs.  Yet, as networks become dense, agent neighborhoods grow, rendering graph-truncation approaches computationally intractable. \looseness=-1

\textbf{Mean-Field RL}. To address large agent neighborhoods, mean-field theory replaces a finite set of local agents with an empirical distribution over agent states \cite{Lasry_Lions_2007,yang2020meanfieldmultiagentreinforcement,Yang_2023}. This simplifies multi-agent interactions into a two-agent formulation, where each agent interacts with a representative ``mean agent'' that evolves according to the empirical distribution of all other agents \citep{gu2022meanfieldmultiagentreinforcementlearning}. For a detailed overview of learning methodologies in mean-field games, we refer the reader to \citet{Laurire2022LearningIM}.~\citet{pásztor2023efficientmodelbasedmultiagentmeanfield} explore efficient model-based approaches to further improve sample efficiency in these settings. By sampling a subset of the total population, \citet{anand2025meanfieldsamplingcooperativemultiagent} build on this mean-field abstraction to achieve a sample complexity that is polynomial in the number of subsampled agents. However, these approaches operate under the assumption that interactions between agents are uniform. Our \texttt{GMFS} algorithm utilizes graphon functions to model heterogeneous interactions in dense graphs, better representing realistic networked systems \citep{fabian2023learningsparsegraphonmean} while maintaining provable performance guarantees. \looseness=-1

\textbf{Graphon Mean-Field MARL.} Building on classical mean-field MARL, a growing body of recent work relaxes homogeneity assumptions by modeling dense, heterogeneous interactions in large populations. \citet{8619367} introduced Graphon Mean-Field Games (GMFG), formalizing graphons as limiting objects that encode heterogeneous structures in infinite networks  \citep{Borgs2007ConvergentSO}. Subsequent work expands GMFG concepts to broader classes of control and learning problems, including dynamical sequential games \citep{cui2022learninggraphonmeanfield, fabian2023learningsparsegraphonmean, zhang2023learningregularizedmonotonegraphon}, learning algorithms for realistic sparse graphs \citep{pmlr-v267-fabian25a}, and graphon estimation from sampled agents \citep{10.5555/3722577.3722949, fabian2022meanfieldgamesweighted}. Moreover, even when the graphon is unknown, it can be efficiently estimated \citep{pmlr-v32-chan14, 10.5555/2999611.2999689}. In the cooperative MARL setting, \citet{hu2022graphonmeanfieldcontrolcooperative} adapts the GMFG framework to Graphon Mean-Field Control, where agents optimize a joint $Q$-function driven by graphon-weighted aggregates. While these works demonstrate that graphons provide a principled mechanism for capturing heterogeneity, existing methods typically require full population aggregation at each time step.  In contrast, \texttt{GMFS} approximates these aggregates using only $\kappa < n$ sampled neighbors.  \looseness-1
 
\textbf{Structured RL.} Our work is also related to factored and exogenous MDPs. In factored MDPs, a global action typically affects every agent, whereas in our setting, each agent selects its own action \citep{pmlr-v202-min23a}. Our approach shares similarities with MDPs involving exogenous inputs \citep{dietterich2018discoveringremovingexogenousstate,foster2022on}, as our subsampling algorithm treats sampled states as endogenous while allowing for dynamic exogenous dependencies. \looseness-1

\textbf{Other Related Work.} Beyond the areas highlighted above, our work contributes to the literature on Centralized Training with Decentralized Execution \citep{zhou2023centralizedtrainingdecentralizedexecution}, as we learn a provably near-optimal policy using centralized information while executing decisions based only on local observations. In distributed settings, V-learning \citep{pmlr-v125-jin20a} reduces the exponential dependence on the joint action space to an additive one. In contrast, our approach further reduces the complexity of the joint state space, which has not been previously achieved. Finally, while linear function approximation can be used to reduce $Q$-table complexity \citep{pmlr-v125-jin20a}, bounding the resulting performance loss is generally intractable without strong assumptions such as Linear Bellman completeness \citep{golowich2024roleinherentbellmanerror} or low Bellman-Eluder dimension \citep{jin2021bellmaneluderdimensionnew}. While our work primarily considers the finite tabular setting, we also provide extensions to the non-tabular case under Linear Bellman completeness. \looseness-1

\section{Preliminaries}
\label{sec: preliminaries}
We formally introduce the problem, state motivating examples for our setting, and provide technical details about the mean-field model and graphon-based techniques used. \looseness-1

\textbf{Notation.} For $n, \kappa \in \mathbb{N}$ with $\kappa \leq n$, let $\binom{[n]}{\kappa}$ denote the set of all $\kappa$-sized subsets of $[n] = \{1, \dots, n\}$. For any vector $z \in \mathbb{R}^d$, let $\|z\|_1$ and $\|z\|_\infty$ denote the standard $\ell_1$ and $\ell_\infty$ norms, respectively. Given a collection of variables $s_1, \dots, s_n$ and $\Delta\subseteq[n]$, the shorthand $s_{\Delta}$ denotes the set $\{s_i : i \in \Delta\}$. We use $\tilde{O}(\cdot)$ notation to suppress polylogarithmic factors in all problem parameters except $n$. For a discrete measurable space $(\mathcal{X}, \mathcal{F})$, the total variation distance between probability measures $\rho_1$ and $\rho_2$ is defined as $\mathrm{TV}(\rho_1, \rho_2) = \frac{1}{2} \sum_{x \in \mathcal{X}} |\rho_1(x) - \rho_2(x)|$. We write $x \sim \mathcal{D}$ to indicate that $x$ is sampled from the distribution $\mathcal{D}$, and $x \sim \mathcal{U}(\Omega)$ to denote sampling from the uniform distribution over a finite set $\Omega$. \looseness=-1

\textbf{Graphons.} A graphon is a bounded, measurable, symmetric function $W\!:\!\mathcal{I}^2\!\to\!\mathcal{I}$, where $\mathcal{I}\!\!=\!\![0,1]$, 
that encodes interaction weights in dense graphs. 
Given deterministic latent points $\{\alpha_i\}_{i=1}^n\!\subset\![0,1]$,
the graphon induces a complete weighted interaction matrix with entries
$w_{ij} \coloneqq W(\alpha_i,\alpha_j)$ (and $w_{ii} \coloneqq 0$). Graphons arise as limit objects for dense graphs, and any finite weighted graph $G$ admits an associated underlying graphon $W_G$ constructed from its adjacency matrix \cite{lovasz2012large}.\looseness=-1

\subsection{Problem Formulation}

We consider a system of $n$ cooperative agents indexed by $[n] = \{1, \dots, n\}$. Let $\mathcal{S}$ and $\mathcal{A}$ denote the finite state and action spaces of each agent, respectively (Assumption \labelcref{assumption: finite s-a}). At each time $t$, the agents have a joint state $s(t) = (s_1(t), \dots, s_n(t)) \in \mathcal{S}^n$ and select a joint action $a(t) = (a_1(t), \dots, a_n(t)) \in \mathcal{A}^n$. The interactions between the agents can be written as a weighted graph $\mathbb{G} = (\mathcal{V}, \mathcal{E})$ with vertex set $\mathcal{V} = [n]$, where edge weights are determined by a measurable, bounded, symmetric graphon $W : \mathcal{I}^2 \to \mathcal{I}$ for $\mathcal{I} = [0,1]$. Each agent $i$ is assigned a latent coordinate $\alpha_i = i/n \in \mathcal{I}$, enabling the graphon to represent non-uniform interaction structures within the network.\looseness=-1

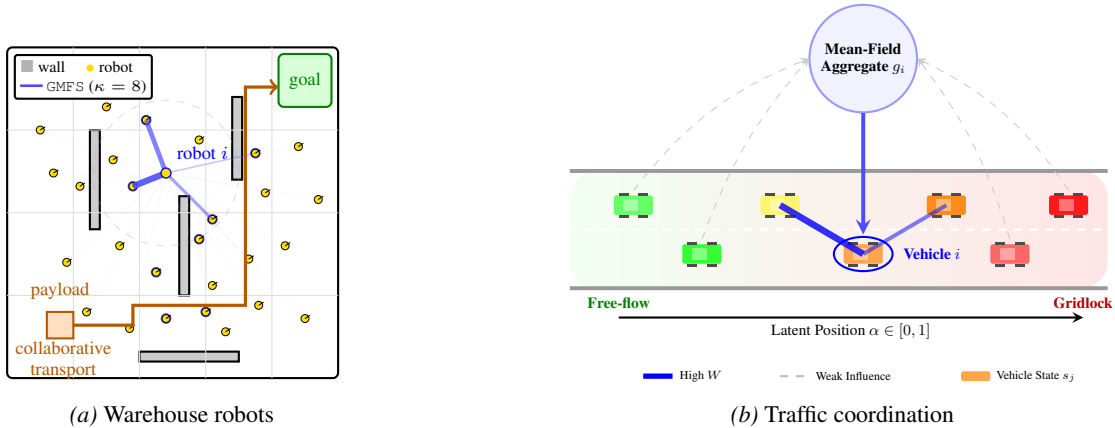
\begin{figure*}[t]  
\centering
 \begin{subfigure}[b]{0.48\textwidth}
\centering
\newcommand{\drawrobot}[2]{
  \fill[yellow!80!orange] (#1,#2) circle (0.012);
  \draw[black, line width=0.3pt] (#1,#2) circle (0.012);
  \draw[black, line width=0.4pt] (#1,#2) -- (#1+0.015,#2+0.010);
}
\newcommand{\drawsampledrobot}[2]{
  \fill[blue!30] (#1,#2) circle (0.014);
  \fill[yellow!80!orange] (#1,#2) circle (0.012);
  \draw[blue!70, line width=0.5pt] (#1,#2) circle (0.014);
  \draw[black, line width=0.3pt] (#1,#2) circle (0.012);
  \draw[black, line width=0.4pt] (#1,#2) -- (#1+0.015,#2+0.010);
}
\begin{tikzpicture}[scale=4.4, every node/.style={font=\small}]
    \draw[thick, rounded corners=2pt] (0,0) rectangle (1,1);
    
    \fill[gray!40] (0.25,0.45) rectangle (0.28,0.75); \draw[black, thick] (0.25,0.45) rectangle (0.28,0.75);
    \fill[gray!40] (0.52,0.25) rectangle (0.55,0.55); \draw[black, thick] (0.52,0.25) rectangle (0.55,0.55);
    \fill[gray!40] (0.68,0.60) rectangle (0.71,0.85); \draw[black, thick] (0.68,0.60) rectangle (0.71,0.85);
    \fill[gray!40] (0.40,0.05) rectangle (0.70,0.08); \draw[black, thick] (0.40,0.05) rectangle (0.70,0.08);

    \foreach \x in {0.2,0.4,0.6,0.8}{ \draw[gray!25] (\x,0) -- (\x,1); }
    \foreach \y in {0.25,0.5,0.75}{ \draw[gray!25] (0,\y) -- (1,\y); }

    \fill[green!15, rounded corners=2pt] (0.82,0.82) rectangle (0.98,0.98);
    \draw[green!50!black, thick, rounded corners=2pt] (0.82,0.82) rectangle (0.98,0.98);
    \node[green!50!black, font=\scriptsize] at (0.90,0.90) {goal};
    \fill[orange!25] (0.12,0.12) rectangle (0.20,0.20);
    \draw[orange!70!black, thick] (0.12,0.12) rectangle (0.20,0.20);
    \node[orange!70!black, anchor=south, font=\scriptsize] at (0.16,0.21) {payload};

    \foreach \x/\y in {0.1/0.75, 0.14/0.62, 0.18/0.35, 0.22/0.58, 0.24/0.2, 0.3/0.82, 0.32/0.66, 0.34/0.4, 0.37/0.15, 0.58/0.72, 0.62/0.28, 0.66/0.14, 0.73/0.36, 0.76/0.22, 0.78/0.56, 0.84/0.4, 0.88/0.7, 0.9/0.18, 0.94/0.54}{\drawrobot{\x}{\y}}

    \coordinate (ri) at (0.48,0.62);
    \foreach \x/\y in {0.1/0.75, 0.14/0.62, 0.18/0.35, 0.22/0.58, 0.24/0.2, 0.3/0.82, 0.32/0.66, 0.34/0.4, 0.37/0.15, 0.58/0.72, 0.62/0.28, 0.66/0.14, 0.73/0.36, 0.76/0.22, 0.78/0.56, 0.84/0.4, 0.88/0.7, 0.9/0.18, 0.94/0.54}{\draw[gray!30, line width=0.2pt, opacity=0.15] (ri) -- (\x,\y);}

    \drawsampledrobot{0.38}{0.58} \drawsampledrobot{0.42}{0.78} \drawsampledrobot{0.62}{0.48} \drawsampledrobot{0.58}{0.42} \drawsampledrobot{0.75}{0.68} \drawsampledrobot{0.45}{0.32} \drawsampledrobot{0.60}{0.20} \drawsampledrobot{0.48}{0.18}

    \draw[blue!70, line width=2.4pt, opacity=0.9] (ri) -- (0.38,0.58);
    \draw[blue!70, line width=1.8pt, opacity=0.7] (ri) -- (0.42,0.78);
    \draw[blue!70, line width=1.2pt, opacity=0.5] (ri) -- (0.62,0.48);
    \draw[blue!70, line width=0.6pt, opacity=0.3] (ri) -- (0.75,0.68);

    \fill[yellow!80!orange] (ri) circle (0.015);
    \draw[blue, line width=0.5pt] (ri) circle (0.015);
    \node[blue, anchor=south west, font=\scriptsize] at (0.485,0.635) {robot $i$};
    \draw[blue!40, dashed, opacity=0.35] (ri) circle (0.22);

    \draw[->, orange!70!black, very thick] (0.2,0.16) -- (0.38,0.16) -- (0.38,0.22) -- (0.72,0.22) -- (0.72,0.88) -- (0.82,0.88);
    \node[orange!70!black, align=center, anchor=north, font=\scriptsize] at (0.17,0.14) {collaborative\\ transport};

    \node[draw, thick, fill=white, align=left, anchor=north west, font=\tiny, inner sep=2pt, rounded corners=1pt] at (0.02,0.98) {
      {\color{gray!60}$\blacksquare$} wall \quad {\color{yellow!80!orange}$\bullet$} robot\\
      {\color{blue!70}\rule[0.5ex]{0.3cm}{1pt}} \texttt{GMFS} ($\kappa=8$)
    };
\end{tikzpicture}
\caption{Warehouse robots}
\label{fig:warehouse}
\end{subfigure}
\hfill
\begin{subfigure}[b]{0.48\textwidth}
\centering
\begin{tikzpicture}[scale=0.65, transform shape]
    \newcommand{\drawcar}[3]{
        \begin{scope}[shift={(#1,#2)}]
            \fill[#3, rounded corners=1pt] (-0.4,-0.2) rectangle (0.4,0.2); 
            \fill[black!70] (-0.3,0.2) rectangle (-0.1,0.25);
            \fill[black!70] (0.1,0.2) rectangle (0.3,0.25); 
            \fill[black!70] (-0.3,-0.2) rectangle (-0.1,-0.25); 
            \fill[black!70] (0.1,-0.2) rectangle (0.3,-0.25); 
            \fill[white, opacity=0.4] (-0.2,-0.12) rectangle (0.1,0.12);
        \end{scope}
    }

    \shade[left color=green!5, right color=red!10, rounded corners=10pt] (-0.5, -0.2) rectangle (10.5, 2.2);
    \draw[gray!80, line width=1.5pt] (-0.5, -0.2) -- (10.5, -0.2);
    \draw[gray!80, line width=1.5pt] (-0.5, 2.2) -- (10.5, 2.2);
    \draw[white, dashed, line width=1pt] (-0.5, 1.0) -- (10.5, 1.0);

    \drawcar{0.8}{1.5}{green!60}
    \drawcar{2.2}{0.5}{green!80}
    \drawcar{3.8}{1.5}{yellow!70}
    
    \drawcar{5.5}{0.5}{orange!70}
    \draw[blue, thick] (5.5, 0.5) ellipse (0.6 and 0.35);
    \node[blue, font=\small\bfseries, anchor=west] at (6.2, 0.5) {Vehicle $i$};

    \drawcar{7.2}{1.5}{orange!90}
    \drawcar{8.5}{0.5}{red!60}
    \drawcar{9.7}{1.5}{red!90}

    \node[draw=blue!40, fill=blue!5, thick, circle, minimum size=2.2cm, align=center, font=\small\bfseries, inner sep=2pt] (mf) at (5.5, 4.5) {Mean-Field\\Aggregate $g_i$};

    \draw[->, >=stealth, gray!40, dashed, thin, bend left=15] (0.8, 1.8) to (mf.west);
    \draw[->, >=stealth, gray!40, dashed, thin, bend left=10] (2.2, 0.8) to (mf.west);
    \draw[->, >=stealth, gray!40, dashed, thin, bend right=10] (8.5, 0.8) to (mf.east);
    \draw[->, >=stealth, gray!40, dashed, thin, bend right=15] (9.7, 1.8) to (mf.east);

    \draw[->, >=stealth, blue!70, line width=1.5pt] (mf.south) -- (5.5, 0.9);

    \draw[blue, line width=2.5pt, opacity=0.8] (5.5, 0.5) -- (3.8, 1.5);
    \draw[blue, line width=1.5pt, opacity=0.5] (5.5, 0.5) -- (7.2, 1.5);

    \draw[->, >=stealth, thick] (0.5, -0.8) -- (10, -0.8);
    \node[below, font=\small] at (5.25, -0.8) {Latent Position $\alpha \in [0, 1]$};
    \node[font=\footnotesize\bfseries, green!50!black] at (0.5, -0.5) {Free-flow};
    \node[font=\footnotesize\bfseries, red!70!black] at (10, -0.5) {Gridlock};

    \begin{scope}[shift={(1.0, -2.0)}]
        \draw[blue, line width=2pt] (0,0) -- (0.6,0) node[right, black, font=\scriptsize] {High $W$};
        \draw[gray!50, dashed, thick] (2.8,0) -- (3.4,0) node[right, black, font=\scriptsize] {Weak Influence};
        \fill[orange!70, rounded corners=1pt] (6.5,-0.1) rectangle (7.1,0.1);
        \node[right, black, font=\scriptsize] at (7.1,0) {Vehicle State $s_j$};
    \end{scope}
\end{tikzpicture}
\caption{Traffic coordination}
\label{fig:traffic}
\end{subfigure}

\caption{\textbf{Graphon mean-field systems with distance-decay interactions}. (a) Warehouse robots collaborate to transport a payload, where robot $i$ uses $\kappa=8$ subsampled neighbors with interaction strength $W(x_i,x_j)$ indicated by line thickness. (b) Traffic vehicles coordinate using graphon-weighted aggregates $g_i$ of the population, where interaction strength decays with distance in latent position space $\alpha \in [0,1]$. \looseness-1}
\label{fig:combined}
\end{figure*}

\begin{definition}[Graphon-weighted neighborhood state-action feature for agent $i$]\label{def: Graphon-weighted neighborhood state--action distribution for agent} \textit{Let $W:[0,1]^2\to[0,1]$ be a graphon. We fix deterministic latent coordinates $\{\alpha_i\}_{i=1}^n \subset [0,1]$. The graphon induces a complete weighted interaction graph $w_{ij} \coloneqq  W(\alpha_i, \alpha_j)$ where $w_{ii}\coloneqq 0$, and normalized influence weights for each $i$, $\bar{w}_{ij} \coloneqq \frac{w_{ij}}{\sum_{m\neq i} w_{im}}$, where $\bar{w}_{ii}\coloneqq 0$. Then, for a joint state/action pair $(\mathbf s, \mathbf a)\!\in\!\mathcal S^n\!\times\!\mathcal A^n$,  
agent $i$’s graphon-weighted neighborhood state/action feature is a probability
mass function $z_i \in \cZ \subset \mathcal P(\mathcal S\times \mathcal A)$,
defined for all $(x,u)\in \mathcal S\times \mathcal A$,\looseness=-1
\[z_i(x,u)\coloneqq \sum_{j\neq i}\bar w_{ij} \mathbbm 1\{\mathbf{s}_j=x, \mathbf{ a}_j=u\},
\]
where its state marginal is $g_{z_i}(x)\!=\!\sum_{u\in \mathcal{A}} z_i(x,u)\!\in\!\mathcal{G}$ and action marginal is $h_{z_i}(u)\!=\!\sum_{x\in \mathcal{S}} z_i(x,u)\!\in\!\mathcal{H}$.}
\end{definition}

\textbf{Agent dynamics.} Let $\mathfrak G(\mathcal S)$ denote the space of neighborhood state features, the probability simplex over $\mathcal{S}$. At time $t$, agent $i$ observes its local state $s_i(t)$ and its neighborhood feature.
It selects an action $a_i(t)\in\mathcal A$, and the next state evolves according to the transition kernel $P:\mathcal{S}\times\mathcal{A}\times\mathfrak{G}(\cS)\to\mathcal P(\mathcal S)$, such that $
s_i(t+1)\sim P(\cdot\mid s_i(t),a_i(t),g_i(t))$ for all $i\in[n]$, where $g_i(t)\in\mathfrak G(\mathcal S)$ is agent $i$'s graphon-weighted neighborhood state feature. That is, the agent transitions depend only on empirical neighborhood features
$g\in \mathfrak G ({\mathcal{S}})$.\looseness=-1

\textbf{Team objective.} We define a local reward function
$r_\ell: \mathcal{S}\times\mathcal{A}\times\mathfrak{G}(\mathcal{S}) \to \mathbb{R}$ for each agent, and the team stage reward on $(\mathbf{s}_{1:n}(t),\mathbf{a}_{1:n}(t),\mathbf{g}_{1:n}(t))\in \mathcal{S}^n\times\mathcal{A}^n\times\mathcal{G}^n$ as $r_t$ where\looseness=-1
\[r_t\coloneqq r_t(\mathbf{s}(t), \mathbf{a}(t), \mathbf{g}(t))\coloneqq\frac{1}{n}\sum_{i=1}^n r_\ell\big(\mathbf s_i(t), \mathbf a_i(t), \mathbf g_i(t)\big).\]
Fix a discounting factor $\gamma\in(0,1)$. Then, the discounted team return with joint policy $\pi = (\pi_1, ....\pi_n)$ and initial state $\mathbf{s}_{1:n}\in\mathcal{S}^n$ is then given by  \looseness=-1
\[V^\pi_{\text{team}} (\mathbf s_{1:n})\coloneqq \mathbb E_\pi\left[\sum_{t\ge 0}\gamma^t r_t \bigg| \mathbf{s}_{1:n}(0) = \mathbf{s}_{1:n}\right].\]

Note that each agent $i$'s state transition and reward depends on agents $[n]\setminus i$ only through the state marginal $g_i(x)$, but the evolution of $g_i(x)$ depends on how they choose actions based on their states, which is captured in $Z_i(x,u)$.

\begin{definition}[$\epsilon$-optimal policy] \textit{A policy $\pi$ is $\epsilon$-optimal if for all $\mathbf{s}_{1:n}\in\mathcal{S}^n$, we have $V^\pi_{\text{team}}(\mathbf s) \geq \sup_{\pi^*} V^{\pi^*}_{\text{team}}(\mathbf s) - \epsilon$.}
\end{definition}

\textbf{Motivating Examples.} Below we give examples of cooperative graphon MARL settings naturally captured by this formulation. Our empirical results show that the average cumulative discounted return of learned policies increases monotonically as the subsampling parameter $\kappa$ approaches $n$, while providing a computational speedup over graphon mean-field $Q$-learning methods.\footnote{We provide more details of our numerical experiments in \cref{sec: numerical simulations} and have released the code \href{https://github.com/SarahLiaw/graphon-marl}{here}.}\looseness=-1
\begin{itemize}[noitemsep, topsep=0pt, parsep=6pt, partopsep=6pt]

\item \textbf{Robot Coordination in Constrained Workspaces.} Consider a swarm of $n$ mobile robots performing collaborative tasks (e.g., cargo transport in a warehouse). Since robots operate in constrained regions (aisles and loading zones), local congestion can significantly affect task completion times. If robot $i$ is assigned a location $\alpha_i \in [0,1]^2$, we can model their interactions with a graphon $W(\alpha_i,\alpha_j) = \mathbbm{1}\{\|\alpha_i- \alpha_j \|_2 \leq r\}$ and a fixed interaction radius $r > 0$, reflecting concerns like collision avoidance, and its empirical neighborhood state-action distribution $\hat{z}_i^{(\kappa)}$  and its state marginal $\hat{g}_i^{(\kappa)}$ are used to compute its reward. For tractability, each agent constructs a $\kappa$-sampled approximation of its neighborhood aggregate via graphon-weighted subsampling, so that a robot’s optimal policy is primarily influenced by nearby agents (e.g., those sharing an aisle) rather than further ones. These heterogeneous densities can be represented with a block-structured graphon where different subsections of the unit square correspond to workspace sectors with varying congestion levels (Figure \labelcref{fig:warehouse}).\looseness-1

\item \textbf{Autonomous Vehicle Coordination.} Consider $n$ autonomous vehicles navigating a road network where efficient coordination is needed to minimize travel times and prevent gridlock. In classical mean-field settings, all agents are assumed to interact uniformly, but real-world traffic dynamics exhibit significant locality as a vehicle's congestion experience typically depends on the density of nearby traffic rather than the state of the entire network. This spatial structure is captured by a distance-decay graphon $W(\alpha_i, \alpha_j) = \exp(-\beta|\alpha_i - \alpha_j|)$, where vehicles that are close to one another exert a strong mutual influence while those further away have a negligible impact on local flow. To enable scalable learning in this environment, each vehicle could use graphon-weighted subsampling to approximate the local traffic density without having to process the full network information. This allows decentralized policies that respect the underlying locality of traffic dynamics, thereby allowing vehicles to make informed decisions to improve overall network throughput (Figure~\ref{fig:traffic}). \looseness-1
\end{itemize}

\textbf{Capturing heterogeneity in state/action spaces.} Following \citet{10.5555/3586589.3586718}, we model heterogeneity in agent states and actions by incorporating agent types directly into the state. Specifically, each agent $i$ is assigned a type $\varepsilon_i\in\mathcal{E}$, and we define the state space $\mathcal{S} = \mathcal{E}\times\mathcal{S}'$, where $\mathcal{E}$ indexes agent types and $\cS'$ denotes the latent state space. The transition and reward functions may then depend explicitly on the agent’s type. \looseness-1

\section{Graphon Mean-Field Subsampling (GMFS)}\label{sec:gmfs_algorithm}

\begin{figure*}
    \centering
    \includegraphics[width=0.9\textwidth]{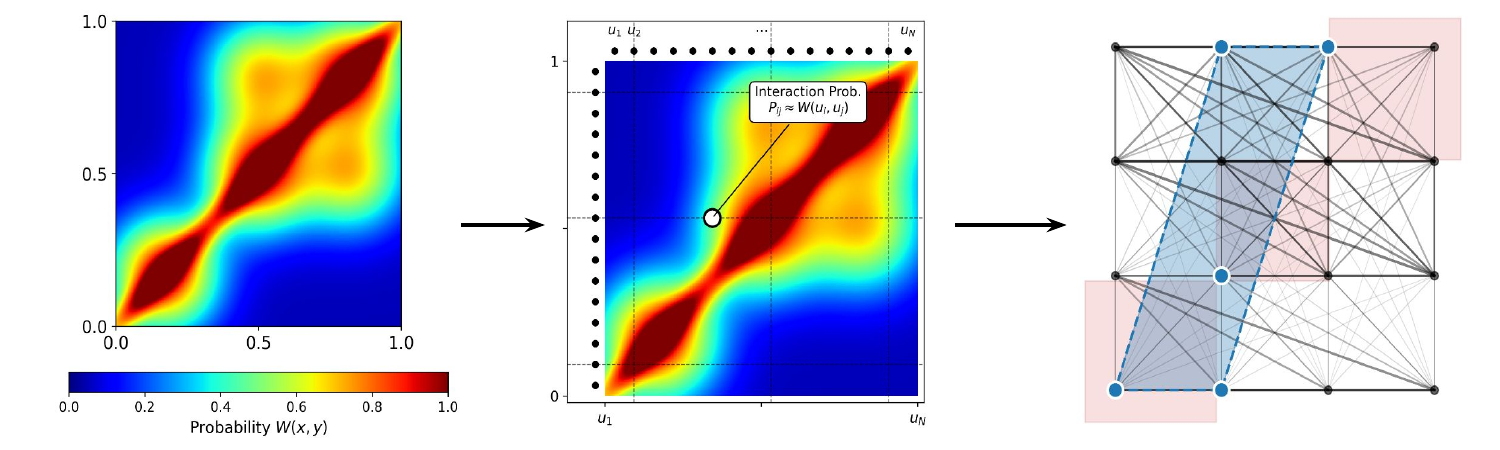}
 \caption{\textbf{Schematic of Graphon Mean-Field Sampling.} (Left) A continuous graphon $W(x,y)$ represents the infinite-population limit of non-uniform interactions. (Middle) The graphon with deterministic latent positions $\{\alpha_i\}_{i=1}^n\subset [0,1]$ induces a complete weighted interaction graph on $n$ agents with edge weights $w_{ij} = W(\alpha_i, \alpha_j)$ and $w_{ii}=0$. These weights specify the intensity with which agent $i$ aggregates neighbor states into its mean-field features used for learning and control. (Right) Each agent approximates its graphon-weighted neighborhood statistics by sampling a small set of agents (random subsample of $\kappa = 5$) according to the normalized weights $\bar{w}_{ij}$ .}
    \label{fig:graphon_concept}
\end{figure*}

We propose Graphon Mean-Field Subsampling (\texttt{GMFS}), a framework that handles agent heterogeneity through non-uniform interactions modeled by graphons. To ensure scalability in large-scale systems, we introduce a subsampled neighborhood approximation, where each agent constructs an empirical neighborhood histogram from a small subset of the population. Using the graphon aggregates from Definition \ref{def: Graphon-weighted neighborhood state--action distribution for agent}, \texttt{GMFS} learns an approximately optimal policy that achieves near-optimal team rewards under a sampled mean-field limit. This approach extends local sampling used in graph neural networks \citep{10.5555/3294771.3294869} to the MARL setting. \looseness-1

\begin{assumption}[Well-defined graphon weights]
We assume the underlying graphon $W:[0,1]^2\to[0,1]$ satisfies
$\int_0^1 W(\alpha,\beta) d\beta > 0
\quad \text{for all } \alpha \in [0,1],$
ensuring that all row-normalized graphon weights are well defined.
\end{assumption}
\begin{definition}[Graphon-weighted Subsampling] \label{def: graphon weighted sampling} \label{def: graphon subsample} \textit{Fix $\kappa\geq 1$. For each agent $i\in[n]$, sample a multiset of $\kappa$ neighbors $\Delta_i=(J_i^{(1)},\cdots,J_i^{(\kappa)})$ where 
$J_i^{(m)}\sim  \bar w_{i,\cdot}$ on $[n]\setminus\{i\}$. Across all the agents, this yields a collection $\{\Delta_i\}_{i=1}^n$.\looseness=-1}
\end{definition}

\begin{definition}[Sampled neighborhood aggregates] 
\label{def: sampled neighborhood agg}
\textit{For $\kappa\geq 1$, let $\mathcal G_\kappa \coloneqq \mathcal P_\kappa(\mathcal S)\subset \mathfrak G(\mathcal S)$
and $\mathcal Z_\kappa \coloneqq \mathcal P_\kappa(\mathcal S\times\mathcal A)$
be the sets of empirical histograms with denominator $\kappa$.
Given $\Delta_i$ from Definition \ref{def: graphon subsample}, define the empirical $\kappa$-sample neighborhood state-action distribution $\hat{Z}_i^{(\kappa)}\in\mathcal{Z}_{\kappa}$, where
\[
\hat Z_i^{(\kappa)}(s,a) \coloneqq \frac{1}{\kappa}\sum_{m=1}^{\kappa}\mathbbm 1\{s_{J_i^{(m)}}=s,\ a_{J_i^{(m)}}=a\},
\]
and let the sampled neighborhood state marginal $\hat{g}_i^{(\kappa)} \in \mathcal{G}_{\kappa}$ be given by $
\hat g_i^{(\kappa)}(s) \coloneqq \sum_{a\in\mathcal A}\hat Z_i^{(\kappa)}(s,a).
$}\looseness-1
\end{definition}
 
\textbf{Sampled local reward.} Each agent evaluates its local reward through the sampled neighborhood feature $r_\ell(s_i,a_i,\hat g_i^{(\kappa)})$ as an approximation of the true mean-field reward. The discrepancy between this sampled value and the ground truth is governed by the regularity conditions established in the following assumptions. \looseness-1

\begin{assumption}[Finite state/action spaces]
\label{assumption: finite s-a}
\textit{We assume that the state and action spaces of all the agents in the system are finite: $|\mathcal{S}|, |\mathcal{A}| < \infty$. Appendix \ref{appendix: extension to continuous state spaces} relaxes this assumption to the non-tabular setting with infinite continuous states.}\looseness-1
\end{assumption}

\begin{assumption}[Bounded rewards] \textit{We assume that the components of the reward function are bounded. Specifically, we assume $\|r_\ell\|_\infty < \infty$.} \looseness-1
\end{assumption}

\begin{assumption}
    [The local component of the reward function is $2 \|r_\ell\|_\infty$-Lipschitz in the empirical distribution] \label{assumption: reward function is Lipschitz} \textit{We assume that for all $s\in \mathcal{S}, a\in\mathcal{A}$, and $g, g' \in \mathfrak{G}(\mathcal S)$, \looseness=-1
    \[|r_\ell(s, a, g) - r_\ell(s,a,g')| \leq 2 \|r_\ell\|_\infty \cdot \mathrm{TV}(g, g').\]}
\end{assumption}

\begin{assumption} [Transitions are $L_P$-Lipschitz] \label{assumption: transition are lipschitz} \textit{There exists $1 \leq L_P < \infty$ such that for all $s,a$ and all $g,g'\in\mathfrak G(\mathcal S)$,
    \[\mathrm{TV}(P(\cdot | s,a,g), P(\cdot | s,a,g')) \leq L_P \cdot \mathrm{TV}(g,g').\]}
\end{assumption}

 For $p\in\{1,\dots,n-1\}$, define the Banach space $\mathcal Y_p\!=\!\{Q:\mathcal S\!\times\!\mathcal A\!\times\! \mathcal Z_p \!\to\!\mathbb R\}$ equipped with the sup-norm $\|Q\|_\infty\!\coloneqq\!\sup_{s,a, z}|Q(s,a, z)|$.

\textbf{Team value decomposition.} Under the policy class where each agent’s decision rule depends only on its local observation $(s_i, g_i)$, the induced value function for agent $i$ can be written as a function of $s_i$ and $g_i$. Then, the team value decomposes as $V_{\text{team}}^\pi(\mathbf s_{1:n}, \mathbf g_{1:n}) = \frac1n \sum_{i=1}^n V^\pi(s_i, g_i)$, which is $ \E[V^\pi(s,g)]$ for a uniformly random agent $i\in[n]$.

\textbf{Q-function.}
We introduce a shared $Q$-function that is optimized centrally.
Fix $\kappa\ge 2$ and let $\mathcal Z_\kappa$ denote the discrete set of empirical joint histograms on
$\mathcal S\times\mathcal A$ with denominator $\kappa$. Let $\mathcal G_\kappa \coloneqq \mathcal P_{\kappa}(\mathcal S)$.
For $z\in\mathcal Z_\kappa$, define its neighborhood state marginal $g_z\in\mathcal G_\kappa$ by $g_z(s) \coloneqq \sum_{a\in\mathcal A} z(s,a)$ for all $s\in  \mathcal S$. For a representative agent in state $s\in\mathcal S$ taking action $a\in\mathcal A$ and joint neighborhood histogram
$z\in\mathcal Z_\kappa$, let \looseness=-1
\begin{align*}
Q^\pi(s,a,z)\!=\!
r_\ell(s,a,g_z)
\!+\!\gamma\mathbb E_{(s',g')\sim \mathcal J_{n}(\cdot\mid s,a,z)}
\!\left[V^\pi\!\big(s',g'\big)\right],
\end{align*}
where $\mathcal J_{n}$ is the induced one-step kernel, for a uniformly random agent, on $(s,g)$ under the $n$-agent dynamics,
\begin{align*}
V^\pi(s,g)\coloneqq \mathbb E_{\substack{a, z \sim \pi(\cdot\mid s,g)}}
[Q^\pi(s,a,z)].
\end{align*} 
\textbf{Decentralized Execution.}
For $g\in\mathcal G_\kappa$, define the fiber $\Gamma_\kappa(g)\coloneqq \{z\in\mathcal Z_\kappa:\ g_z=g\}$ and the Bellman backup given by $\mathcal{M}_\kappa Q(s,g)\coloneqq \max_{a\in\mathcal A, z\in\Gamma_\kappa(g)} Q(s,a,z)$. The neighborhood histogram $z$ is a latent variable generated by the environment and determined by the other agents' states and actions. It is not a directly controllable parameter for individual agents. Thus the maximization over $z$ is not an action choice, but used as an optimization over latent variables to define a value function on $(s,g)$ that upper-bounds achievable returns under under all completions consistent with $g$. The optimal (greedy) policy therefore selects an action 
\[
\pi^*(\cdot\mid s,g) \in \mathcal{M}_\kappa Q^*(s,g) =  \arg\!\!\!\max_{a\in\mathcal A, z\in \Gamma_\kappa(g)} Q^*(s,a,z),
\]
with the maximization over $z$ serving to evaluate the action under the most favorable compatible neighborhood realization. In the fully centralized setting, optimal control would involve coordinating actions across all agents, which jointly determine the resulting neighborhood aggregate. During decentralized execution however, agents execute only the local action component and do not attempt to control or realize any particular completion $z$. \looseness=-1

Let $\mathcal{T}$ be the Bellman operator, where we then apply the Bellman update iteratively on each agent's Q-function. Execution is decentralized since each agent is able to locally construct an estimate $\hat{z}_i^{(\kappa)}$. Thus, rewards and transitions depend on neighbors only through the neighborhood state marginal $g_z$, enabling each agent to optimize using local observations without knowing explicit identity information.\looseness=-1
\begin{definition}[Bellman operator $\mathcal T$]
\textit{For $Q\in\mathcal Y_{n-1}$, define
\begin{align*}
\mathcal T Q(s,a,z)\!
&\coloneqq\!
r_\ell(s,a,g_z) \\
&+\gamma\mathbb E_{(s',g')\sim \mathcal J_{n}(\cdot\mid s,a,z)}
[\mathcal M_{n-1} Q(s',g')],
\end{align*}
where $\mathcal J_n(\cdot|s,a,z)$ is the induced one-step kernel on $(s',g')$ under the $n$-agent dynamics, for a uniformly random agent.}
\end{definition}

\begin{definition}[Sampled Bellman operator $\hat{\mathcal{T}}_\kappa$]\label{def: sampled bellman op}
\textit{For $\hat{Q}_\kappa \in\mathcal Y_\kappa$, define 
\begin{align*}
\widehat{\mathcal T}_\kappa \hat{Q}_\kappa(s,a,z)
&\coloneqq
r_\ell(s,a,g_z)
\\
&+\gamma\mathbb E_{(s',g')\sim \mathcal J_{\kappa}(\cdot\mid s,a,z)}
[\mathcal M_\kappa \hat{Q}_\kappa(s',g')],
\end{align*}
where $\mathcal J_\kappa(\cdot\mid s,a,z)$ is the induced one-step kernel on $(s',g')$ generated by the
$(\kappa+1)$-agent surrogate dynamics, for a uniformly random agent.}
\end{definition}

\begin{definition}[Empirical sampled operator $\widehat {\mathcal T}_{\kappa,m}$]\label{def: empirical sampled bellman op}
\textit{Given $m\in\mathbb N$, let $(s'_\ell,g'_\ell)_{\ell=1}^m$ be i.i.d.\ samples from $\mathcal J_\kappa(\cdot\mid s,a,z)$. Then, define
$$\widehat {\mathcal T}_{\kappa,m} \hat{Q}_{\kappa,m}(s,a,z)
\!\coloneqq\! r_\ell(s,a,g_z)
+\frac{\gamma}{m}\sum_{\ell=1}^m \mathcal M_\kappa \hat{Q}_{\kappa,m}(s'_\ell,g_{\ell}').$$}
\end{definition}
As a centralized-training decentralized-execution framework, \texttt{GMFS} proceeds as follows. First,~\cref{algorithm: GMFS offline} uses a generative oracle to derive an optimal $Q$-function on a $(\kappa + 1)$-agent surrogate model. Learning on this restricted subspace yields a policy governed by the subsampling parameter $\kappa$ instead of the total population size $n$. Subsequently,~\cref{algorithm: GMFS online} describes how agents deploy this learned policy in a decentralized environment. During execution, each agent $i$ independently approximates its neighborhood marginal $\hat{g}_i^{(\kappa)}$ by sampling $\kappa$ neighbors based on the graphon weights. This ensures that the computational cost per agent is independent of $n$, enabling scalable coordination in large-scale heterogeneous systems. \looseness-1

\begin{algorithm}[t]
\caption{GMFS (Graphon Mean-Field Subsampling): Offline Learning}
\begin{algorithmic}[1]
   \REQUIRE Number of iterations $T$, subsampling parameters $\kappa$ and $m$, and generative oracle \(\mathcal{O}\).
   \STATE Initialize \(\hat{Q}_{\kappa,m}^{(0)}(s, a, z) = 0, \forall (s, a, z) \in \cS\times\cA\times\cZ_{\kappa}\) 
   \FOR{\(t=1,\dots,T\)}
   \FOR{\((s, a, z) \in \mathcal{S}\times\mathcal{A}\times\mathcal{Z}_\kappa\)}
   \STATE Update \(\hat{Q}_{\kappa,m}^{(t+1)}(s,a,z) = \hat{\mathcal{T}}_{\kappa,m}\hat{Q}_{\kappa,m}^{(t)}(s,a,z)\)
   \ENDFOR
   \ENDFOR
   \STATE Return $\hat{Q}_{\kappa,m}^{(T)}$.
\end{algorithmic}
\label{algorithm: GMFS offline}
\end{algorithm}
\begin{algorithm}[hbt!]
\caption{GMFS (Graphon Mean-Field Subsampling): Online Execution}
\begin{algorithmic}[1]
    \REQUIRE Parameter $T'$ for length of the game, subsampling parameter $\kappa$, graphon weights $\{\bar{w}_{ij}\}$, discount factor $\gamma$, and learned policy $\pi_{\kappa,m}^{T} = \mathcal{M}_\kappa\hat{Q}_{\kappa,m}^{(T)}$.\looseness=-1
    \STATE Sample initial state $(s_1(0),\dots,s_n(0))\sim s_0$.
    \STATE Initialize total reward $R_0 = 0$.
    \FOR {$t = 0, \dots, T' - 1$}
    \FOR {$i = 1$ to $n$}
    \STATE Let $\Delta_i(t)\!=\!(J_i^1(t), \dots, J_i^\kappa(t))\!\stackrel{\text{(iid)}}{\sim}\!\bar{w}_{i,\cdot}$ on $[n]\!\setminus\!\{i\}$. 
    \STATE Compute subsampled graphon-weighted mean-field features for $x\in\mathcal{S}$\vspace{-0.2cm}
    \[
    \hat{g}_i^{(\kappa)}(x) = \frac{1}{\kappa} \sum_{m=1}^\kappa \mathbbm{1}\{s_{J_i^{(m)}(t)}(t) = x\}.\]\vspace{-0.2cm}
    \STATE Choose action
    $a_i(t) \sim \pi_{\kappa,m}^{T}(\cdot\mid s_i(t), \hat g_i^{(\kappa)}(t))$.
    \ENDFOR
    \STATE Get stage reward $\bar{R}_t\!\coloneqq\! r(\mathbf{s}_{1:n}(t), \mathbf{a}_{1:n}(t), \mathbf{g}_{1:n}(t))$.
    \STATE Let $s_i(t+1) \sim P(\cdot | s_i(t), a_i(t), g_i(t))$ for $i\in[n]$.
    \STATE $R_{t+1} = R_t + \gamma^t \cdot \bar{R}_t$.
    \ENDFOR
\end{algorithmic}
\label{algorithm: GMFS online}
\end{algorithm}

\section{Theoretical Guarantees and Analysis}\label{sec:theoretical guarantees and analysis_main}
We establish theoretical guarantees for \texttt{GMFS} by analyzing approximation errors due to neighborhood subsampling and finite-sample estimation. We define the properties of the sampled Bellman operators, and provide optimality bounds for the learned policy as a function of $\kappa$ and $m$.\looseness=-1

\textbf{Bellman noise.} We introduce the Bellman noise, $\epsilon_{\kappa, m}$, to account  for the error in estimating the operator from finite samples. The empirical operator $\hat{\mathcal{T}}_{\kappa,m}$ is an unbiased estimator of the sampled Bellman operator $\hat{\mathcal{T}}_\kappa$. As shown in Lemma~\ref{lemma: empirical sampled bellman operator is a gamma contraction}, both $\hat{\mathcal{T}}_\kappa$ and $\hat{\mathcal{T}}_{\kappa,m}$ are $\gamma$-contractions with fixed-points $\hat{Q}_\kappa^*$ and $\hat{Q}_{\kappa,m}^*$ respectively. By the law of large numbers,  $\smash{\lim_{m\to\infty}\hat{\mathcal{T}}_{\kappa,m}=\hat{\mathcal{T}}_\kappa}$ and $\smash{\|\hat{Q}_{\kappa,m}^* - \hat{Q}_\kappa^*\|_\infty\to 0}$ as $m\to \infty$. For finite $m$, we define this discrepancy as  $\epsilon_{\kappa,m}\coloneq\|\hat{Q}_{\kappa,m}^* - \hat{Q}_\kappa^*\|_\infty$.

Let ${\pi}^\est_{\kappa}$ be the corresponding greedy \texttt{GMFS} policy defined in Definition \labelcref{def: estimated gmfs policy}. First, noting that $\|V^{\pi^*} - V^{\pi_{\kappa,m}^{\text{est}}}\|_\infty\leq \epsilon$ implies $\|V_{\text{team}}^{\pi^*} - V_{\text{team}}^{\pi_{\kappa,m}^{\text{est}}}\|_\infty \leq \epsilon$, we show that the expected discounted cumulative reward produced by ${\pi}^\est_{\kappa,m}$ is approximately optimal, with an optimality gap that decays as the sampling parameters $\kappa$ and $m$ increase.

    \begin{theorem}
\label{probably main theorem}
For all states $s\in\mathcal{S}$ and graphon state-aggregates $g\in\mathcal{G}$, if $T\geq \frac{1}{1-\gamma}\log\frac{\|r_\ell\|_\infty \sqrt{\kappa}}{1-\gamma}$, then
        \begin{align*}
        V^{\pi^*} (s, g) &- V^{{\pi}^\est_{\kappa,m}} (s, g)
        \leq \frac{1}{20\sqrt{\kappa} (1-\gamma)} + \frac{\epsilon_{\kappa,m}}{1-\gamma} \\
        &+ \frac{2L_P \|r_\ell\|_\infty}{(1-\gamma)^2} \sqrt{\frac{|\cS|\ln 2\!+\!|\cA|\ln \frac{20\|r_\ell\|_\infty |\mathcal{A}|\kappa}{(1-\gamma)^2}}{2\kappa}}. 
        \end{align*}
    \end{theorem}
We generalize this result to stochastic rewards in Appendix~\ref{Appendix/stochastic}. To derive a final performance bound, we specify the number of samples $m$ needed to bound   $\epsilon_{\kappa,m}$.

\begin{lemma}[Controlling the Bellman Noise]\label{assumption:qest_qhat_error}
For $\kappa\in[n]$, let the number of samples in \cref{probably main theorem} be given by
\[m^* = \frac{25\kappa^2\gamma^2}{(1-\gamma)^4}\|r_\ell\|_\infty^2\cdot \ln (200|\cS|^2 |\cA|^2 \kappa^{|\cS||\cA|}).\]
If $T$ satisfies $\smash{T\geq\frac{2}{1-\gamma}\log \frac{\|r_\ell\|_\infty\sqrt{\kappa}}{1-\gamma}}$, then we have that 
\[\Pr\left[\epsilon_{\kappa,m^*} \leq \frac{1}{5\sqrt{\kappa}}\right] \geq 1 - \frac{1}{100e^\kappa}.\]
\end{lemma}

See Appendix \ref{subsection: bounding the bellman error} for the proof of Lemma \ref{assumption:qest_qhat_error}. Combining the approximation bounds of Theorem~\ref{probably main theorem} with the noise limits of Lemma~\ref{assumption:qest_qhat_error} and defining the resulting policy as ${\pi}_\kappa^\est \coloneq {\pi}_{\kappa,m^*}^\est$, we arrive at our main result in~\cref{actually main result}:

\begin{theorem} \label{actually main result} Suppose $T\geq\frac{2}{1-\gamma}\log\frac{\|{r}_\ell\|_\infty\sqrt{\kappa}}{1-\gamma}$ and the number of samples $m^*$ is chosen according to Lemma~\ref{assumption:qest_qhat_error}. Then, for all states $s\in\cS$ and graphon state aggregates $g\in\cG$, with probability\footnote{The $1/100e^\kappa$ term can be replaced by an arbitrary $\delta>0$ at the cost of attaching $\log1/\delta$ dependencies to the error bound.} at least $1 - 1/100 e^\kappa$,
\begin{align*}
    &V^{\pi^*}(s, g) - V^{{\pi}_{\kappa}^\est}(s, g) \leq \frac{1}{4\sqrt\kappa (1-\gamma)}\\
    &\quad\quad\quad\quad + \frac{2L_P \cdot \|r_\ell\|_\infty}{(1-\gamma)^2} \sqrt{\frac{|\cS|\ln 2 + |\cA|\ln \frac{20\|r_\ell\|_\infty |\mathcal{A}|\kappa}{(1-\gamma)^2}}{2\kappa}}.
    \end{align*}
\end{theorem}

\textbf{Sample complexity and optimality.} The efficiency of \texttt{GMFS} is reflected in its sample complexity. For a fixed $\kappa$,~\cref{algorithm: GMFS offline} learns $\hat{\pi}^\est_{k}$ with an asymptotic sample complexity of $\tilde{O}(\kappa^{|\cS||\cA|}|\cS|^2|\cA|^2)$, which is at least polynomially faster than standard $Q$-learning or mean-field value iteration. As shown in~\cref{probably main theorem}, the optimality gap decays as $\kappa \to n$, which reveals a fundamental trade-off: increasing $\kappa$ improves policy performance but increases the size of the $Q$-function. If we set $\kappa=O(\log n)$, the complexity becomes $\tilde{O}((\log n)^{|\cS||\cA|}|\cS|^2|\cA|^2)$. This is an exponential speedup over the complexity of mean-field value iteration, from $\mathrm{poly}(n)$ to $\mathrm{poly}(\log n)$, as well as over traditional value-iteration, where the optimality gap decays at a rate of $O(\frac{1}{\sqrt{\log n}})$. \looseness-1

There is evidence suggesting that the optimality gap of $\tilde{O}(1/\sqrt{\kappa})$ is sharp. An obstacle to improving this bound is the known optimal error of $\tilde{O}(1/\sqrt{n})$ in standard mean-field MARL. The algorithmic bottleneck in achieving a faster rate than $\tilde{O}(1/\sqrt{\kappa})$ comes from learning the $\hat{Q}_\kappa$ function rather than the online execution strategy. When $\kappa = n-1$, \texttt{GMFS} reduces to mean-field learning with a rate of $\tilde{O}(1/\sqrt{n})$, which matches the tight bound by \citet{yang2020meanfieldmultiagentreinforcement}. This illustrates a natural difficulty in improving the rate and provides evidence for why our bound is tight.\looseness-1

The $\mathrm{poly}(\frac{1}{1-\gamma})$-dependence in our results may be loose because we do not use more complicated variance reduction techniques as in \cite{sidford2018near,Sidford2018VarianceRV, wainwright2019variance, jin2024truncated} to optimize the number of samples $m$ used to bound the Bellman error $\epsilon_{\kappa,m}$. Incorporating variance reduction would significantly increase the complexity of the algorithm and the underlying intuition. Finally, the \texttt{GMFS} formulation extends to off-policy $Q$-learning~\cite{chen2021lyapunov}, which replaces the generative oracle with a stochastic approximation scheme to learn from historical data. This extension is detailed in  \cref{Appendix/off-policy}, where we provide theoretical guarantees with a similar decaying optimality gap. \looseness=-1

\textbf{Generalization to infinite state spaces.} In non-tabular environments with infinite state and action spaces, value-based RL methods can use function approximation to learn $\hat{Q}_\kappa$ via deep $Q$-networks~\citep{Silver_Huang_Maddison_Guez_Sifre_van_den_Driessche_Schrittwieser_Antonoglou_Panneershelvam_Lanctot_et_al._2016}. This introduces an additional error term in the performance bound of~\cref{probably main theorem}, which we analyze under a Linear MDP structure. \looseness-1

\begin{assumption}[Linear MDP with infinite state spaces] \label{def: linear mdp} \textit{Let $\cS$ be an infinite compact set, and assume a feature map $\phi:\mathcal{S}\times\mathcal{A}\times \cZ \to\mathbb{R}^d$, $d$ unknown (signed) measures $\mu = (\mu^1,\dots,\mu^d)$ over $\mathcal{S}$, and a vector $\theta\in\mathbb{R}^d$ such that for any $(s,a,z)\in \mathcal{S}\times \mathcal{A}\times\cZ$, we have $\mathbb{P}(\cdot|s,a,z) = \langle \phi(s,a,z),\mu(\cdot)\rangle$ and $r(s,a,g_z) = \langle \phi(s,a,z),\theta\rangle$.}
\end{assumption}\looseness-1

The existence of $\phi$ implies one can estimate the $Q$-function of any policy as a linear function. This assumption is used in policy iteration methods \cite{pmlr-v119-lattimore20a,wang2023centralizedtrainingdecentralizedexecution}, and we exploit it to obtain sample complexity bounds  independently of $|\cS|$ and $|\cA|$. As is standard in RL, we assume bounded feature-norms \citep{tkachuk2023efficientplanningcombinatorialaction}:
\begin{assumption}[Bounded features] \label{def: bounded features} We assume that $\|\phi(s,a,z)\|_2\leq 1$ for all $(s,a,z)\in\cS\times\cA\times\cZ$.
\end{assumption}

Following the reduction from~\citet{zhang2023provable,ren2024scalablespectralrepresentationsnetwork}, we use function approximation to learn spectral features $\phi_\kappa$ for $\hat{Q}_\kappa$. We derive a performance guarantee for the learned policy $\pi_\kappa^\est$, where the optimality gap decays with $\kappa$.\looseness-1
    
\begin{theorem}\label{thm: extended_lfa}When ${\pi}^\est_{\kappa}$ is derived from the spectral features $\phi_\kappa$ learned in $\hat{Q}_\kappa$, and $M$ is the number of samples used in the function approximation, then with probability at least $1 - \frac{1}{50\kappa} - \frac{201}{100\sqrt{\kappa}}$, we have
\begin{align*}
V^{\pi^*}(s, g)\!&-\!V^{{\pi}_{\kappa,m}}(s, g)\!\\
&\leq\!\tilde{O}\left(\sqrt{\frac{d\!+\!|\cA|}{\kappa}}\!+\!\frac{d }{\sqrt{M}}\!+\!\frac{2L_P \gamma \|r_\ell\|_\infty}{\sqrt{\kappa}}\right).
\end{align*}
\end{theorem}
Although $\frac{d}{\sqrt{M}}$ grows linearly with the dimension, it is controlled by the sample budget $M$ (i.e., chosen to scale with $\kappa$, e.g., $M \gtrsim \kappa^2 d^2$) so that it remains lower order relative to the $\tilde O(1/\sqrt{\kappa})$ terms.  We defer the proof of \cref{thm: extended_lfa} to Appendix \ref{appendix: extension to continuous state spaces}.\looseness-1

\section{Proof Outline}\label{sec:proof outline_main}

In this section, we provide an outline of the derivation for our main results through three steps: (1) proving Lipschitz stability of the Bellman iterates, (2) bounding subsampling error via concentration inequalities, and (3) establishing a global performance guarantee using the performance difference lemma.  See Appendices \labelcref{appendix: lipschitz cont in mf,appendix: performance gap}.

\textbf{Step 1: Lipschitz Continuity.}
We compare the $Q$-function evaluated under two neighborhood aggregates $z \in \mathcal{Z}$ and $\hat{z} \in \mathcal{Z}_{\kappa}$, whose state marginals $g_z$ and $g_{\hat z}$ differ in total variation. Specifically, we show:
\begin{theorem}[Lipschitz Continuity of the Bellman iterates]
\label{theorem: lip cont of bellman op} Fix a subsampling parameter $\kappa\geq 1$. Fix $(s,a)\in\mathcal{S}\times\mathcal{A}$, and let  $z\in \cZ\coloneqq \Delta(\cS\times\cA)$. Let $\hat{z}\in \mathcal Z_\kappa$ be the empirical histogram of $\kappa$ i.i.d. draws from $z$, and let $g_z, g_{\hat{z}}$ be the marginals in $\cS$. Then, for all $t\in\mathbb{N}$, we have
    \begin{align*}
        \left| Q^t(s,a,z) -  \hat{Q}^t_\kappa(s,a,\hat{z})\right| \leq  \frac{4 \|r_\ell\|_\infty }{1 - \gamma} L_P \cdot \mathrm{TV}(g_z, g_{\hat{z}}).
    \end{align*}
\end{theorem}
We defer the full proof of Theorem \labelcref{theorem: lip cont of bellman op} to Appendix \labelcref{appendix: lipschitz cont in mf}.

\textbf{Step 2: Concentration of the Subsampled Mean-Field.}
Next, we bound the discrepancy between the $\kappa$-sampled aggregate $\hat{g}_i^{(\kappa)}$ and the true graphon-weighted mean-field $g_i$. We establish a concentration inequality for empirical distributions drawn from a finite population to show that with probability at least $1-\delta$:
\[\left| Q^t(s,a,z)\!-\!\hat{Q}^t_\kappa(s,a,\hat{z})\right|\!\leq\!\frac{4  L_P \|r_\ell\|_\infty }{1 - \gamma}\!\sqrt{\frac{|\cS|\ln 2\!+\!\ln \frac{2}{\delta}}{2\kappa}}.\]
This result introduces the $O(1/\sqrt{\kappa})$ rate; the proof is provided in Appendix \labelcref{appendix: concentration}.

\textbf{Step 3: Performance Difference.}
Finally, we combine the previous steps to bound the performance gap between the learned policy $\pi_\kappa^{\mathrm{est}}$ from \texttt{GMFS} and the optimal policy $\pi^*$. Using Lemma \labelcref{Lemma: uniform bound}, we obtain a uniform bound on the optimality of the estimated joint policy in terms of the discrepancy between $Q^*$ and $\hat{Q}^*_\kappa$. We also account for the additional error introduced by finite-sample Bellman noise $\varepsilon_{\kappa, m}$. This allows us to apply the performance difference lemma \cite{Kakade2002ApproximatelyOA}, yielding Theorem \labelcref{actually main result}. The full proof is provided in Appendix \labelcref{appendix: performance gap}.\looseness-1

\section{Conclusion}\label{sec:conclusion-main}
In this work, we consider the problem of learning an optimal policy in a cooperative system of $n$ heterogeneous agents. We propose an algorithm, $\texttt{GMFS}$, which derives a policy $\tilde{\pi}_{\kappa}^{\mathrm{est}}$ where $\kappa \leq n$ is a tunable parameter for the number of agents sampled. We show that $\tilde{\pi}_\kappa^{\mathrm{est}}$ converges to the optimal policy $\pi^*$ with a decay rate of $O(1/\sqrt{\kappa})$. To establish this result, we develop an adapted Bellman operator $\hat{\mathcal{T}}_\kappa$ and prove its contraction property. The key technical novelty of this work lies in proving a Lipschitz continuity result for $\hat{Q}_\kappa^*$ and leveraging the weights of a corresponding graphon. Finally, we supplement our theoretical results with motivating examples and provide additional empirical validation (see Appendix \labelcref{appendix: numerics}) via numerical simulations of robotic control.

\textbf{Limitations and Future Work.} While \texttt{GMFS} is a scalable framework for heterogeneous MARL, there are several areas for extending its theoretical and practical scope. For instance, a natural next step is to derive matching lower bounds to demonstrate the tightness of our analysis for subsampling in graphon-weighted systems. Our analysis assumes access to a generative simulator; extending \texttt{GMFS} to a purely online setting with streaming interaction data and  exploration-exploitation trade-offs would extend its applicability to model-free environments. It would also be interesting to study higher-order interaction structures, such as Markov random fields or hypergraph-based mean-fields, to capture context-dependent heterogeneity \cite{wang2026unsuperviseddecompositionrecombinationdiscriminatordriven}. Another direction would be to examine whether online mirror descent can be integrated with our algorithm, as in \citet{fabian2023learningsparsegraphonmean}, to get better numerical stability. Finally, it would be interesting to adapt \texttt{GMFS}  beyond cooperative MARL to mixed cooperative-competitive environments or federated learning with non-uniform communication. 

\section*{Impact Statement}
This paper advances the field of machine learning by enabling more principled design and analysis of large-scale systems. Potential applications include autonomous driving, robotics, and other cooperative control settings. We believe these improvements will have a positive impact on society by improving the efficiency of autonomous infrastructure.

\section*{Acknowledgements}
This work was supported by NSF Grants CCF 2338816, CNS 2146814, CNS 2106403, CPS 2136197. SL acknowledges the support of the Kempner Institute Graduate Research Fellowship. We gratefully acknowledge insightful discussions with Yuzhou Wang, Sam van der Poel, Ishani Karmarkar, and Guannan Qu. \\

\nocite{hanashiro2020linearseparationoptimism,abernethy2022activesamplingminmaxfairness, anand2025the,lin2023online2}
\bibliography{main}
\bibliographystyle{icml2026}
\newpage
\appendix
\onecolumn

\textbf{Outline of the Appendices}.
\begin{itemize}
   \item Section \ref{sec: numerical simulations} provides our numerical simulations, experimental details, and supplemental examples;
    \item Section \ref{sec: mathematical background and additional remarks} provides mathematical background, and relevant  introductory lemmas;
    \item Section \ref{appendix: lipschitz cont in mf} presents the proof of the Lipschitz continuity of the $Q$-function in the mean-field measure;
    \item Section \ref{appendix: performance gap} bounds the optimality gap between the learned stochastic policy and the optimal policy;
    \item Section \ref{Appendix/stochastic} extends the result to stochastic rewards;
    \item Section \ref{Appendix/off-policy} extends the result to off-policy learning;
    \item Section \ref{appendix: extension to continuous state spaces} extends the result to continuous state spaces.
\end{itemize}

\subsection*{Notation}
\begin{tabular}{ l| l }
\hline
\textbf{Symbol} & \textbf{Description} \\
\hline
$\mathbb{R}$ & Space of real numbers \\
$\mathbb{S}^p$ & The $p$-dimensional sphere that forms the boundary of a ball in $\mathbb{R}^{(p+1)}$ \\
$n$ & Number of agents in the system \\
$\mathbb{G}=(\mathcal{V},\mathcal{E})$ & $|V|$-agent interaction graph \\
$\mathcal{S}, \mathcal{A}$ & Agent state space and agent action space\\
$\mathcal{G}, \mathcal{H}$ & Neighboring agent state space and action space \\
$r_\ell (s,a,g)$ & Local reward for each agent given state, action, and mean-field \\
$r(s,a,g)$ & Team reward $\frac{1}{n}\sum_{i=1}^n r_\ell(\cdot)$ \\
$\bar w_{ij}$ & Normalized interaction weight, $\bar w_{ij} = \frac{w_{ij}}{\sum_{m\neq i} w_{im}}$ \\
$\Delta_i$ & Multi-set of $\kappa$ neighbors sampled for agent $i$\\
$\Gamma_\kappa(g)$ & Fiber $\{z\in\mathcal Z_\kappa:\ g_z=g\}$ that completes a neighborhood aggregate $g$\\
$z \in \mathcal{Z} \coloneqq \Delta(\cS\times\cA)$ & Graphon-weighted neighborhood state-action distribution\\
$\hat{z} \in \mathcal{Z}_\kappa$ & Empirical histogram of $\kappa$ i.i.d. draws from $z$\\
$g_z, g_{\hat{z}}$ & State marginals in $\mathcal{S}$\\
$h_{z_i}$ & Action marginal $\sum_{x\in \mathcal{S}} z_i(x,u)\!\in\!\mathcal{H}$\\
$\hat z_i^{(\kappa)}$ & Empirical neighborhood state-action histogram from $\Delta_i$ \\
$\hat g_i^{(\kappa)}$ & Empirical neighborhood state marginal of $\hat z_i^{(\kappa)}$ \\
$\mathcal J_n$ & Induced one-step kernel on $(s',g')$ under the $n$-agent dynamics.\\
$\mathcal J_\kappa$ & Induced one-step kernel on $(s',g')$ generated by the $(\kappa+1)$-agent surrogate dynamics.\\
$Q^\pi(s,a,z), V^\pi(s,g)$ & Centralized critic for agent $i$ and value function for agent $i$ under policy $\pi$. \\
$Q^*,V^*$ & Optimal centralized $Q$- and $V$-functions \\
$\hat Q^*_\kappa$ & Optimal $Q$-function under $\kappa$-sampled mean-field dynamics \\
$\hat Q^*_{\kappa,m}$ & Empirical estimate of $\hat Q^\ast_\kappa$ using $m$ samples \\
$\varepsilon_{\kappa,m}$ & Bellman noise: $\|\hat Q^*_{\kappa,m} - \hat Q^*_\kappa\|_\infty$ \\
$\mathcal{T}$ & Bellman operator for agent $i$ (true mean-field) \\
$\hat{\mathcal{T}}_\kappa$ & Sampled Bellman operator (using $\kappa$-subsampled aggregates) \\
$\hat{\mathcal{T}}_{\kappa, m}$ & Empirical Bellman operator using $m$ samples.\\
$\pi^*$ & Optimal joint policy\\
$\pi^{\mathrm{est}}_{\kappa,\Delta}$ & Estimated joint policy under graphon-weighted subsampling\\
$\pi^{\mathrm{est}}_\kappa$ & Learned policy\\
$\|\cdot\|_\infty$ & $\ell_\infty$-norm \\
 $\gamma\in(0,1)$ & Discount factor \\
$P$ & Transition kernel for an agent's next state, $s_i \sim P(\cdot\mid s_i,a_i,g_i)$ \\
\hline
\end{tabular}

\clearpage

\section{Numerical Simulations and Additional Motivating Examples}
\label{sec: numerical simulations}

In this section, we provide an additional conceptual formulation of \texttt{GMFS} within the context of smart grid management. We then provide an empirical evaluation of the algorithm using a cooperative robotics coordination task.

\subsection{Motivating Example: Cooperative Autonomous Driving}
\label{appendix: vehicles}

Consider a population of $n$ autonomous vehicles indexed by $i \in \{1,\dots, n\}$. Each agent $i$ is associated with a latent feature $\alpha_i \in [0,1]$ representing its position along a road segment. Interactions are governed by a graphon $W: [0,1]^2 \to [0,1]$, where $W(\alpha_i, \alpha_j)$ represents the strength of the interaction between vehicles indexed by $i$ and $j$. For example, we model distance-decaying influence along the road with: $W(\alpha_i, \alpha_j) = \exp (-\beta|\alpha_i - \alpha_j|)$. Each agent has state $s_i \in \mathcal{S}$, which can encode the position or velocity, and selects actions $a_i \in \mathcal{A}$, which can include steering and acceleration. Vehicle $i$ transitions according $s_i(t+1)\sim P\big(\cdot\mid s_i(t),a_i(t),g_i(t)\big)$. At each time $t$, vehicle $i$ constructs a multiset of $\kappa$ neighbors $\Delta_i(t)=(J_i^{(1)}(t),\dots,J_i^{(\kappa)}(t))$ and forms the empirical histograms:
\begin{align*}
    \hat Z_i^{(\kappa)}(t)(s,a)
    &=
    \frac{1}{\kappa}\sum_{m=1}^\kappa
    \mathbbm 1\!\left\{s_{J_i^{(m)}(t)}(t)=s, a_{J_i^{(m)}(t)}(t)=a\right\},\\
    \hat g_i^{(\kappa)}(t)(s)
    &=
    \sum_{a\in\mathcal A}\hat Z_i^{(\kappa)}(t)(s,a).
\end{align*}
Since the distance-decay graphon induces decaying weights, the dominant contributors in this example to $g_i(t)$ are the nearby vehicles. Thus, sampling from $\bar{w}_{ij}$ targets the most informative neighbors. Increasing $\kappa$ tightens concentration of $\hat{g}_i^{(\kappa)}(t)$ while keeping the per-agent computation scalable. 

\subsection{Motivating Example: Cooperative Robot Coordination Task}
\label{app: robot coordinating}
Consider $n$ mobile robots indexed by $i\in[n]$ operating in a warehouse.
Associate each robot with a latent coordinate $\alpha_i\in[0,1]$ capturing its
location along an embedding of aisles/loading zones. Let $W:[0,1]^2\to[0,1]$ encode interaction intensity. We consider a radius graphon $W(\alpha,\beta)=\mathbbm 1\{|\alpha-\beta|\le r\}$ for some $r > 0$, so that robots primarily interact with others in nearby regions. Let $\mathcal S$ be a finite state space encoding each robot's local mode and task-relevant information, such as idle, moving, standing. Let $\mathcal A$ be a finite action space which includes move, wait, pickup. At time $t$, robot $i$ observes its local state $s_i(t)$ and a
neighborhood summary, selects an action $a_i(t)\in\mathcal A$, and transitions
according to $s_i(t+1)\sim P\big(\cdot \mid s_i(t),a_i(t),g_i(t)\big)$ where $g_i(t)$ is the graphon-weighted neighborhood state feature. In decentralized execution, robot $i$ does not enumerate all neighbors. The sampled local reward is evaluated as $r_\ell(s_i(t),a_i(t),\hat g_i^{(\kappa)}(t))$,
capturing congestion/collision-risk effects induced by nearby robots. The interaction locality implies that the most relevant information for robot $i$ is concentrated among a small subset of agents with large weights $\bar{w}_{ij}$. Graphon-weighted subsampling therefore preserves the dominant neighborhood statistics while keeping per-agent computation independent of $n$.\looseness=-1

\subsection{Motivating Example: Energy Distribution for a Smart Grid}

Modern smart grids present a natural application for \texttt{GMFS} due to their high topological heterogeneity and decentralized nature. Unlike traditional power systems, a substation's demand response in a smart grid is heavily influenced by immediate topological neighbors and local transmission constraints rather than the global average of the entire grid. In this conceptual setting, we consider $n$ agents representing local substations, or energy consumers, indexed by $i \in [n]$. 

One could model this environment by labeling each agent with a latent position $\alpha_i \in [0,1]$, which represents geographical coordinates or a position within an infrastructural hierarchy. The interaction graphon $W: [0,1]^2 \to [0,1]$ would encode the transmission efficiency and physical connectivity between substations. As noted by \citet{deep_iot}, renewable energy integration requires minimizing transmission losses, which are inherently non-uniform. A graphon structure would capture this by assigning higher weights to substation pairs with high-capacity links or proximity. To ensure the resulting neighborhood sampling distribution is well-defined, we assume $\int_0^1 W(\alpha, \beta) d\beta > 0$ for all $\alpha \in [0,1]$, ensuring every node has a non-zero interaction density.

Under this formulation, each agent $i$ maintains a state $s_i \in \mathcal{S}$ representing its current load or energy deficit and selects an action $a_i \in \mathcal{A}$ to request a specific allocation from a shared supply. At each step, agent $i$ would reconstruct a graphon-weighted subsampled aggregate $z_i \in \mathcal{Z}_\kappa$ (Def.~\ref{def: sampled neighborhood agg}). The reward function $r_\ell(s_i, a_i, g_i)$ can be designed to penalize ``over-allocation'' costs and transmission line heating, which are non-linear functions of the local demand density. Ultimately, this example shows how \texttt{GMFS} could enable substations to make near-optimal allocation decisions by observing only $\kappa \ll n$ neighboring nodes, which avoids the need for a centralized controller to process the state of the entire national grid.

\subsection{Evaluation on Cooperative Robot Coordination Task}

\label{appendix: numerics}

We evaluate the empirical performance of \texttt{GMFS} on a cooperative robot coordination task in \labelcref{app: robot coordinating} within a spatially constrained warehouse environment. This domain is a natural fit for graphon models because robot interactions are inherently local; a robot is significantly more affected by collision risks or blocked aisles in its immediate vicinity than by the status of a robot on the opposite side of a facility.

A key difficulty in such systems is the \textit{perception-action gap}, where agents must make decisions based on incomplete social information. In Figure~\ref{fig:perception-time}, we visualize the time-evolution of an agent's perception. At low $\kappa$, the agent's view of its neighborhood is sparse and noisy, resulting in high-variance estimates of local congestion. As $\kappa$ increases, the agent's empirical measure $\hat{g}_i$ converges to the true geometric ball defined by the graphon. In Figure~\ref{fig:robotics-performance}, we demonstrate that \texttt{GMFS} allows agents to achieve near-optimal coordination even when sampling only a small fraction of the total population.

\begin{figure*}[hbt!]
    \centering
    \begin{subfigure}[t]{0.47\textwidth}
        \centering
        \includegraphics[width=1.1\linewidth]{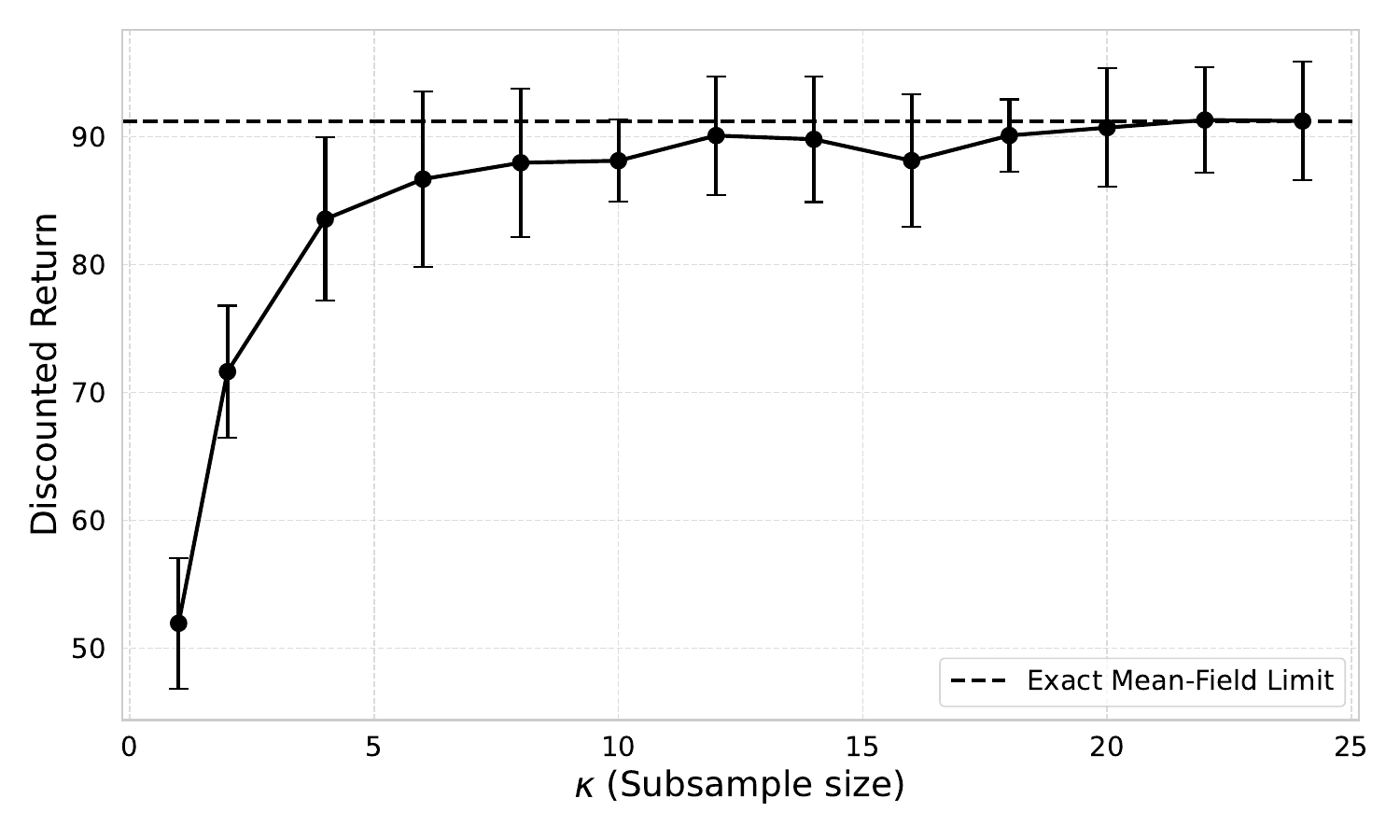}
        \caption{Robotics control performance: average cumulative discounted reward $V^\pi$ against subsample size $\kappa$ (30 runs).}
        \label{fig:robotics-performance}
    \end{subfigure}
    \hfill
    \begin{subfigure}[t]{0.47\textwidth}
        \centering
        \includegraphics[width=1.1\linewidth]{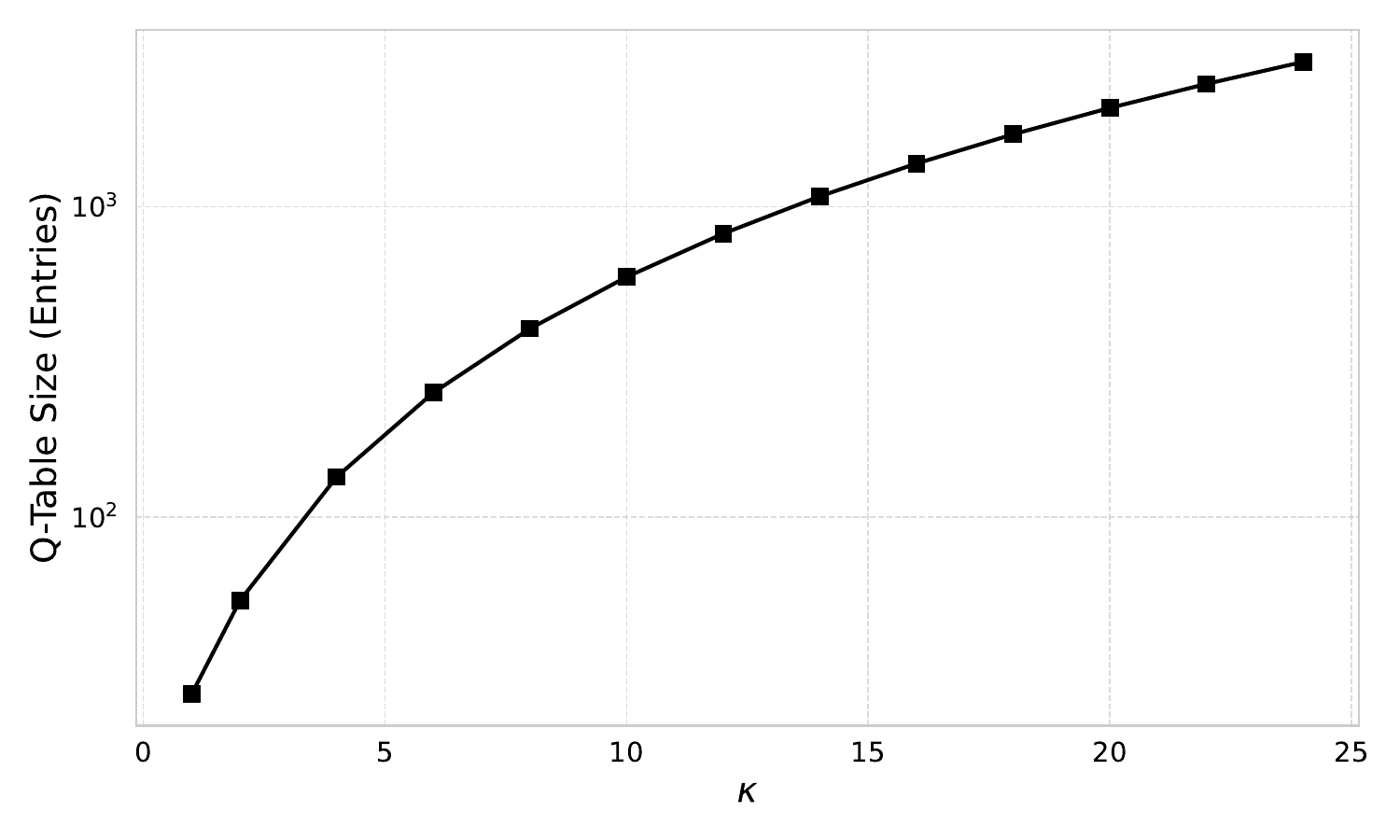}
        \caption{Computational complexity scaling: Q-table size (total entries) as a function of the subsample size $\kappa$.}
        \label{fig:complexity-scaling}
    \end{subfigure}
    \caption{
    \textbf{Performance-scalability tradeoff of GMFS.}
    (Left) \texttt{GMFS} rapidly achieves near-optimal performance in the robotics coordination task starting from around $\kappa = 8$, approaching the full graphon mean-field baseline at $\kappa = 24$ (which corresponds to the optimal solution obtained without sampling) \citep{hu2022graphonmeanfieldcontrolcooperative,fabian2022meanfieldgamesweighted}.
    (Right) The computational cost, measured by the number of entries in the discrete neighborhood state space $\mathcal{Z}_\kappa$, grows polynomially in $\kappa$.
}
    \label{fig:robotics-and-complexity}
\end{figure*}

\begin{figure}
    \centering
    \includegraphics[width=1.06\linewidth]{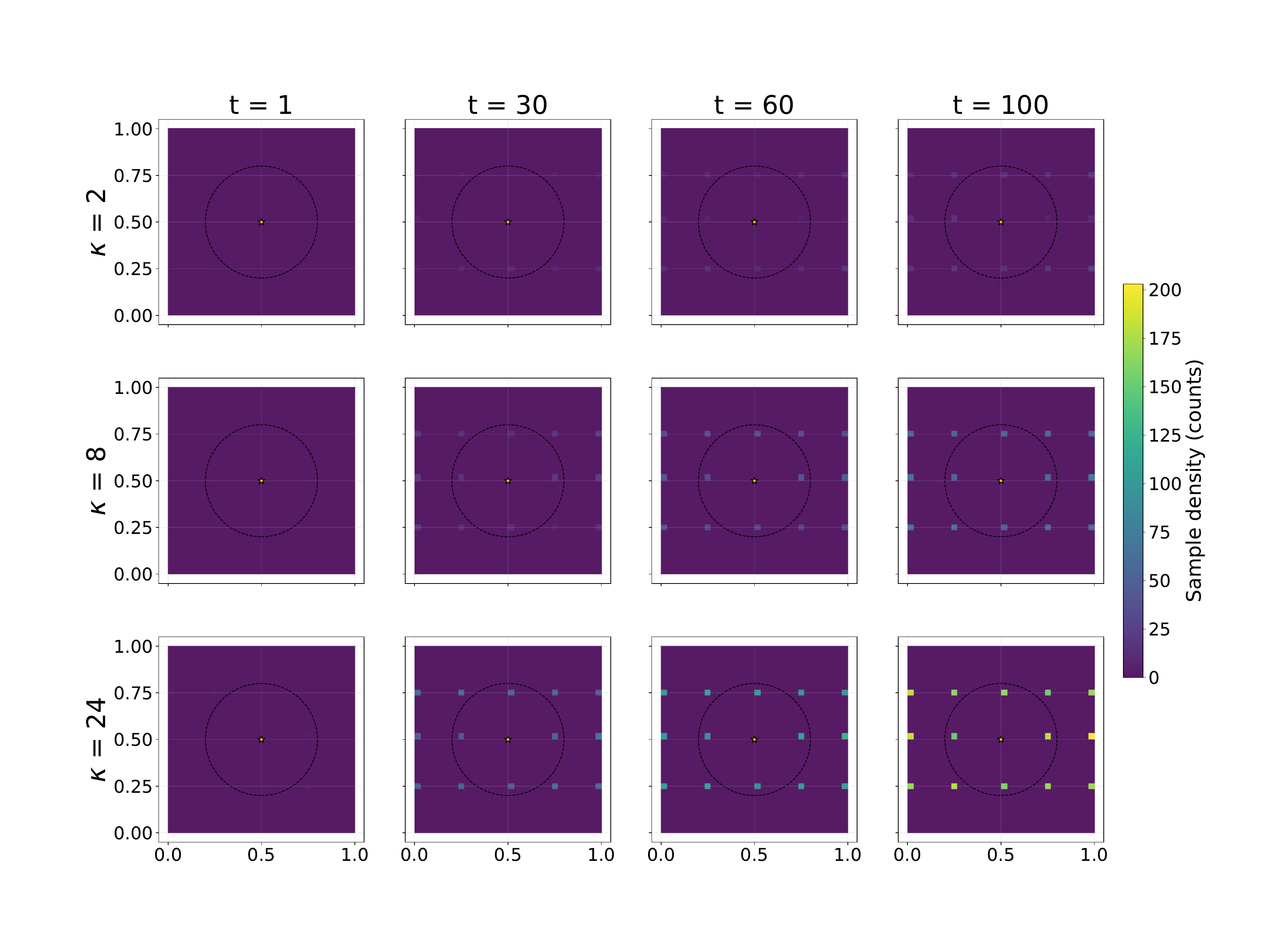}
    \caption{Perception time‑evolution comparison for the focal agent (center of the $5 \times 5$ grid) under the radial graphon. Rows correspond to $\kappa \in \{2, 8, 24\}$, while columns correspond to time horizons $t \in \{1, 30, 60, 100\}$. Each panel aggregates the focal agent’s sampled neighbors up to time $t$, with the dashed circle indicating the true interaction ball. As $\kappa$ increases, the empirical neighborhood density converges faster and more uniformly to the support of the radial graphon.}
    \label{fig:perception-time}
\end{figure}

\begin{figure}
    \centering
    \includegraphics[width=0.6\linewidth]{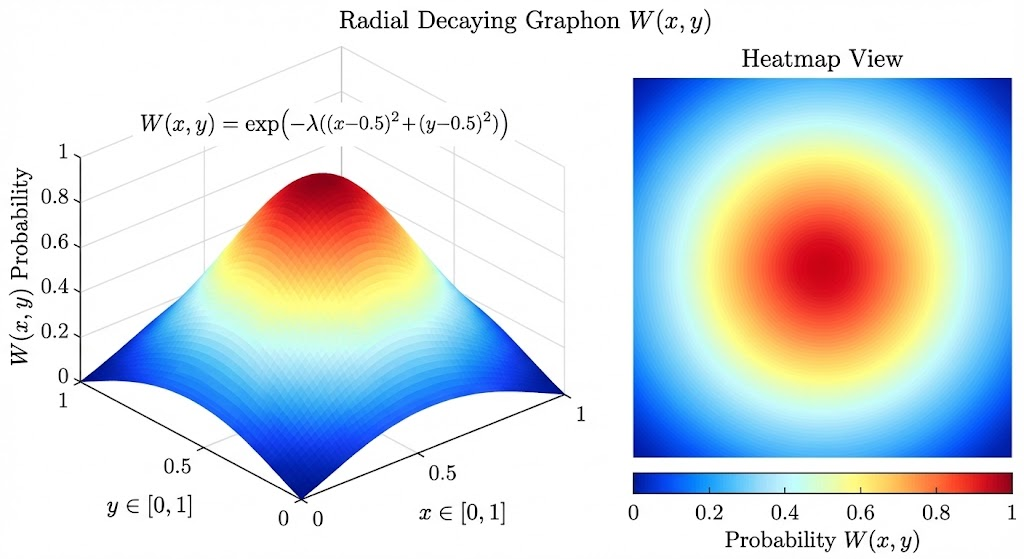}
    \caption{Visualization of a radial graphon. We note that generative AI was used to refine the aesthetics of this figure.}
    \label{fig:radial_graph}
\end{figure}

\subsubsection{Experimental Setup}
\textbf{Environment.} The robotics environment consists of $n=25$ agents initialized on a fixed $5 \times 5$ spatial initialization grid ($G$) within a unit square $[0, 1]^2$. The agents operate within a state space $\mathcal{S} = \{0, 1, 2\}$ and an action space $\mathcal{A} = \{0, 1, 2\}$, corresponding to \textit{idle} (0), \textit{transit} (1), and \textit{working} (2) states. Actions represent the agent's intended next state. Each evaluation run is conducted over a horizon of $H=100$ time steps.

\textbf{Compute.} Experiments were implemented in Python and ran on a high-performance computing node with Intel CPUs and 1.0 TiB of DDR5 ECC RAM. We parallelized the training across subsampling parameters; the total suite took approximately 20 minutes to execute, with the exhaustive mean-field limit ($\kappa = 24$) requiring approximately 440 seconds for training.

\subsubsection{Dynamics and Rewards}
We define the interaction topology using a radial graphon $W(x,y)$ connecting agents based on their fixed latent positions $\alpha_i, \alpha_j \in [0,1]^2$: $W(\alpha_i, \alpha_j) = \mathbbm{1}(\|\alpha_i - \alpha_j\|_2 \leq r),$ with an interaction radius $r=0.3$. We row-normalize these weights to induce the sampling distribution $\bar{w}_{ij} = W_{ij} / \sum_{l \neq i} W_{il}$. To handle the edge case of an isolated agent (where the row sum is zero), we define $\bar{w}_{ij}$ as the uniform distribution over all other agents $[n] \setminus \{i\}$.\looseness-1

The stochastic transition kernel $P(s' \mid s, a, \mu)$ simulates physical bottlenecks by making transitions to the working state ($a=2$) congestion-dependent; the success probability is defined as $P(s'=2 \mid a=2, \mu) = \max(0.1, 0.9 - 0.8 \cdot \mu(2))$, with failures reverting the agent to the transit state ($s'=1$). For all other actions ($a \in \{0, 1\}$), the agent successfully transitions to the intended state with probability 0.9, i.e., $P(s'=a \mid a, \mu) = 0.9$ and with a 0.1 probability of failure leaving the agent's state unchanged ($s'=s$).\looseness-1

To satisfy the Lipschitz and positive rewards assumptions in our framework, we define the reward function as:
\begin{equation}r(s, a, \mu) = V(s) \cdot \max(0.4, 1.0 - L \cdot \mu(2)) - C(a).\end{equation}
The state values $V(s)$ are set to 10, 5, and 20 for the \textit{idle}, \textit{transit}, and \textit{working} states, respectively. Action costs $C(a)$ are 0 for \textit{idle} and \textit{transit} and 5.0 for the \textit{working} state. We set the congestion sensitivity to $L=5.0$ with a minimum utility multiplier of 0.4. \looseness-1

\subsubsection{GMFS Algorithm Configuration}
We learn the optimal $Q$-function using offline value iteration as described. In order to maintain computational tractability for large $\kappa$, we use the state-marginal histograms, which reduces the neighborhood state space $\mathcal{Z}_\kappa$ from $\binom{\kappa+8}{8}$ (joint) to $\binom{\kappa+2}{2}$ (marginal) states. This is because the environment's rewards and transitions depend only on the neighborhood state marginal $\mu$, so the state-marginal histogram is a sufficient statistic for the optimal $Q$-function in this setting. We check convergence using the Bellman error $\Delta = \|Q_{t+1} - Q_t\|_\infty$ and find that all runs reach a stable fixed point ($\Delta < 10^{-4}$) in 250 iterations. The hyperparameters used are detailed in Table~\ref{tab:hyperparams}.

\begin{table}[hbt!]
\centering
\begin{tabular}{@{}ll@{}}
\toprule
Hyperparameter & Value \\ \midrule
Discount Factor ($\gamma$) & 0.95 \\
Training Iterations ($T$) & 250 \\
Subsample Sizes ($\kappa$) & $\{1, 3, 6, 9, 12, 15, 18, 21, 24\}$ \\
Monte Carlo Samples ($M$) & 50 (per Bellman update) \\
Learning Rate ($\alpha$) & 1.0 (Exact Operator Update) \\
Exploration Rate ($\epsilon$) & 0.0 (Exhaustive Offline Sweep) \\
Independent Runs & 30 seeds per $\kappa$ \\ \bottomrule
\end{tabular}
\caption{GMFS Training Configuration}
\label{tab:hyperparams}
\end{table}

\section{Mathematical Background and Additional Remarks}
\label{sec: mathematical background and additional remarks}

In this section, we focus on the mathematical foundations and auxiliary results required for our analysis. We first establish the contraction properties and uniform boundedness of the Bellman operators to ensure the existence and uniqueness of the optimal $Q$-functions. We also present an identity-aware sampling algorithm to contrast the efficiency of our mean-field reparameterization against an exhaustive labeled state representation.

\begin{definition}
    [Lipschitz continuity]
    Given metric spaces $(\mathcal X, d_1)$ and $(\mathcal Y, d_2)$ and a constant $L>0$, a map $f:\mathcal X\to\mathcal Y$ is $L$-Lipschitz continuous if for all $x,y\in\mathcal X, d_2(f(x), f(y)) \leq L d_1(x,y)$.
\end{definition}

\begin{theorem}
    (Banach-Caccioppoli Fixed Point Theorem \cite{Banach1922}). Let $(X, d)$ be a non-empty complete metric space with a $\gamma$-contraction
    mapping $T : X \to X$ for $\gamma\in(0,1)$. Then, $T$ admits a unique fixed point $x^* \in X$ such that $T(x^*) = x^*$. Moreover, for any $x_0 \in X$, $T^n(x_0) \to x^*$ as $n \to \infty$. 
\end{theorem}
 
\begin{lemma}[The Bellman operator and sampled Bellman operator are $\gamma$-contractions] Fix $1\leq \kappa\leq n-1$. Let $\mathcal Y_\kappa \coloneqq \{Q:\mathcal S\times\mathcal A\times\mathcal Z_\kappa \to\mathbb R\}$ be equipped with the sup norm
$\|Q\|_\infty \coloneqq \sup_{(s,a,z)}|Q(s,a,z)|$. Suppose $\gamma\in(0,1)$.
Then the Bellman operator $\hat{\mathcal T}_\kappa:\mathcal Y_\kappa\to\mathcal Y_\kappa$ admits a unique fixed point $Q^*$ satisfying $\hat{\mathcal T}_\kappa Q^* = Q^*$. 
\end{lemma}

\begin{proof} We show that $\hat{\mathcal T}_\kappa$ is a $\gamma$-contraction such that for all $Q_1,Q_2\in\cY_\kappa$, $\|\hat{\mathcal T}_\kappa Q_1-\hat{\mathcal T}_\kappa Q_2\|_\infty \le \gamma \|Q_1-Q_2\|_\infty$. Fix any $Q_1,Q_2\in\mathcal Y_\kappa$ and any $(s,a,z) \in \mathcal{S}\times\mathcal{A}\times\mathcal{Z}_\kappa$, we have
\begin{align*}
    |\hat{\mathcal T}_\kappa Q_1(s, a, z) &- \hat{\mathcal T}_\kappa Q_2(s, a, z)| \\
    &= \left|r_\ell(s,a,g_z) + \gamma\mathbb E_{(s',g')\sim \mathcal J_{\kappa}(\cdot\mid s,a,z)}
[\mathcal M_\kappa Q_1(s',g')] - r_\ell(s,a,g_z) - \gamma\mathbb E_{(s',g')\sim \mathcal J_{\kappa}(\cdot\mid s,a,z)}
[\mathcal M_\kappa Q_2(s',g')]\right| \\
    &= \left|  \gamma \mathbb{E}_{\substack{(s',g')}}
\left[\max_{\substack{a'\in\mathcal{A}, z' \in \Gamma_\kappa(g')}} Q_1(s',a', z')\right] -  \gamma \mathbb{E}_{\substack{(s',g')}}
\left[\max_{\substack{a'\in\mathcal{A}, z' \in \Gamma_\kappa(g')}} Q_2(s',a', z')\right] \right| \\
    &\leq \gamma \mathbb{E}_{\substack{(s',g')}} \left|  
\max_{\substack{a'\in\mathcal{A}, z' \in \Gamma_\kappa(g')}} Q_1(s',a', z') -  \max_{\substack{a'\in\mathcal{A}, z' \in \Gamma_\kappa(g')}} Q_2(s',a', z') \right|  \\
    &\leq \gamma  \|Q_1  -Q_2\|_\infty,
\end{align*}

where the  the third line follows by Jensen's inequality, and the last line follows by the $1$-Lipschitzness of the max operator under $\ell_\infty$-norm. Taking the supremum over $(s,a,z)$ yields $\|\hat{\mathcal T} _\kappa Q_1- \hat{\mathcal T}_\kappa Q_2\|_\infty\le \gamma\|Q_1-Q_2\|_\infty$.
Since $(\mathcal Y_\kappa,\|\cdot\|_\infty)$ is complete, Banach's fixed point theorem implies that $\hat{\mathcal T_\kappa}$ has a unique fixed point. The result follows for the original Bellman operator by taking $\kappa=n-1$, where $\hat{\mathcal{T}}_{n-1}=\mathcal{T}$.
\end{proof}

\begin{lemma}[The empirical sampled Bellman operator is a $\gamma$-contraction] \label{lemma: empirical sampled bellman operator is a gamma contraction}  Fix $1\leq \kappa\leq n-1$. Let $\mathcal Y_\kappa \coloneqq \{Q:\mathcal S\times\mathcal A\times\mathcal Z_\kappa \to\mathbb R\}$ be equipped with the sup norm
$\|Q\|_\infty \coloneqq \sup_{(s,a,z)}|Q(s,a,z)|$. Suppose $\gamma\in(0,1)$.
Then the sampled Bellman operator $\hat{\mathcal T}_{\kappa,m}:\mathcal Y_\kappa\to\mathcal Y_\kappa$ admits a unique fixed point $Q^*$ satisfying $\hat{\mathcal T}_{\kappa,m} \hat{Q}_{\kappa,m}^* = \hat{Q}_{\kappa,m}^*$. 
\end{lemma}
\begin{proof}
    We show that $\hat{\mathcal T}_{\kappa,m}$ is a $\gamma$-contraction such that for all $Q_1,Q_2\in\cY_\kappa$, $\|\hat{\mathcal T}_{\kappa,m} Q_1-\hat{\mathcal T}_{\kappa,m} Q_2\|_\infty \le \gamma \|Q_1-Q_2\|_\infty$. Fix any $Q_1,Q_2\in\mathcal Y_\kappa$ and any $(s,a,z) \in \mathcal{S}\times\mathcal{A}\times\mathcal{Z}_\kappa$, we have
\begin{align*}
    |\hat{\mathcal T}_{\kappa,m} Q_1(s, a, z) &- \hat{\mathcal T}_{\kappa,m} Q_2(s, a, z)| \\
    &=  \left|r_\ell(s,a,g_z) + \frac{\gamma}{m}\sum_{\ell=1}^m \mathcal{M}_\kappa Q_1(s',g') - r_\ell(s,a,g_z) - \frac{\gamma}{m}\sum_{\ell=1}^m \mathcal M_\kappa Q_2(s',g')\right| \\
    &\leq \frac{\gamma}{m}\sum_{\ell=1}^m\left|\max_{\substack{a'\in\mathcal{A}, z' \in \Gamma_\kappa(g')}} Q_1(s',a', z') - \max_{\substack{a'\in\mathcal{A}, z' \in \Gamma_\kappa(g')}} Q_2(s',a', z')\right| \\
    &\leq \gamma  \|Q_1  -Q_2 \|_\infty,
\end{align*}
where the  the second line follows by triangle inequality, and the last line follows by the $1$-Lipschitzness of the max operator under $\ell_\infty$-norm. Taking the supremum over $(s,a,z)$ yields $\|\hat{\mathcal T}_{\kappa,m} Q_1- \hat{\mathcal T}_{\kappa,m} Q_2\|_\infty\le \gamma\|Q_1-Q_2\|_\infty$.
Since $(\mathcal Y_\kappa,\|\cdot\|_\infty)$ is complete, Banach's fixed point theorem implies that $\hat{\mathcal T}_{\kappa,m}$ has a unique fixed point. Labeling the fixed point $\hat{Q}_{\kappa,m}^*$ completes the proof.
\end{proof} 

We next show that the $Q$-function is bounded throughout its iterations.
\begin{lemma} For all $T\geq 0$,
    $\|Q^T\|_\infty \leq \frac{\|r\|_\infty}{1-\gamma}$.\label{lemma: q function bound}
\end{lemma}
\begin{proof}
    The proof follows by induction on $T$. The base case follows as $Q^0\coloneqq 0$. For the induction, note that by triangle inequality $\|Q^{T+1}\|_\infty \leq \|r_\ell\|_\infty + \gamma \|Q^T\|_\infty \leq \|r_\ell\|_\infty + \gamma \frac{ \|r_\ell\|_\infty}{1-\gamma} = \frac{\|r_\ell\|_\infty}{1-\gamma}$,
    which proves the claim.\qedhere\\
\end{proof}

\begin{corollary}\label{corollary: unrolling bound}
    Observe by recursively using the $\gamma$-contractive property for $T$ time steps, with the initializations $\hat{Q}_\kappa = 0$ and $\hat{Q}_{\kappa,m} = 0$, and the bounds  $\|\hat{Q}_{\kappa}^*\|_\infty \leq \frac{\|{r}_\ell\|_\infty}{1-\gamma}$ and $\|\hat{Q}_{\kappa,m}^*\|_\infty \leq \frac{\|{r}_\ell\|_\infty}{1-\gamma}$ from Lemma \ref{lemma: q function bound}, that
    \begin{equation}\|\hat{Q}_\kappa^* - \hat{Q}_{\kappa}^T\|_\infty \leq \gamma^T \cdot \|\hat{Q}_\kappa^* - \hat{Q}_{\kappa}^0\|_\infty  \leq \gamma^T \frac{\|r_\ell\|_\infty}{1-\gamma},\end{equation}
    and
    \begin{equation}
    \|\hat{Q}_{\kappa,m}^* - \hat{Q}_{\kappa,m}^T\|_\infty \leq \gamma^T \cdot \|\hat{Q}_{\kappa,m}^* - \hat{Q}_{\kappa,m}^0\|_\infty \leq \gamma^T \frac{\|r_\ell\|_\infty}{1-\gamma}.\\
        \end{equation}
\end{corollary}

\begin{remark}
    \cref{corollary: unrolling bound} characterizes the error decay between $\hat{Q}_\kappa^T$ and $\hat{Q}_\kappa^*$ and shows that it decays exponentially in the number of Bellman iterations by a $\gamma^T$ multiplicative factor.\\
\end{remark}

\subsection{Identity-Aware Sampling Algorithm and Analyses}

We present an identity-aware subsampling algorithm that explicitly retains the indices of sampled neighbors within the state representation. We provide this as a simple baseline to demonstrate how subsampling and graphon-based weights can be used to construct an unbiased estimator of neighborhood-dependent rewards and transitions. The numerical stability of this approach follows the framework from \citet{anand2024bitcomplexitydynamicalgebraic}. We use this baseline to motivate the statistical guarantees of our main algorithm; specifically, we show that while identity-aware representations provide unbiasedness, their state-action space cardinality grows exponentially with $\kappa$. This motivates the mean-field reparameterization introduced in the main text.

\begin{algorithm}[H]
\begin{algorithmic}
\caption{Offline Learning (Identity-aware subsampling + Graphon-based Importance Sampling)}
\REQUIRE{Subsampling parameter $\kappa \ge 1$, iterations $T$, Monte Carlo samples $m$, proposal distribution $q_i$ on $[n]\setminus\{i\}$, and a generative simulator for the $n$-agent system.}
\STATE Define the labeled neighborhood space $\mathcal{Y}  \coloneqq  ([n]\setminus\{i\})^\kappa \times (\mathcal{S} \times \mathcal{A})^\kappa$.
\STATE Initialize $\widehat{Q}_\kappa^{(0)}(s,a,y) = 0$ for all $(s,a,y) \in \mathcal{S} \times \mathcal{A} \times \mathcal{Y}$.
\FOR{$t=0, 1, \dots, T-1$}
  \FOR{$(s,a,y) \in \mathcal{S} \times \mathcal{A} \times \mathcal{Y}$}
    \STATE Extract sampled neighbor indices $\{J^{(1)}, \dots, J^{(\kappa)}\}$ and their observed pairs $\{(s_j, a_j)\}_{j \in \{J^{(m)}\}}$ from $y$.
    \STATE Compute importance sampling weights $\rho^{(m)} = \bar{w}_{i, J^{(m)}} / q_i(J^{(m)})$ for $m=1, \dots, \kappa$.
    \STATE Construct the importance sampling estimate $\widehat{g}$ (or $\widehat{Z}$) using $y$ and $\rho$.
      \STATE Sample $(s'_\ell, y'_\ell) \stackrel{\text{i.i.d.}}{\sim} \widetilde{\mathcal{J}}(\cdot \mid s, a, y)$ for $\ell=1, \dots, m$ \ \ \COMMENT{Induced kernel in the labeled space}.
      \STATE Update: $\widehat{Q}_\kappa^{(t+1)}(s,a,y) \gets r_\ell(s,a,\widehat{g}) + \frac{\gamma}{m} \sum_{\ell=1}^m \max_{a' \in \mathcal{A}} \widehat{Q}_\kappa^{(t)}(s'_\ell, a', y'_\ell)$.
  \ENDFOR
  \ENDFOR
\STATE \textbf{Return} $\widehat{Q}_\kappa^{(T)}$.
\label{algorithm: identity aware sampling}
\end{algorithmic}
\end{algorithm}

\textbf{Algorithm analysis.} In Algorithm~\ref{algorithm: identity aware sampling}, the Monte Carlo samples can be derived from standard MCMC or uniform sampling methods \citep{chapelle2011empirical,vempala2019rapid,may2012optimistic,8187198,riquelme2018deep,anand2025feelgoodthompsonsamplingcontextual}. The main limitation of this approach is the dimensionality of the augmented state space. The number of possible neighborhood tuples in $\mathcal{Y}$ scales as $(|\mathcal{S}| |\mathcal{A}|)^\kappa$. Consequently, the cardinality of the tabular identity-aware $Q$-function is $|\mathcal{S}| |\mathcal{A}| \cdot (|\mathcal{S}| |\mathcal{A}|)^\kappa = (|\mathcal{S}| |\mathcal{A}|)^{\kappa+1}$.\looseness-1

The exponential dependence on the subsampling parameter $\kappa$ motivates the mean-field reparameterization used throughout the \texttt{GMFS} model. Instead of learning on the labeled space $\mathcal{Y}$, we map a sampled labeled neighborhood $y$ to an unlabeled empirical distribution $\hat{Z}_i^{\kappa}$. This representation discards explicit neighbor identities but retains the aggregate information required for the reward and transition dynamics. The reparameterization thus makes learning tractable by replacing an exponentially large labeled neighborhood space with a distributional state summary whose complexity depends on $|\mathcal{S}| |\mathcal{A}|$ instead. \looseness-1

\begin{definition}[Horvitz-Thompson estimator]
Let $q$ be a proposal distribution on a finite set $\mathcal{J}$ and let $w$ be a target weight vector
with $w(j)\geq 0$ such that $\sum_{j\in\mathcal{J}} w(j)=1$, and $q(j)>0$ whenever $w(j)>0$.
For any function $\phi:\mathcal{J}\to\mathbb{R}$, sample $J_1,\dots,J_\kappa \stackrel{\text{i.i.d.}}{\sim} q$ and define the Horvitz-Thompson estimator
\begin{equation}\hat \Phi_\kappa(\phi)  \coloneqq  \frac{1}{\kappa}\sum_{m=1}^\kappa \frac{w(J_m)}{q(J_m)}\phi(J_m).\end{equation}
Then $\E[\hat \Phi_\kappa(\phi)] = \sum_{j\in\mathcal{J}} w(j)\phi(j)$.\\
\end{definition}

\begin{lemma}
    [Unbiasedness of the neighborhood estimator] Assume $q_i(j) > 0$ for all $j$ with $\bar w_{ij} > 0$. Next, define the importance sampling weights $\rho^{(m)}$ such that
    \begin{equation}\rho^{(m)} = \frac{\bar w_{iJ^{(m)}}}{q_i(J^{(m)})}.\end{equation}
    Then the importance weighted estimator $\hat{Z}_i^{(\kappa)}(s,a)$ is an unbiased estimator of the true neighborhood aggregate $Z_i(s,a) = \sum_{j \neq i} \bar w_{ij} \mathbbm{1} \{s_{j}=s, a_{j} = a \}$ under proposal distribution $q_i$. 

\end{lemma}

\begin{proof}
    By linearity of expectations, we have
    \begin{align*}
        \mathbb E [\hat{Z}_i^{(\kappa)}(s,a)] &= \mathbb E \Big[ \frac{1}{\kappa} \sum_{m=1}^\kappa \frac{\bar w_{iJ^{(m)}}}{q_i(J^{(m)})}  \mathbbm{1} \{s_{J_i^{(m)}}=s, a_{J_i^{(m)}} = a\}\Big]\\
        &= \frac{1}{\kappa} \sum_{m=1}^\kappa \mathbb E \Big[ \frac{\bar w_{iJ^{(m)}}}{q_i(J^{(m)})}  \mathbbm{1} \{s_{J_i^{(m)}}=s, a_{J_i^{(m)}} = a \} \Big],
    \end{align*}
    where
    \[
    \mathbb E \Big[ \frac{\bar w_{iJ^{(m)}}}{q_i(J^{(m)})}  \mathbbm{1} \{s_{J_i^{(m)}}=s, a_{J_i^{(m)}} = a \} \Big] = \mathbb E \Big[ \frac{\bar w_{iJ}}{q_i(J)}  \mathbbm{1} \{s_{J}=s, a_{J} = a \} \Big] 
    \]
as the $J_i^{(m)}$'s are drawn i.i.d. Therefore, 
    \begin{align*}
        \mathbb E [\hat{Z}_i^{(\kappa)}(s,a)] &= \mathbb E \Big[ \frac{\bar w_{iJ}}{q_i(J)}  \mathbbm{1} \{s_{J}=s, a_{J} = a \} \Big]\\
        & = \sum_{j \neq i} q_i(j) \frac{\bar w_{ij}}{q_i(j)} \mathbbm{1} \{s_{j}=s, a_{j} = a\}\\
        & = \sum_{j \neq i} \bar w_{ij} \mathbbm{1} \{s_{j}=s, a_{j} = a \} \\
        & = Z_i(s,a),
    \end{align*}
    where the second equality follows by using $J \in [n] \backslash \{i\}$.
    We conclude that since $\mathbb E \big[ \hat{Z}_i^{(\kappa)}(s,a) \big] = Z_i(s,a)$, the estimator is unbiased and thus a Horvitz-Thompson estimator.\\
\end{proof}

\begin{lemma} [Horvitz-Thompson unbiasedness for identity-aware neighborhoods]

    Let $q_i$ be a proposal distribution on $[n] \backslash \{i\}$, such that $q_i(j)>0$ for all $j$ with $\bar{w}_{ij}>0$. Define the importance weights:
    \[
    \rho^{(m)} = \frac{\bar w_{iJ^{(m)}}}{q_i(J^{(m)})} 
    \]
    Then for any function $\varphi: ([n] \backslash \{i\}) \times \mathcal{S} \times \mathcal{A} \to \mathbb{R}$, the Horvitz-Thompson estimator
    \begin{equation}
        \hat{\Phi}_i^{(\kappa)}(\varphi) = \frac{1}{\kappa} \sum_{m=1}^\kappa \rho^{(m)}  \varphi (J_i^{(m)}, s_{J_i^{(m)}}, a_{J_i^{(m)}})
    \end{equation}
     is unbiased for the graphon-weighted neighborhood: $\Phi_i(\varphi) = \sum_{j \neq i} \bar w_{ij} \varphi(j, s_j, a_j)$, i.e. $\mathbb{E}[\hat{\Phi}_i^{(\kappa)}(\varphi)] = \Phi_i(\varphi)$

\end{lemma}

    \begin{proof}
        By linearity of expectation and i.i.d. sampling, we have
        \begin{align*}
            \mathbb{E}[\hat{\Phi}_i^{(\kappa)}(\varphi)] &= \mathbb{E} \Big[\frac{\bar w_{ij}}{q_i(j)} \varphi (J, s_J, a_J)   \Big]\\
            &= \sum_{j \neq i} q_i(j) \frac{\bar w_{ij}}{q_i(j)} \varphi (j, s_j, a_j)\\
            &= \sum_{j \neq i} \bar w_{ij} \varphi(j, s_j, a_j)\\
            &= \Phi_i(\varphi)
        \end{align*}

        which proves the claim.\qedhere\\
    \end{proof}

\section{Lipschitz Continuity in the Mean-Field Measure} 
\label{appendix: lipschitz cont in mf}
When no subsampling occurs (at $\kappa = n-1$), \texttt{GMFS} recovers the graphon mean-field MARL formulation corresponding to interaction graphon $W$. As $\kappa \to n-1$, we show that $\hat{Q}_\kappa^* \to Q^*$ via a Lipschitz continuity bound between $Q^*$ and $\hat{Q}_\kappa^*$.

\begin{definition}
    [Total Variation Distance] Let $P$ and $Q$ be discrete probability distributions over some domain $\Omega$. Then,
    \[\mathrm{TV}(P,Q) = \frac12 \|P-Q\|_1 = \sup_{E\subseteq \Omega}\left|\Pr_P(E) - \Pr_Q(E)\right|.\]
\end{definition}

\begin{definition}
    [Mixing measures] Fix $\kappa\in[n-1]$. For any fixed agent $i\in[n]$ and a sampled multiset $\Delta_i = (J_i^{(1)},\dots, J_\kappa^{(\kappa)})$, we define a probability measure $\mu_i^{[n]}$ on $\cS\times \cA$ and an empirical measure $\hat{\mu}_{\Delta_i}$ on $\cS\times \cA$ by
   \[{\mu}_i^{[n]}(x,u) = \sum_{j\neq i}\bar{w}_{ij}\mathbbm{1}\{s_j = x, a_j = u\},\quad\quad \hat{\mu}_{\Delta_i}(x,u) = \frac{1}{\kappa}\sum_{m=1}^\kappa\mathbbm{1}\{s_{J_i^{(m)}}=x, a_{J_i^{(m)}}=u\}.\]  
\end{definition}

\begin{theorem}[Lipschitz Continuity of the Bellman iterates] \label{Lemma: lipschitz cont of operators} Fix a subsampling parameter $\kappa\geq 1$. Fix $(s,a)\in\mathcal{S}\times\mathcal{A}$, and let  $z\in \cZ\coloneqq \Delta(\cS\times\cA)$. Let $\hat{z}\in \mathcal Z_\kappa$ be the empirical histogram of $\kappa$ i.i.d. draws from $z$, and let $g_z, g_{\hat{z}}$ be the marginals in $\cS$. Then, for all $t\in\mathbb{N}$, we have
    \begin{align*}
        \left| Q^t(s,a,z) -  \hat{Q}^t_\kappa(s,a,\hat{z})\right| \leq  \frac{4 \|r_\ell\|_\infty }{1 - \gamma} L_P \cdot \mathrm{TV}(g_z, g_{\hat{z}}).
    \end{align*}
\end{theorem}
\begin{proof}

    We proceed by induction on $t$. At $t=0$, we have $Q^0(s,a,z) = \hat{Q}^0_\kappa(s,a,\hat{z}) = 0$. At $t=1$, 
    \begin{align*}
        |Q^{1}(s,a,z) - \hat{Q}_\kappa^{1}(s,a,\hat{z})| 
        &=  |\mathcal{T}Q^{0}(s,a,z) - \hat{\mathcal{T}}_\kappa\hat{Q}_\kappa^{0}(s,a,\hat{z})| 
\\
        &= \left|  r_\ell (s , a, g_z) + \gamma\E_{\mathcal{J}_n}\max_{a', z'} Q^{(0)}(\cdot) -   r_\ell (s, a, g_{\hat{z}}) - \gamma\E_{\mathcal{J}_\kappa}\max_{a', \hat z'} \hat Q^{(0)}(\cdot)\right| \\
        &= \left| r_\ell (s, a, g_z) -   r_\ell (s, a, g_{\hat z})\right| \\
        &\leq 2 \|r_\ell\|_\infty \cdot \mathrm{TV}(g_z, g_{\hat z}),
    \end{align*}
    where the last inequality follows by Assumption \ref{assumption: reward function is Lipschitz}. This proves the base case. For $t+1$:
    \begin{align*}
        \big| Q^{(t+1)}(s,a,z) - \hat{Q}_\kappa^{(t+1)}(s,a,\hat{z}) \big| &= |r_\ell(s,a,g_z) + \gamma \E_{s',g'\sim J_n}[M_{n-1} Q(s', g')] - r_\ell(s,a,g_{\hat{z}}) - \gamma \E_{s',\hat{g}'\sim \mathcal{J}_\kappa} [M_\kappa \hat{Q}_\kappa(s', \hat{g}')]|
        \\
        &\leq \underbrace{|r_\ell(s, a, g_z) - r_\ell(s, a, g_{\hat{z}})|}_{\text{Term (I)}} + \gamma  \underbrace{|\E_{s',g'\sim J_n}[M_{n-1} Q(s', g')] - \E_{s',\hat{g}'\sim \mathcal{J}_\kappa} [M_\kappa \hat{Q}_\kappa(s', \hat{g}')]|}_{\text{Term (II)}}
    \end{align*}
    By Assumption~\ref{assumption: reward function is Lipschitz}, Term (I) is bounded by
    $L_r   \cdot \mathrm{TV}(\hat g,g)$. We bound Term (II)  by Lemma \ref{lemma: lipschitz wrt expected q-function}, 
    \begin{align*}
        |\E_{s',g'\sim J_n}[M_{n-1} Q(s', g')] - \E_{s',\hat{g}'\sim \mathcal{J}_\kappa} [M_\kappa \hat{Q}_\kappa(s', \hat{g}')]| \leq  \frac{4\|r_\ell\|_\infty}{1-\gamma}  L_P \cdot \mathrm{TV}(g_z, g_{\hat{z}}) 
    \end{align*}
    Then, recalling that $L_P\geq 1$ by Assumption \ref{assumption: transition are lipschitz}, we have
    \begin{align*}
    \big| Q^{(t+1)}(s,a,z)
    - \hat{Q}^{(t+1)}_\kappa(s,a,\hat z)\big|
    &\le   2 \|r_\ell\|_\infty \cdot \mathrm{TV}(g_z, g_{\hat{z}})  +  \frac{4\gamma\|r_\ell\|_\infty}{1-\gamma}   L_P \cdot \mathrm{TV}(g_z, g_{\hat{z}})\\
    & \leq 4 \|r_\ell\|_\infty \cdot L_P \cdot \mathrm{TV}(g_z, g_{\hat{z}})  +  \frac{4\gamma\|r_\ell\|_\infty}{1-\gamma}   L_P \cdot \mathrm{TV}(g_z, g_{\hat{z}}) \\ 
    &\leq \frac{4\| r_\ell \|_\infty  }{1 - \gamma}    \cdot L_P \cdot \mathrm{TV}(g_z, g_{\hat{z}}),
\end{align*}
 which proves the claim.\qedhere\\
\end{proof}

\begin{definition}[One-step kernel] 
    \label{Def: Joint stochastic} Let
$J_n(\cdot|s,a,z)$ denote the induced one-step distribution on $(s',g')\in \cS\times \cG$
under the $n$-agent dynamics, for a uniformly random agent, given $(s,a,z)\in \cS\times\cA\times\cZ$.
    \end{definition}

\begin{definition}[Surrogate one-step kernel for GMFS]
    For $\kappa\in\{1,\dots,n-1\}$, let $J_\kappa(\cdot|s,a,\hat z)$ denote the induced one-step distribution on $(s',\hat g')\in \cS\times \cG_\kappa$
under the $(\kappa+1)$-agent surrogate dynamics, for a uniformly random agent, given $(s,a,\hat z)\in \cS\times\cA\times\hat{\cZ}_\kappa$.
\end{definition}

For Lemma \labelcref{lemma: lipschitz wrt expected q-function}, we want to show that the value term is Lipschitz in the mean-field argument, which later allows us to control the discrepancy between the original and subsampled Bellman operators.

\begin{lemma}[Lipschitz-continuity of the expected values of the Bellman iterates]\label{lemma: lipschitz wrt expected q-function} Fix a subsampling parameter $\kappa\geq 1$. Fix $(s,a)\in\mathcal{S}\times\mathcal{A}$, and let  $z\in \cZ\coloneqq \Delta(\cS\times\cA)$. Let $\hat{z}\in \mathcal Z_\kappa$ be the empirical histogram of $\kappa$ i.i.d. draws from $z$, and let $g_z, g_{\hat{z}}$ be the marginals in $\cS$. Then, for all $t\in\mathbb{N}$, we have
    \[|\E_{(s', g') \sim \mathcal{J}_n(\cdot | s,a,z)} \max_{a', z'} Q^t(s', a', z') - \E_{(s', \hat g') \sim \mathcal{J}_\kappa(\cdot | s,a,\hat{z})}\max_{a', \hat{z}'}\hat{Q}^t_\kappa(s', a', \hat{z}')| 
    \leq  \frac{4\| r_\ell\|_\infty}{1 - \gamma}   L_P\cdot \mathrm{TV}(g_z,g_{\hat{z}}),\]
where $\mathcal{J}$ is the joint stochastic kernel that generates the next state.
\end{lemma}

\begin{proof} We proceed by induction. For the base case where $t=0$, both $Q^0$ and $\widehat Q_\kappa^0$ are identically zero, hence the inequality holds trivially. When $t = 1$, we first note that $Q^1(s,a,z) = r_\ell(s,a,g_z)$ and we let $\tilde{r}_*(s,g_z) = \max_a r_\ell(s,a,g_z)$. Then, we have $M_{n-1} Q^1(s',g_{z'}) = \tilde{r}_*(s',g_{z'})$ and $M_\kappa \hat{Q}_\kappa^1 = \tilde{r}_*(s',g_{\hat{z}'})$. Therefore,
    \begin{align*}
        \Big| \mathbb{E}_{(s', g') \sim \mathcal{J}_n} \max_{a',z'}[r(s',a',g_{z'})] &- \mathbb{E}_{(s', \hat{g}') \sim \mathcal{J}_\kappa} \max_{a',\hat{z}'}[r(s',a',g_{\hat{z}'})]\Big| \\
        &= \Big| \mathbb{E}_{(s', g') \sim \mathcal{J}_n} \tilde{r}_*(s',g_{z'}) - \mathbb{E}_{(s', \hat{g}') \sim \mathcal{J}_\kappa} \tilde{r}_*(s',g_{\hat{z}'}) \Big| \\
        &\leq 2\|\tilde{r}_*\|_\infty \cdot \mathrm{TV}(J_n(\cdot|s,a,g_z), J_\kappa(\cdot|s,a,g_{\hat{z}})) \leq 4\|r_\ell\|_\infty \cdot L_P \cdot \mathrm{TV}(g_z,g_{\hat{z}}).
    \end{align*}
where the first inequality uses Lemma \ref{lemma: reward mismatch sampling}, and the second inequality follows by triangle inequality and Assumption \ref{assumption: transition are lipschitz}. Assume for $t \geq 1$: 
    \[\Big| \E_{(s', g') \sim \mathcal{J}_n(\cdot|s,a,g)}[M_{n-1}Q^t(s',g')] - 
        \E_{(s', \hat{g}') \sim \mathcal{J}_\kappa(\cdot|s,a,\hat{g})}
        [M_\kappa \hat{Q}_\kappa^T(s',\hat g')]
        \Big| \leq 
        \frac{4\|r_\ell\|_\infty}{1-\gamma}   L_P \cdot \mathrm{TV}(g_z, g_{\hat{z}'}).\]
    Then, for the inductive step,
    using the Bellman updates for $Q^{t+1}$ and $\hat Q_\kappa^{t+1}$, we write
    \begin{align*}
         \Big| 
        \E_{(s', g') \sim \mathcal{J}_n} \max_{a',z'} Q^{t+1}(s',a',z')
        &-
        \E_{(s', \hat{g}') \sim \mathcal{J}_\kappa} \max_{a',\hat{z}'} \widehat Q_\kappa^{t+1}(s',a',\hat z')\Big| \\
        &=
        \Big|
        \E_{(s', g') \sim \mathcal{J}_n}\max_{a',z'}
        \big[
        r_\ell(s',a',g_{z'}) + \gamma\E_{(s'', g'') \sim \mathcal{J}_n} \max_{a'',z''} Q^t(s'',a'', z'')
        \big]\\
        &\quad\quad -
        \E_{(s', \hat{g}') \sim \mathcal{J}_\kappa}\max_{a',\hat{z}'}
        \big[
        r_\ell(s',a',g_{\hat z'}) + \gamma \E_{(s'', \hat{g}'') \sim \mathcal{J}_\kappa} \max_{a'',\hat{z}''} \widehat Q_\kappa^t(s'',a'', \hat z'')
        \big]
         \Big|
    \end{align*}
        Then, using the fact that $\max(\cdot)$ is $1-$Lipschitz, we have
    \begin{align*}
        &\max_{a',z'} Q^{t+1}(s',a',z') - \max_{a',\hat{z}'} \hat Q_\kappa^{t+1}(s',a',\hat z') \\
        &\leq \max_{a',z',\hat{z}'}\Big| r_\ell(s',a',g_{z'})- r_\ell(s',a',g_{\hat z'}) + \gamma \Big(\E_{(s'', g'') \sim \mathcal{J}_n} \max_{a'',z''} Q^t(s'',a'',z'') - \E_{(s_i'', \hat{g}'') \sim \mathcal{J}_\kappa} \max_{a'',\hat{z}''} \hat Q_\kappa^t(s'',a'',\hat z'') \Big) \Big|\\
        & \leq \max_{a'}\Big| r_\ell(s',a',g_{z'})- r_\ell(s',a', g_{\hat z'}) \Big| + \gamma \max_{a',z',\hat{z}'} \Big| \E_{(s'', g'') \sim \mathcal{J}_n} [\max_{a'',z''} Q^t(s'',a'',z'')] - \E_{(s'', \hat{g}'') \sim \mathcal{J}_\kappa} \max_{a'', \hat z''} [\hat Q_\kappa^t(s'',a'',\hat z'')] \Big|
    \end{align*}

    Now taking our original expectations over $\mathcal{J}_n$ and $\mathcal{J}_\kappa$, triangle inequality gives us:
    \begin{align*}
        \Big|
        \E_{(s', g') \sim \mathcal{J}_n}\max_{a',z'}
        \big[
        r_\ell(s',a', g_{z'}) &+ \gamma\E_{(s'', g'') \sim \mathcal{J}_n} \max_{a'',z''} Q^t(s'',a'', z'')
        \big]\\
        &\quad\quad\quad \quad\quad\quad-
        \E_{(s', \hat{g}') \sim \mathcal{J}_\kappa}\max_{a'}
        \big[
        r_\ell(s',a', g_{\hat z'}) + \gamma\E_{(s'',  \hat{g}'') \sim \mathcal{J}_\kappa} \max \widehat Q_\kappa^t(s'',a'',\hat z'')
        \big]
         \Big|\\
        &\leq \underbrace{\Big|
        \E_{\mathcal{J}_n} \big[\max_{a'} r_\ell(s',a',g_{z'})\big] -
        \E_{\mathcal{J}_\kappa}\big[\max_{a'} r_\ell(s',a', g_{\hat z'})\big] \Big|}_{\text{Term (I)}}\\
        &+ 
        \underbrace{\gamma  \Big|
        \E_{\mathcal{J}_n}
        \Big[\max_{a'} \E_{\mathcal{J}_n}
        \big[\max_{a'',z''} Q^t(s'', a'', z'')\big]\Big]
        -
        \E_{\mathcal{J}_\kappa}
        \Big[\max_{a'} \E_{\mathcal{J}_\kappa}
        \big[\max_{a'',\hat{z}''} \hat Q_\kappa^t(s'', a'', \hat z'')\big]\Big] \Big|}_{\text{Term (II)}}
    \end{align*}

    Term (I) follows the same structure as in the base case ($t=1$); hence, using Lemmas \ref{lemma: reward mismatch sampling} and \ref{lemma: higher order kernel contraction} we bound it by:
    \begin{align*}
    \Big|
        \E_{\mathcal{J}_n} \big[\max_{a'} r_\ell(s',a',g_{\hat z'})\big] -
        \E_{\mathcal{J}_\kappa}\big[\max_{a'} r_\ell(s',a', g_{\hat z'})\big] \Big|
    &\leq 
  4\|r_\ell\|_\infty \cdot L_P \cdot \mathrm{TV}(g_{z'},g_{\hat{z}'}) \\
    &\leq 4\|r_\ell\|_\infty \cdot L_P \cdot \mathrm{TV}(g_z, g_{\hat{z}}),
    \end{align*}
    where the second inequality follows by Lemma \ref{lemma: tv mixture}. 
    
    For Term (II), the inner difference is the inductive hypothesis applied at the next step:
    \begin{align*}
        \gamma  \Big|
        \E_{\mathcal{J}_n}
        \Big[\max_{a',z'} \E_{\mathcal{J}_n}
        \big[\max_{a'',z''} Q^t(s'', a'',   z'')\big]\Big]
        &-
        \E_{\mathcal{J}_\kappa}
        \Big[\max_{a', z'} \E_{\mathcal{J}_\kappa}
        \big[\max_{a'', z''} \hat Q_\kappa^t(s'', a'', \hat z'')\big]\Big] \Big| 
        \\
        &\leq \frac{4\gamma \| r_\ell\|_\infty}{1-\gamma}   L_P \cdot \mathrm{TV}(g_{z''}, g_{\hat{z}''}) \\
        &\leq \frac{4\gamma \|r_\ell\|_\infty}{1-\gamma}   L_P \cdot \mathrm{TV}(g_z, g_{\hat{z}}),
    \end{align*}
    where the first inequality follows from Lemma \ref{lemma: higher order kernel contraction} and the last inequality follows by Lemma \ref{lemma: tv mixture}.
    
Then, combining Term (I) and Term (II), we get
    \begin{align*}
        |\E_{(s', g') \sim \mathcal{J}_n(\cdot | s,a,z)} \max_{a', z'} Q^t(s', a', z') &- \E_{(s', \hat g') \sim \mathcal{J}_\kappa(\cdot | s,a,\hat{z})}\max_{a', \hat{z'}}\hat{Q}^t_\kappa(s', a', \hat{z})| \\
        &\leq  4\|r_\ell\|_\infty \cdot L_P \cdot \mathrm{TV}(g_z, g_{\hat{z}}) + \frac{4\gamma\|r_\ell\|_\infty}{1-\gamma}  L_P \cdot \mathrm{TV}(g_z, g_{\hat{z}}) \\
        &= \frac{4\|r_\ell\|_\infty}{1-\gamma}  L_P \cdot \mathrm{TV}(g_z, g_{\hat{z}}), \end{align*}
which completes the proof.
\qedhere\\
\end{proof}

\begin{definition}[Joint transition probability kernel] Fix an agent $i\in[n]$. Let $s$ and $a$ denote the state and action of agent $i$ and $\hat{z}$ denote its graphon-weighted state/action feature on a sample of size $\kappa$ (which we denote by $\Delta_i)$. Let the state marginals of $\hat{z}$ be $\hat{g}$. The joint transition probability kernel for agent $i$ and its neighborhood is then:
\[
\mathcal{J}_{\Delta_i \cup \{i\}} (s', \hat{g}' | s, a, \hat z).
\]
\end{definition}

\begin{definition} [Single Agent Transition Kernel]
    \label{def: single agent kernel}
    The single agent transition kernel for agent $i$ is given by the marginal
    \[
    \mathcal{J}_1 \bigl(s_i'
    | s_i, s_{\Delta_i}, a_i, a_{\Delta_i} \bigr) =
    \sum_{s_{\Delta_i}'}
    \mathcal{J}_{\Delta_i \cup \{i\}}
    \bigl(s_i', s_{\Delta_i}'| s_i, s_{\Delta_i}, a_i, a_{\Delta_i}\bigr).
    \]  
\end{definition}

\begin{lemma}
    Assuming that the next state of agent $i$ is conditionally independent of the neighbors' next states given the current states and actions, the kernel $\mathcal{J}_1$ satisfies:
    \[
    \mathbb P\left(S_i' = s_i' | s_i, s_{\Delta_i}, a_i, a_{\Delta_i} \right) = \mathcal J_1 \big( s_i'| s_i, s_{\Delta_i}, a_i, a_{\Delta_i}\big)
    \]
\end{lemma}

\begin{proof}
    \begin{align*}
        \mathbb P\left(S_i' = s_i' | s_i, s_{\Delta_i}, a_i, a_{\Delta_i} \right) &= \sum_{s'_{\Delta_i}} \mathbb P\left(S_i' = s_i', S_{\Delta_i} = s_{\Delta_i}' | s_i, s_{\Delta_i}, a_i, a_{\Delta_i} \right)  \\
        & = \sum_{s_{\Delta_i}'}
        \mathcal{J}_{\Delta_i \cup \{i\}}
        \bigl(s_i', s_{\Delta_i}'| s_i, s_{\Delta_i}, a_i, a_{\Delta_i}\bigr)  \\
        & = \mathcal{J}_1 \bigl(s_i'
        | s_i, s_{\Delta_i}, a_i, a_{\Delta_i} \bigr),\\
    \end{align*}
    where the first equality follows by evaluating the conditional probability, and the second equality follows by definition \ref{def: single agent kernel}. Together, this proves the lemma.\qedhere\\
\end{proof}

We now prove a contraction property in TV-distance between the stochastic kernels $\cJ_\kappa$ and $\cJ$ under a common neighborhood.

\begin{lemma}[Coupling stability of the surrogate kernel]\label{lemma: higher order kernel contraction}
Fix $(s,a)\in \cS\times \cA$. Let $z\in \mathcal P(\cS\times \cA)$ and let $\hat z\in \cZ_\kappa$ be any empirical
histogram; write $g \coloneqq g_z$ and $\hat g \coloneqq g_{\hat z}$. 
Then, under the dynamics of $\cJ_n$ and $\cJ_\kappa$, there exists a coupling of $(S'_0,g')$ and $(\hat S'_0,\hat g')$ such that $\Pr[S'_0\neq \hat S'_0] \le \mathrm{TV}\!\left(P(\cdot\mid s,a,g),P(\cdot\mid s,a,\hat g)\right)
\le L_P\cdot \mathrm{TV}(g,\hat g)$ and 
\begin{equation}\mathbb{E}\big[\mathrm{TV}(g',\hat g')\big] \leq L_P\cdot \mathrm{TV}(g, \hat g).\\
\end{equation}
\end{lemma}

\begin{proof}
    For the focal agent, take a maximal coupling of $P(\cdot\mid s,a,g)$ and $P(\cdot\mid s,a,\hat g)$,
so that \[\Pr[S'_0\neq \hat S'_0]=\mathrm{TV}(P(\cdot\mid s,a,g), P(\cdot\mid s,a,\hat g)).\]
Then, applying the Lipschitz continuity of the transition kernels from Assumption \ref{assumption: transition are lipschitz} completes the proof of the first inequality. For the second inequality, we consider the neighborhood marginal. Couple the two constructions by using the same i.i.d. draws
$(X_m,U_m)\sim z$ and, conditional on each $(X_m,U_m)$,
couple $X'_m\sim P(\cdot\mid X_m,U_m,g)$ with $\hat X'_m\sim P(\cdot\mid X_m,U_m,\hat g)$ via an optimal
coupling. Then, from Assumption \ref{assumption: transition are lipschitz}, we have that
\begin{align*}
\Pr[X'_m\neq \hat X'_m \mid X_m,U_m]
&= \mathrm{TV}\!\left(P(\cdot\mid X_m,U_m,g),P(\cdot\mid X_m,U_m,\hat g)\right) \\
&\le L_P\cdot \mathrm{TV}(g,\hat g),
.
\end{align*}
Next, let $D \coloneqq \sum_{m=1}^\kappa \mathbbm{1}\{X'_m\neq \hat X'_m\}$ be the number of mismatches.
A single mismatch can change the empirical histogram by at most $2/\kappa$ in $\ell_1$,
hence at most $1/\kappa$ in TV. Therefore, deterministically, we have
\begin{align*}
\|g' - \hat{g}'\|_1 \leq \frac{2D}{{\kappa}} \implies  \frac12 \|g' - \hat{g}'\|_1 \leq \frac{D}{ \kappa}.\end{align*}
 Finally, taking expectations and using linearity, we have
\begin{align*}
\mathbb{E}[\mathrm{TV}(g',\hat g')]
&\leq \frac{1}{\kappa}\sum_{m=1}^\kappa \Pr[X'_m\neq \hat X'_m] \leq L_P \cdot \mathrm{TV}(g,\hat g),
\end{align*}
which completes the proof.
\end{proof}

\begin{lemma} [Total variation contraction under mixture kernels]
    \label{lemma: tv mixture} Let $(\cZ, \cA)$ be the index measurable space and $(\cX,\cB)$ be the output measurable space. Let $J_1(\cdot|z)$ be a Markov kernel from $\cZ$ to $\cX$. For probability measures $\mu,\nu$ on $\cZ$, define the mixtures $J_\mu(B)=\int_\cZ J_1(B|z)\mu(\mathrm dz)$ and $J_\nu(B)\coloneqq \int_\cZ J_1(B|z)\nu(\mathrm dz)$ for all $B\in \cB$. Then,
    \begin{equation}
    \mathrm{TV}(\mathcal J_\mu,\mathcal J_\nu) \leq \mathrm{TV}(\mu,\nu).
    \end{equation}
\end{lemma}

\begin{proof}
    Note that the lemma is a special case of the data processing inequality for $f$-divergences, where $f(t)=\frac12 |t-1|$. To prove it, we use that for probability measures $P,Q$ on a measurable space,
    \[\mathrm{TV}(P,Q)=\sup_B |P(B)-Q(B)| = \frac12 \sup_{\|f\|_\infty \leq 1} \left| f\mathrm d(P-Q)\right|.\]
    So, fix any measurable $B\in \cB$ and let $\delta\coloneqq \mu-\nu$ be a signed measure on $\cZ$, where $\delta(\cZ) = 0$. Then, 
    \begin{align*}
        J_\mu(B) - J_\nu(B) &= \int_\cZ J_1(B|z)\delta(\mathrm dz) \\
        &= \frac12 \int (2 J_1(B|z) - 1)\mathrm d\delta,
    \end{align*}
    where the last equality follows from $\int 1 d\delta = \delta(\cZ) = 0$. Then, define $\phi_B(z) \coloneqq 2J_1(B|z) - 1$. Since $J_1(B|z)\in [0,1]$, we have $\phi_B(z) \in [-1,1]$ and hence $\|\phi_B\|_\infty \leq 1$.
    
    Therefore, we write
    \begin{align*}
        |J_\mu(B) - J_\nu(B)| &= \frac12 \left|\phi_B\mathrm d(\mu-\nu)\right| \\
        &\leq \frac12 \sup_{\|f\|_\infty \leq 1} \left|f \mathrm d(\mu-\nu)\right| \\
        &= \mathrm{TV}(\mu,\nu).\\
    \end{align*}
Since this holds for every $B\in\cB$, taking the supremum gives 
\[\mathrm{TV}(J_\mu,J_\nu) = \sup_B |J_\mu(B) - J_\nu(B)| \leq \mathrm{TV}(\mu,\nu),\]
which proves the claim.\qedhere\\
\end{proof}

\begin{lemma}[Reward expectation mismatch under subsampling] \label{lemma: reward mismatch sampling}
Let $\tilde r:\mathcal S\times\mathcal G\to\mathbb R$ be a bounded local immediate reward with $\|{r}_\ell\|_\infty < \infty$. Let $\mu_{[n]}$ and $\hat\mu_\Delta$ are distributions on the index space  from Lemma \ref{lemma: tv mixture}.  Next, define the induced transition kernels
\begin{equation}
\mathcal J_n(\cdot) = \int \mathcal J_1(\cdot\mid z) \mu_{[n]}(\mathrm dz)
\end{equation}
and
\begin{equation}
\mathcal J_\kappa(\cdot) = \int \mathcal J_1(\cdot\mid z) \hat\mu_\Delta(\mathrm dz),
\end{equation}
where $\mathcal J_1(\cdot\mid z)$ is the single-agent transition kernel. The reward expectation mismatch is then given by
\[|\E_{\mathcal{J}_n}[\tilde{r}(x,g')] - \E_{\mathcal{J}_\kappa}[\tilde{r}(x,\hat{g}']| \leq 2 \|{r}_\ell\|_\infty \cdot \mathrm{TV}(\mu_{[n]}, \hat{\mu}_\Delta)
.\]
\end{lemma}
\begin{proof} By expanding the expectations and using $|\tilde{r}|\leq \|r_\ell\|_\infty$, we have
    \begin{align*}
        |\E_{\mathcal{J}_n}[\tilde{r}(x,g')] - \E_{\mathcal{J}_\kappa}[\tilde{r}(x,\hat{g}']| &\le
        \sum_{(x,g')\in\mathcal S\times\mathcal G}
        \bigl|\mathcal J_n(x,g'|\cdot)- \mathcal J_\kappa(x,g'|\cdot)\bigr| 
        |\tilde r_\ell(x,g')|\\
        & \leq \|r_\ell\|_\infty \cdot \sum_{(x,g')\sim \mathcal S\times \mathcal G} \left| \mathcal{J}_n(x,g'|\cdot) - \mathcal{J}_\kappa (x,g'|\cdot) \right| \\
        &= 2 \|{r}_\ell\|_\infty \cdot  \mathrm{TV} (\mathcal{J}_n, \mathcal{J}_\kappa)
    \end{align*}
 
    Now we bound the total variation distance between the two kernels. Since $\mathcal J_n$ and $\mathcal J_\kappa$ are mixture kernels induced by $\mu_{[n]}$ and
    $\hat\mu_\Delta$, lemma \labelcref{lemma: tv mixture} gives $\mathrm{TV}(\mathcal J_n,\mathcal J_\kappa) \leq \mathrm{TV}(\mu_{[n]}, \hat{\mu}_\Delta)$. Hence,  the bound on our expectation mismatch is
    \[
        |\E_{J_n}[\tilde{r}(x,g')] - \E_{J_\kappa}[\tilde{r}(x,\hat{g}']| \leq 2 \| r_\ell \|_\infty \cdot \mathrm{TV}(\mu_{[n]}, \hat{\mu}_\Delta),
    \]
    which completes our proof.\qedhere\\
\end{proof}

\subsection{Concentration on the Subsampled Mean-Field Features}
\label{appendix: concentration}

Fix $\kappa\in[n-1]$. We now wish to bound the TV-distance between the true graphon-weighted neighborhood feature $g$ and its subsampled estimate $\hat{g}$. Since $g$ is a probability mass over a finite discrete space $\mathcal S$ and $\hat{g}$ is the empirical distribution formed by $\kappa$ i.i.d. samples drawn from $g_i$, we can argue a concentration bound using a similar argument as in \citet{anand2024efficient}.\\

We begin by stating the probability model.

\begin{lemma}
    [i.i.d. neighbor-state sampling]
    Fix an agent $i$ and condition on the global joint state $s(t) = (s_1(t), \dots, s_n(t))$. Define the ground-truth graphon-weighted neighborhood state distribution for $x\in\cS$ by
    \[g_i(x) \coloneqq \sum_{j\neq i}\bar{w}_{ij} \mathbbm{1}\{s_j(t) = x\},\] where $\bar{w}_{ij}$ are the normalized interaction weights from Definition \ref{def: Graphon-weighted neighborhood state--action distribution for agent}. Then, sample neighbor indices $J_1^{(1)},\dots,J_i^{(\kappa)}\sim \bar{w}_i$ and define the sampled neighbor states $X_m \coloneqq s_{J_i^{(m)}}\in\cS$. Then, conditional on $s$, the random variables $X_1,\dots, X_\kappa$ are i.i.d. with law $g_i$. In other words, $\Pr[X_m = x | s] = g_i(x)$, for all $x\in\cS$ and $m=1,\dots,\kappa$.
\end{lemma}

\begin{proof}
    The proof follows by noting that
    \begin{align*}
        \Pr[X_m = x | s] &= \sum_{j\neq i}\Pr[J_i^{(m)} = j]\mathbbm{1}\{s_j = x\} \\
        &= \sum_{j\neq i}\bar{w}_{ij}\mathbbm{1}\{s_j = x\} \\
        &= g_i(x),
    \end{align*}
    where independence follows from the i.i.d sampling of $J_i^{(m)}$.\qedhere\\
\end{proof}

\begin{theorem}
    [TV concentration for empirical distributions on a finite alphabet]
    Let $\cS$ be a finite set, and let $X_1,\dots, X_\kappa$ be i.i.d. samples from a distribution $p\in \Delta(\cS)$. Let $\hat{p}$ be the empirical distribution given by $\hat{p}(x) \coloneqq \frac1\kappa \sum_{m=1}^\kappa \mathbbm{1}\{X_m = x\}$. Then, for any $\epsilon>0$, we have that with probability at least $1-\delta$,
    \[\mathrm{TV}(\hat p, p) \leq \sqrt{\frac{|\cS|\ln 2 + \ln (2/\delta)}{2\kappa}}\]
    \label{lemma: concentration subsampled mf}
\end{theorem}
\begin{proof} Recall that 
\[\mathrm{TV}(\hat p, p) = \sup_{E\subseteq \cS} |\hat p(E) - p(E)|,\]
where $\hat p(E) \coloneqq \sum_{x\in E}\hat p(x)$. Fix any subset $E\subseteq S$, and define $Y_m \coloneqq \mathbbm{1}\{X_m \in E\} \in \{0,1\}$. Then, $Y_1, \dots, Y_\kappa$ are i.i.d. Bernoulli with mean $p(E)$ and $\hat{p}(E) = \frac{1}{\kappa}\sum_{m=1}^\kappa Y_m$.

Then, by Hoeffding's inequality \cite{409cf137-dbb5-3eb1-8cfe-0743c3dc925f}, we have
        \[\Pr[\mathrm{TV}(\hat p, p) \geq \epsilon] \leq 2^{|S|+1} \exp(-2\kappa\epsilon^2).\]
Now, taking a union bound over all subsets $E\subseteq \cS$, we have
\begin{align*}\Pr\left[\sup_{E\subseteq \cS} |\hat p (E) - p(E)| \geq \epsilon\right] &\leq \left|\sum_{E\subseteq \cS}1\right| \cdot 2\exp(-2\kappa\epsilon^2) \\
&= 2^{|\cS|+1} \exp(-2\kappa\epsilon^2).
\end{align*}
Finally, reparameterizing by setting $\delta = 2^{|\cS|+1} \exp(-2\kappa\epsilon^2)$ yields the claim.\qedhere\\
\end{proof}

 \begin{corollary}
     Fix a subsampling parameter $\kappa\geq 1$ and error parameter $\delta\in(0,1)$. Fix $(s,a)\in\mathcal{S}\times\mathcal{A}$, and let  $z\in \cZ\coloneqq \Delta(\cS\times\cA)$. Let $\hat{z}\in \mathcal Z_\kappa$ be the empirical histogram of $\kappa$ i.i.d. draws from $z$, and let $g_z, g_{\hat{z}}$ be the marginals in $\cS$. Then, for all $t\in\mathbb{N}$, we have that with probability at least $1-\delta$,
    \begin{align*}
        \left| Q^t(s,a,z) -  \hat{Q}^t_\kappa(s,a,\hat{z})\right| \leq  \frac{4  L_P \|r_\ell\|_\infty }{1 - \gamma} \cdot \sqrt{\frac{|\cS|\ln 2 + \ln (2/\delta)}{2\kappa}}.\\
    \end{align*}
 \end{corollary}
  
\section{Performance Gap between Optimal and Subsampled Policies}
 \label{appendix: performance gap}

In this section, we relate the discrepancy between the optimal $Q$-function $Q^*$ and its subsampled fixed point $\hat{Q}_\kappa^*$ to the performance gap between the optimal policy $\pi^*$ and the estimated \texttt{GMFS} policy $\pi^{\mathrm{est}}_\kappa$. We consider a state space $\mathcal{S} \times \mathcal{G}_\kappa$ and an action space $\mathcal{A} \times \mathcal{H}_\kappa$. Specifically, for each $(s,g)$, feasible actions are completions of the fiber $z \in \Gamma_\kappa(g)$. A policy $\pi(\cdot \mid s,g)$ is thus a distribution over the product space $(a,z) \in \mathcal{A} \times \Gamma_\kappa(g)$.

\begin{definition}[Value Function] 
Given a policy $\pi$, the value function is given by
\[
V^\pi(s,g) \coloneqq \mathbb{E}_{(a, z) \sim \pi(\cdot \mid s,g)} [Q^\pi(s,a,z)].
\]
\end{definition}

We define $\pi^*$ as the optimal policy induced by the optimal $Q$-function $Q^*$.

\begin{definition}[Optimal policy $\pi^*$]
For each $(s,g) \in \mathcal{S} \times \mathcal{G}_\kappa$, the optimal greedy policy is:
\[
\pi^*(\cdot \mid s,g) \in \arg\max_{(a, z) \in \mathcal{A} \times \Gamma_\kappa(g)} Q^*(s,a,z),
\]
where the associated Bellman backup is $\mathcal{M}_\kappa Q(s,g) \coloneqq \max_{a \in \mathcal{A}, z \in \Gamma_\kappa(g)} Q(s,a,z)$.
\end{definition}

\begin{definition}[Estimated \texttt{GMFS} policy $\pi^{\mathrm{est}}_\kappa$]
\label{def: estimated gmfs policy}
Let $\hat{Q}_\kappa$ be the $Q$-function obtained after $T$ Bellman updates in Algorithm~\ref{algorithm: GMFS offline}. For any $(s, \hat{g}) \in \mathcal{S} \times \mathcal{G}_\kappa$, the estimated greedy joint action for the agent and its fiber completion is:
\[
(a, \hat{z}) \coloneqq \arg\max_{a \in \mathcal{A}, \hat{z} \in \Gamma_\kappa(g)} \hat{Q}_\kappa(s, a, \hat{z}).
\]
\end{definition}

We next introduce the celebrated performance difference lemma from \citet{Kakade2002ApproximatelyOA}, which is a key tool for bounding the optimality gap of our learned policy.
\begin{lemma}[Performance Difference Lemma, \cite{Kakade2002ApproximatelyOA}]
Given two policies $\pi$ and $\pi'$, for any initial state $s_0$, we have:
\[
V^\pi(s_0) - V^{\pi'}(s_0) = \frac{1}{1 - \gamma} \mathbb{E}_{s \sim d_{s_0}^\pi} \left[ \mathbb{E}_{a \sim \pi(\cdot \mid s)} [ A^{\pi'}(s, a) ] \right],
\]
where $A^{\pi'}(s, a) = Q^{\pi'}(s, a) - V^{\pi'}(s)$ is the advantage function, and $d_s^{\pi}(s') = (1-\gamma) \sum_{h=0}^\infty \gamma^h \Pr_h^{\pi}[s',s]$ where $\Pr_h^{\pi}[s',s]$ is the probability of $\pi$ reaching state $s'$ at time step $h$ when starting from state $s$.
\end{lemma}

We are now ready to formulate the proof for~\Cref{probably main theorem}.
   \begin{theorem} 
Fix $\delta \in (0,1)$. For all states $s \in \mathcal{S}$ and graphon state aggregates $g \in \mathcal{G}$, if $T \geq \frac{1}{1-\gamma} \log \frac{\|r_\ell\|_\infty \sqrt{\kappa}}{1-\gamma}$, then
       
       \[
V^{\pi^*} (s, g) - V^{\pi^{\mathrm{est}}_\kappa} (s, g) \leq \frac{2L_P \|r_\ell\|_\infty}{(1-\gamma)^2} \sqrt{\frac{|\cS|\ln 2 + \ln (2/\delta)}{2\kappa}} + \frac{\epsilon_{\kappa,m}}{1-\gamma} + \frac{2 \|r_\ell\|_\infty}{(1-\gamma)^2} |\mathcal{A}| \cdot \kappa^{|\mathcal A|} \delta. 
\]
    \end{theorem}
    \begin{proof}
        Applying the Performance Difference Lemma, let $A^{\pi'}(s,a,z) = Q^{\pi'}(s,a,z) - V^{\pi'}(s,g)$ be the advantage function of policy $\pi'$ at $(s,a,z)$. Let $d^\pi_{(s,g)}$ denote the discounted occupancy measure over the space $\mathcal{S} \times \mathcal{A} \times \mathcal{Z}$ induced by policy $\pi$ starting from $(s,g)$. Then,
        \begin{align*}
            V^{\pi^*} (s_0, g_0) - V^{\pi^{\mathrm{est}}_\kappa} (s_0, g_0) &= \frac{1}{1-\gamma} \mathbb{E}_{(s,a,z) \sim d_{(s_0,g_0)}^{\pi_\kappa^{\text{est}}}} \big[\mathbb{E}_{a \sim \pi^*(\cdot | s,g_z)} A^{\pi^*}(s,a,z) \big] \\
        &= \frac{1}{1-\gamma} \mathbb{E}_{(s,a,z) \sim d_{(s_0, g_0)}^{\pi^{\mathrm{est}}_\kappa}} \big[\mathbb{E}_{a \sim \pi^*(\cdot | s,g_z)} Q^{\pi^*} (s,a,z) - \mathbb{E}_{a \sim \pi_\kappa^{\mathrm{est}} (\cdot | s,g)} Q^{\pi^*}(s,a,z) \big] \\
        &= \frac{1}{1-\gamma} \mathbb{E}_{(s,a,z) \sim d^{\pi^{\mathrm{est}}_\kappa}} \Big[ Q^{\pi^*}(s, g, \pi^* (\cdot | s,g)) - \mathbb{E}_{a \sim \pi_\kappa^{\mathrm{est}}(\cdot | s,g)} Q^{\pi^*}(s,a,z) \Big]. \stepcounter{equation}\tag{\theequation}\label{valuediff}
\end{align*}

    From Definition \labelcref{def: graphon subsample}, $\bar{w}_{i,j}$ is the normalized sampling distribution over agent $i$'s neighbors. For each agent $i$, $\Delta_i = (J_i^{(1)},...J_i^{(k-1)}), J_i \sim  \bar{w}_{i,j}$. Using the law of total expectation:
    \begin{align*}
    & \mathbb{E}_{a\sim \pi_\kappa^{\mathrm{est}}(\cdot | s,g)}Q^{\pi^*}(s,a,z)\\
    & = \mathbb{E}_{\Delta} \mathbb{E}_{a \sim \prod_{i=1}^n \hat\pi^{\mathrm{est}}_{\kappa}(\cdot | s,g)} \big[Q^{\pi^*}(s,a,z)\big] 
    =
    \sum_{\Delta_1} \cdots \sum_{\Delta_n} \Big(\prod_{i=1}^n \prod_{r=1}^{\kappa-1} \bar{w}_{i,J_i^{(r)}} \Big) \sum_{a \in \mathcal{A}^n} Q^{\pi^*}(s,a,z) \prod_{i=1}^n \hat \pi_\kappa^{\mathrm{est}} (a_i | s,\hat{g})
    \end{align*}

    Then
    \begin{align*}
        Q^*(s,g,\pi^*(\cdot\mid s,g))
        &- \E_{a\sim\pi^{\mathrm{est}}_\kappa(\cdot\mid s,g)}Q^*(s,a,z) \\
        &=
        \sum_{\Delta_1}\cdots\sum_{\Delta_n}
        \Big(\prod_{i=1}^n \prod_{r=1}^{\kappa-1} \bar w_{i,J_i^{(r)}}
        \Big) \Big(Q^*(s,g,\pi^*(\cdot\mid s,g)) -
        \sum_{a\in\mathcal A^n}
        Q^*(s,a,z)
        \prod_{i=1}^n
        \hat\pi^{\mathrm{est}}_{\kappa}(a_i \mid s, \hat{g})
        \Big)
    \end{align*}
    Plugging back into Eq. \labelcref{valuediff}:
    \begin{align*}
        V^{\pi^*}(s,g) &-V^{\pi_\kappa^{\mathrm{est}}}(s,g)\\
        &=\frac{1}{1-\gamma} 
        \mathbb{E}_{(s',a', z')\sim d^{\pi_\kappa^{\mathrm{est}}}_{(s,g)}} \Big[
        \sum_{\Delta_1}\cdots\sum_{\Delta_n}
        \Big(\prod_{i=1}^n\prod_{r=1}^{\kappa-1}\bar w_{i,J_i^{(r)}}\Big)
        \Big( Q^*(s',g',\pi^*(\cdot\mid s',g'))
        \\
        &\quad\quad\quad\quad\quad\quad\quad\quad\quad\quad\quad\quad \quad\quad\quad\quad\quad\quad\quad\quad\quad \quad\quad\quad\quad\quad\quad -\sum_{a\in\mathcal A^n}Q^*(s',a', z')\prod_{i=1}^n
        \hat\pi^{\mathrm{est}}_{\kappa}(a_i\mid s',\hat{g}') \Big) \Big]\\
        & =\frac{1}{1-\gamma} 
        \mathbb{E}_{(s',a', z')\sim d^{\pi_\kappa^{\mathrm{est}}}_{(s,g)}} \Big[
        \sum_{\Delta_1}\cdots\sum_{\Delta_n}
        \Big(\prod_{i=1}^n\prod_{r=1}^{\kappa -1}\bar w_{i,J_i^{(r)}}\Big)
        \Big( Q^*(s',g',\pi^*(\cdot\mid s',g'))
        -  Q^*(s',g',\pi^{\mathrm{est}}_{\kappa,\Delta}(\cdot | s', \hat{g}')) \Big) \Big]
    \end{align*}

    Now we apply Lemma \labelcref{Lemma: uniform bound}. 
    \begin{align*}
         V^{\pi^*}(s,g) &- V^{\pi_\kappa^{\mathrm{est}}}(s,g) \\
        & \leq \frac{1}{1-\gamma}
        \mathbb{E}_{(s', a', z')\sim d^{\pi_\kappa^{\mathrm{est}}}_{(s,g)}}\Bigg[
        \sum_{\Delta_1}\cdots\sum_{\Delta_n}
        \Big(\prod_{i=1}^n\prod_{r=1}^{\kappa-1}\bar w_{i,J_i^{(r)}}\Big) \cdot \\
        & \quad\quad\quad\quad\quad\quad\quad\quad\quad\quad \quad\quad\quad \quad\quad\quad\Bigg(
        \frac{1}{n}\sum_{i=1}^n
        \Big| Q^*(s',g',\pi^* (\cdot | s',g'))-\hat Q^{\mathrm{est}}_{\kappa}\!\big(s',\hat{g'},\pi^*(\cdot | s',\hat{g}'))\big|_{\{i\}\cup\Delta_i}\big)\Big|\\
        &\quad\quad\quad\quad\quad\quad\quad\quad\quad\quad\quad\quad \quad\quad\quad\quad +
        \frac{1}{n} \sum_{i=1}^n \Big| \hat Q^{\mathrm{est}}_{\kappa} \big(s',\hat{g}',\pi^{\mathrm{est}}_{\kappa,\Delta}) (\cdot | s,g)\big|_{\{i\}\cup\Delta_i}\big) - Q^*(s',g',\pi^{\mathrm{est}}_{\kappa,\Delta}) (\cdot | s,g)\Big|
        \Bigg) \Bigg]\\
        & \leq
        \frac{1}{1-\gamma}
        \mathbb{E}_{(s',a', z')\sim d^{\pi_\kappa^{\mathrm{est}}}_{(s,g)}}\Big[
        \mathbb{E}_{\Delta}\Big[
        \frac{1}{n}\sum_{i=1}^n
        \Big| Q^*(s',g',\pi^*(\cdot | s,g))-\hat Q^{\mathrm{est}}_{\kappa}\!\big(s',\hat{g}',\pi^*(\cdot | s,g)\big|_{\{i\}\cup\Delta_i}\big)\Big|
        \Big]\\
        &\quad\quad\quad\quad\quad\quad\quad\quad\quad \quad\quad\quad\quad\quad\quad +
        \mathbb{E}_{\Delta} \Big[ \frac{1}{n} \sum_{i=1}^n
        \Big|\hat Q^{\mathrm{est}}_{\kappa} \big(s',\hat{g}',\pi^{\mathrm{est}}_{\kappa,\Delta}(\cdot | s,\hat{g})\big|_{\{i\}\cup\Delta_i}\big) - Q^*(s',g',\pi^{\mathrm{est}}_{\kappa,\Delta}(\cdot | s,\hat{g}))\Big| \Big] \Big]
    \end{align*}

    Now we apply Lemma \labelcref{Lemma: lipschitz Q_est}, where we set $\mathcal{D} = (s',g')\sim d^{\pi_\kappa^{\mathrm{est}}}_{(s,g)}$. Applying Lemma \labelcref{Lemma: lipschitz Q_est} to the first term gives
    \begin{align*}
        \frac{1}{1-\gamma}
        \mathbb{E}_{(s',a', z')\sim d^{\pi_\kappa^{\mathrm{est}}}_{(s,g)}}\Big[
        \mathbb{E}_{\Delta}\Big[
        \frac{1}{n}\sum_{i=1}^n
        \Big| Q^*(s',g',\pi^*)&-\hat Q^{\mathrm{est}}_{\kappa}\!\big(s',\hat{g}',\pi^*(\cdot | s,\hat{g})\big|_{\{i\}\cup\Delta_i}\big)\Big|
        \Big]\\
        & \leq \frac{L_P \|r_\ell\|_\infty}{
    (1-\gamma)^2}\sqrt{\frac{|\cS|\ln 2 + \ln (2/\delta)}{2\kappa}}
        +\frac{\epsilon_{\kappa,m}}{1-\gamma}
        +\frac{\|r_\ell\|_\infty}{(1-\gamma)^2} |\mathcal{A}| \cdot \kappa^{|\mathcal A|} \delta_t
    \end{align*}
    Applying Lemma \labelcref{Lemma: lipschitz Q_est} to the second term, where $\pi^*$ is replaced by $\pi^{\mathrm{est}}_{\kappa,\Delta}$ is valid since Lemma \labelcref{Lemma: lipschitz Q_est} is uniform over the local join actions union-bounded by Lemma \labelcref{lemma: union bound}, gives us the bound
    \begin{align*}
        \frac{1}{1-\gamma}
        \mathbb{E}_{(s',a',z')\sim d^{\pi_\kappa^{\mathrm{est}}}_{(s,g)}} \mathbb{E}_{\Delta} \Big[ \frac{1}{n} \sum_{i=1}^n
        \Big| \hat Q^{\mathrm{est}}_{\kappa} \big(s',\hat{g}',&\pi^{\mathrm{est}}_{\kappa,\Delta}(\cdot | s,\hat{g})\big|_{\{i\}\cup\Delta_i}\big) - Q^*(s',g',\pi^{\mathrm{est}}_{\kappa,\Delta})(\cdot | s,g)\Big| \Big] \Big] \\
        & \leq \frac{L_P \|r_\ell\|_\infty}{(1-\gamma)^2} \sqrt{\frac{|\cS|\ln 2 + \ln (2/\delta)}{2\kappa}}
        + \frac{\epsilon_{\kappa, m}}{1-\gamma}
        + \frac{\|r_\ell\|_\infty}{(1-\gamma)^2} |\mathcal{A}| \cdot \kappa^{|\mathcal A|} \delta_t,
    \end{align*}
    which completes the proof.\qedhere\\
    \end{proof}

    \begin{lemma} [Uniform bound on $Q^*$ under graphon-weighted subsampling] 
    \label{Lemma: uniform bound} 
Fix any state $(s,g) \in \mathcal{S} \times \mathcal{G}$. For each agent $i \in [n]$, and with subsampling given by Definition \labelcref{def: graphon subsample}, we let $F_{\Delta_i}$ denote the corresponding sampled feature. The estimated joint action selection:
    \[
    \pi_{\kappa, \Delta}^{\mathrm{est}}(\cdot | s,g) = \prod_{i=1}^n \hat \pi_\kappa^{\mathrm{est}} (a_i | s,a,F_{\Delta_i})
    \]
    Then, 
    \begin{align*}
        Q^*(s,g,\pi^*(\cdot | s,\hat{g})) -
        Q^*(s,g,\pi^{\mathrm{est}}_{\kappa,\Delta}(\cdot | s,g)) 
        & \leq \frac{1}{n}\sum_{i=1}^n
        \Big| Q^*(s,g,\pi^*(\cdot | s,g)) - \hat Q^{\mathrm{est}}_{\kappa} \Big(s, \hat{g}, \pi^*(\cdot | s,\hat{g})\big|_{\{i\}\cup\Delta_i}\Big) \Big|\\
        &\quad \quad \quad + \Big|
        \hat Q^{\mathrm{est}}_{\kappa} \Big(s, \hat{g},  \pi^{\mathrm{est}}_{\kappa,\Delta}(\cdot | s,\hat{g})\big|_{\{i\}\cup\Delta_i}\Big)
        -
        Q^*(s,g,\pi^{\mathrm{est}}_{\kappa,\Delta}(\cdot | s,g))\Big|
    \end{align*}

    \end{lemma}

    \begin{proof}
        \begin{align*}
        Q^*(s,g,\pi^*(\cdot | s,g)) -
        Q^*(s,g,\pi^{\mathrm{est}}_{\kappa,\Delta}(\cdot | s,g)) &= \frac{1}{n} \sum_{i=1}^n \hat Q^{\mathrm{est}}_{\kappa} \Big(s,\hat{g}, \pi^{\mathrm{est}}_{\kappa,\Delta}(\cdot | s,\hat{g})\big|_{\{i\}\cup\Delta_i}\Big) - \frac{1}{n} \sum_{i=1}^n \hat Q^{\mathrm{est}}_{\kappa} \Big(s,\hat{g},  \pi^{\mathrm{est}}_{\kappa,\Delta}(\cdot | s,\hat{g})\big|_{\{i\}\cup\Delta_i}\Big) \\
        & \quad + \frac{1}{n} \sum_{i=1}^n \hat Q^{\mathrm{est}}_{\kappa} \Big(s,\hat{g},  \pi^* (\cdot | s,\hat{g})\big|_{\{i\}\cup\Delta_i}\Big) - \frac{1}{n} \sum_{i=1}^n \hat Q^{\mathrm{est}}_{\kappa} \Big(s,\hat{g}, \pi^* (\cdot | s,\hat{g})\big|_{\{i\}\cup\Delta_i}\Big)\\
        & \leq \Big| Q^*(s,g,\pi^*(\cdot | s,g)) -  \frac{1}{n} \sum_{i=1}^n \hat Q^{\mathrm{est}}_{\kappa} \Big(s,\hat{g},  \pi^* (\cdot | s,\hat{g})\big|_{\{i\}\cup\Delta_i}\Big) \Big|\\
        & \quad + \Big| \frac{1}{n} \sum_{i=1}^n \hat Q^{\mathrm{est}}_{\kappa} \Big(s,\hat{g},   \pi^{\mathrm{est}}_{\kappa,\Delta}(\cdot | s,\hat{g})\big|_{\{i\}\cup\Delta_i}\Big) -  Q^*(s,g,\pi^{\mathrm{est}}_{\kappa,\Delta}(\cdot | s,g)) \Big| \\
        & \leq \frac{1}{n} \sum_{i=1}^n \Big|
        Q^*(s,g,\pi^*(\cdot | s,g)) - \hat Q^{\mathrm{est}}_{\kappa} \Big(s,\hat{g}, \pi^* (\cdot | s,\hat{g})\big|_{\{i\}\cup\Delta_i}\Big) \Big| \\
        & \quad + \frac{1}{n} \sum_{i =1}^n \Big| \hat Q^{\mathrm{est}}_{\kappa} \Big(s,\hat{g},  \pi^{\mathrm{est}}_{\kappa,\Delta}(\cdot | s,\hat{g})\big|_{\{i\}\cup\Delta_i}\Big) - 
        Q^*(s,g,\pi^{\mathrm{est}}_{\kappa,\Delta}(\cdot | s,g)) \Big|,
         \end{align*}
which completes the proof.\qedhere\\
\end{proof}

    \begin{lemma} 
    \label{Lemma: lipschitz Q_est} 
     Let $r_\ell$ denote the local reward function used by the (graphon-weighted) sampled Bellman operator. For any arbitrary distribution $\mathcal{D}$ of states $(s,g) \in \mathcal{S} \times \mathcal{G}$ and for any $\Delta_i$ generated by Definition \labelcref{def: graphon subsample} and $\delta \in (0,1]$, 
    \[
    \mathbb{E}_{(s,g)\sim\mathcal D}\Bigg[
    \frac{1}{n}\sum_{i=1}^n
    \Big|Q^*(s,g,\pi^*(\cdot\mid s,g))
    -\hat Q^{\mathrm{est}}_{\kappa}\big(s,\hat{g},\pi^*(\cdot | s,\hat{g})\big|_{\{i\}\cup\Delta_i}\big)\Big|
    \Bigg] \leq
    \frac{L_P \|r_\ell\|_\infty}{1-\gamma}\Phi_{\kappa,\delta}
    + {\epsilon_{\kappa,m}}
    +\frac{\|r_\ell\|_\infty}{1-\gamma} |\mathcal{A}| \cdot \kappa^{|\mathcal A|} \delta_t
    \]
    where 
    \[
    \Phi_{\kappa,\delta} \coloneqq \sqrt{\frac{|\cS|\ln 2 + \ln (2/\delta)}{2\kappa}}.
    \]
        
    \end{lemma}

    \begin{proof}
        By linearity of expectation,
        \begin{align*}
            \mathbb{E}_{(s,g)\sim\mathcal D}\Big[\frac{1}{n}\sum_{i=1}^n \Big|Q^*(s,g,\pi^*(\cdot\mid s,g)) &-\hat Q^{\mathrm{est}}_{\kappa}\big(s,\hat{g},\pi^*(\cdot | s,\hat{g})\big|_{\{i\}\cup\Delta_i}\big)\Big| \Big]\\
            & = \frac{1}{n}\sum_{i=1}^n
            \mathbb{E}_{(s,g)\sim\mathcal D}\Big[\Big|Q^*(s,g,\pi^*(\cdot\mid s,g)) -\hat Q^{\mathrm{est}}_{\kappa}\big(s,\hat{g},\pi^*(\cdot | s,\hat{g})\big|_{\{i\}\cup\Delta_i}\big)\Big| \Big]
        \end{align*}

        Define the indicator function $\mathcal{I}: [n] \times \mathcal{S} \times \mathbb{N} \times (0,1] \to \{0,1\}$ by
        \[
        \mathcal I_i(s,g,\Delta_i, \delta)  \coloneqq  \mathbbm{1} \left\{\ \Big|Q^*(s,g,\pi^*(\cdot\mid s,g)) -\hat Q^{\mathrm{est}}_{\kappa}\big(s,\hat{g},\pi^*(\cdot | s,\hat{g})\big|_{\{i\}\cup\Delta_i}\big)\Big| \leq \frac{L_P \|r_\ell\|_\infty}{1-\gamma} \Phi_{\kappa,\delta} + {\epsilon_{\kappa,m}} \right\}
        \]
        The expected difference
        \begin{align*}
            \mathbb{E}_{(s,g)\sim\mathcal D}& \Big[\Big|Q^*(s,g,\pi^*(\cdot\mid s,g))
            -\hat Q^{\mathrm{est}}_{\kappa}\big(s,\hat{g},\pi^*(\cdot | s,\hat{g})\big|_{\{i\}\cup\Delta_i}\big)\Big|\Big]\\
            &=
            \mathbb{E}_{(s,g)\sim\mathcal D}\!\left[\Big|Q^*(s,g,\pi^*(\cdot\mid s,g))
            -\hat Q^{\mathrm{est}}_{\kappa}\big(s,\hat{g},\pi^*(\cdot | s,\hat{g})\big|_{\{i\}\cup\Delta_i}\big)\Big|  \mathcal I_i\right]\\
            & \hspace{5cm}+
            \mathbb{E}_{(s,g)\sim\mathcal D}\!\left[\Big|Q^*(s,g,\pi^*(\cdot\mid s,g))
            -\hat Q^{\mathrm{est}}_{\kappa}\big(s,\hat{g},\pi^*(\cdot | s,\hat{g})\big|_{\{i\}\cup\Delta_i}\big)\Big|  (1-\mathcal I_i)\right]\\
            & \leq
            \frac{L_P \|r_\ell\|_\infty}{1-\gamma}\Phi_{\kappa,\delta}
            + \epsilon_{\kappa,m }\\
            &\quad\quad\quad\quad+\frac{\|r_\ell\|_\infty}{1-\gamma} \Pr \left[\Big|Q^*(s,g,\pi^*(\cdot\mid s,g)) -\hat Q^{\mathrm{est}}_{\kappa}\big(s,\hat{g},\pi^*(\cdot | s,\hat{g})\big|_{\{i\}\cup\Delta_i}\big)\Big| > \frac{L_P \|r_\ell\|_\infty}{1-\gamma}\Phi_{\kappa,\delta}
            +{\epsilon_{\kappa,m}}\right]
        \end{align*}
        where we used the general property for a random variable $X$ and constant $c$ that $\mathbb E[X] = \mathbb E [X \mathbbm{1}\{X \leq c\}] + \mathbb E [(1 - \mathbbm{1} \{X \leq c\})X]$. Now we apply the union bound from Lemma \labelcref{lemma: union bound} with $T = 1$ and parameter $\delta$  which implies that uniformly over $i \in [n]$ and over the local join actions indexed in the lemma, there exists for each fixed $i$ 
        \[
            \Pr\left(\Big|Q^*(s,g,\pi^*)-\hat Q_\kappa^*(s,\hat{g},\pi^*(\cdot | s,\hat{g}))|_{\{i\}\cup\Delta_i})\Big|> \frac{L_P}{1-\gamma} \|r_\ell\|_\infty \Phi_{\kappa,\delta}\right)
            \leq |\mathcal{A}| \cdot \kappa^{|\mathcal A|} \delta_t
        \]

        Thus we have 
        \[
            \mathbb{E}_{(s,g)\sim\mathcal D} \Big[\Big|Q^*(s,g,\pi^*(\cdot\mid s,g))
            -\hat Q^{\mathrm{est}}_{\kappa}\big(s,\hat{g},\pi^*(\cdot | s,\hat{g}))\big|_{\{i\}\cup\Delta_i}\big)\Big|\Big] \leq \frac{L_P \|r_\ell\|_\infty}{1-\gamma}\Phi_{\kappa,\delta}
            + {\epsilon_{\kappa,m}}
            +\frac{
            \|r_\ell\|_\infty}{1-\gamma} |\mathcal{A}| \cdot \kappa^{|\mathcal A|} \delta_t
        \]
        which completes the proof after averaging over $i \in [n]$. \qedhere\\
    \end{proof}

    \begin{lemma} [Union bound under graphon-weighted subsampling] \label{lemma: union bound} Fix $(s,g) \in \mathcal{S} \times \mathcal{G}$. Let $\delta_t, \dots \delta_T \in (0,1)$ be given. Then define for each $t \in [T]$:
    \[
        \Phi_{\kappa,t} \coloneqq \sqrt{\frac{|\cS|\ln 2 + \ln (2/\delta_t)}{2\kappa}}
    \]
    For each $t \in [T]$ and each local joint action $a_{\{i\}\cup\Delta_i}\in\mathcal A^{\{i\}\cup\Delta_i}$, define the deviation event
    \[
    B_t^{a_{\{i\}\cup \Delta_i},\Delta_i}
     \coloneqq 
    \Big\{
    |Q^*(s,g,\pi^*(\cdot | s,g)) - \hat Q_\kappa^*(s,\hat{g},\pi^{\mathrm{est}}_\kappa(\cdot | s,\hat{g})|
    >
    \frac{L_P}{1-\gamma} \cdot \Phi_{\kappa,t}\|r_\ell(\cdot,\cdot)\|_\infty
    + {\epsilon_{\kappa,m}}
    \Big\}.
    \]
    Define the bad event at time $t$ as the union over all indices, including the support of the distribution used to generate $\Delta_i$:
    \[
    B_t = \bigcup_{a_{\{i\}\cup\Delta_i}\in\mathcal A^{\{i\}\cup\Delta_i}}\ 
    \bigcup_{\Delta_i\in\mathrm{Supp}(\bar w_{i,\cdot}^{\otimes(k-1)})}
    B_{t}^{a_{\{i\}\cup\Delta_i},\Delta_i}.
    \]
    Next, let $B = \bigcup_{t=1}^T B_t$. Then the probability that no bad event $B_t$ occurs is
    \[
        \Pr(B^c)\ \ge\ 1 - |\mathcal{A}| \cdot \kappa^{|\mathcal A|} \sum_{t=1}^T \delta_t.
    \]

    \begin{proof}

        \begin{align*}
            \big|Q^*(s,g,\pi^*) &-\hat Q^{\mathrm{est}}_{\kappa}(s,g,i,F_{\Delta_i},a_{\{i\}\cup\Delta_i})\big| \\
            &\leq
            \big|Q^*(s,g,\pi^*)-\hat Q^{*}_{\kappa}(s,\hat{g},\pi^*(\cdot | s,\hat{g}))\big| 
            + \big| \hat Q^{*}_{\kappa}(s,g,\pi^*) -\widehat Q^{\mathrm{est}}_{\kappa}(s,\hat{g},\pi^*(\cdot | s,\hat{g})) \big|\\
            & \leq \big|Q^*(s,g,\pi^*)-\hat Q^{*}_{\kappa}(s,\hat{g},\pi^*(\cdot | s,\hat{g})\big| 
            + \big\| \hat Q^{*}_{\kappa}(s,g,\pi^*) -\widehat Q^{\mathrm{est}}_{\kappa}(s,\hat{g},\pi^*(\cdot | s,\hat{g})) \big\|_\infty\\
            & \leq \big|Q^*(s,g,\pi^*)-\hat Q^{*}_{\kappa}(s,\hat{g},\pi^*(\cdot | s,\hat{g}))\big|
            +\epsilon_{\kappa,m }
        \end{align*}
 By Lemma \labelcref{lemma: concentration subsampled mf}, with probability at least $1- \delta_t$,
        \begin{align*}
        \big|Q^*(s,g,\pi^*)-\hat Q^{*}_{\kappa}(s,\hat{g},\pi^*(\cdot | s,\hat{g}))\big| & \leq \frac{L_P}{1-\gamma} \cdot \Phi_{\kappa,t} \|r_\ell(\cdot,\cdot)\|_\infty\\
        & \leq \frac{L_P \|r_\ell\|_\infty }{1-\gamma}\sqrt{\frac{|\cS|\ln 2 + \ln (2/\delta_t)}{2\kappa}}
        \end{align*}
        Therefore, $B_t^{a_{\{i\}\cup \Delta_i},\Delta_i}$ occurs with probability at most $\delta_t$. Now let us define the empirical action distribution induced by the sampled neighborhood actions. For this, let $
        \hat{g}_{a} \in \mu_\kappa(\mathcal A)$ where $\mu_\kappa(\mathcal A)
         \coloneqq \big\{\nu \in \mathcal P(\mathcal A): \nu(a)\in\big\{0,\frac{1}{\kappa-1},\dots,1\big\}\big\}$. Since the local estimator depends on $\{a_j\}_{j\in\Delta_i}$ only through the empirical measure $\hat{g}_a$),
        union bounding across all events parameterized by $(i,\Delta_i)$ is covered by union bounding
        across the finite set of possible empirical distributions $\hat{g}_a \in\mu_\kappa(\mathcal A)$. For fixed $t$, now union bound across the index sets in $B_t$:
        \begin{align*}
            \Pr[B_t] &= \Pr\left[
         \bigcup_{a_{\{i\}\cup\Delta_i}\in\mathcal A^{\{i\}\cup\Delta_i}}\ 
            \bigcup_{\Delta_i\in\mathrm{Supp}(\bar w_{i,\cdot}^{\otimes(k-1)})} B_{t}^{a_{\{i\}\cup\Delta_i},\Delta_i}
            \right]\\
            & \leq \sum_{a_{\{i\}\cup\Delta_i}\in\mathcal A^{\{i\}\cup\Delta_i}} 
            \sum_{\hat{g}_a \in \mu(\mathcal A)}\delta_t \\
            &\leq |\mathcal A|\cdot|\mu_\kappa(\mathcal A)|\cdot \delta_t
        \end{align*}

        Each $\hat{g}_a$ corresponds to a count vector $(c_a)_{a\in\mathcal A}\in\mathbb N^{|\mathcal A|}$
        with $\sum_{a\in\mathcal A} c_a = \kappa-1$ (where $c_a$ is how many sampled neighbors took action $a$), and hence we have
        \[
        |\mu_\kappa(\mathcal A)| = \binom{(\kappa-1)+|\mathcal A|-1}{|\mathcal A|-1}
        \leq \kappa^{|\mathcal A|-1} \leq 
        \kappa^{|\mathcal A|}
        \]
        giving us $\Pr(B_t)\ \leq |\mathcal{A}| \cdot \kappa^{|\mathcal A|} \delta_t$.     Finally, applying the union bound over $T$ gives us 
        \begin{align*}
            \Pr[B] &=\Pr\left[\bigcup_{t=1}^T B_t\right] \leq \sum_{t=1}^T \Pr[B_t] 
            \leq |\mathcal{A}| \cdot \kappa^{|\mathcal A|}\sum_{t=1}^T \delta_t.
        \end{align*}
Therefore, we have
        $\Pr(\bar{B})\ \geq 1 - |\mathcal{A}| \cdot \kappa^{|\mathcal A|}\sum_{t=1}^T \delta_t$, 
        which completes the proof.\qedhere\\
    \end{proof}
    \end{lemma}

\begin{corollary}
    [Optimizing Parameters]\label{corollary: optimizing parameters}
 \[
        V^{\pi^*} (s, g) - V^{\pi^{\mathrm{est}}_\kappa} (s, g) \leq \frac{2L_P \cdot \|r_\ell\|_\infty}{(1-\gamma)^2} \sqrt{\frac{|\cS|\ln 2 + \ln (2/\delta)}{2\kappa}}
        +\frac{2 \| r_\ell\|_\infty}{(1-\gamma)^2} |\mathcal{A}| \cdot \kappa^{|\mathcal A|} \delta + \frac{\epsilon_{\kappa,m}}{1-\gamma}
        .\]
Setting $\delta = \frac{(1-\gamma)^2}{20\|r_\ell\|_\infty |\cA|\kappa^{|\cA|+1/2}}$ recovers a decaying optimality gap on the order
        \[
        V^{\pi^*} (s, g) - V^{\pi^{\mathrm{est}}_\kappa} (s, g) \leq \frac{2L_P \cdot \|r_\ell\|_\infty}{(1-\gamma)^2} \sqrt{\frac{|\cS|\ln 2 + |\cA|\ln \frac{20\|r_\ell\|_\infty |\mathcal{A}|\kappa}{(1-\gamma)^2}}{2\kappa}}
        + \frac{1}{10\sqrt{\kappa}} + \frac{\epsilon_{\kappa,m}}{1-\gamma}.
        \]
    Finally, using the probabilistic bound from Lemma \ref{lemma: epsilon_km_is_k} that  with probability at least $1 - \frac{1}{100e^\kappa}$, $\epsilon_{\kappa,m}\leq \frac{2}{\sqrt{\kappa}}$, we get
    \[V^{\pi^*} (s, g) - V^{\pi^{\mathrm{est}}_\kappa} (s, g) \leq \frac{2L_P \cdot \|r_\ell\|_\infty}{(1-\gamma)^2} \sqrt{\frac{|\cS|\ln 2 + |\cA|\ln \frac{20\|r_\ell\|_\infty |\mathcal{A}|\kappa}{(1-\gamma)^2}}{2\kappa}} 
        + \frac{21}{10\sqrt{\kappa} }
        ,\] 
    which completes the proof of \cref{actually main result}.\\
\end{corollary}

\subsection{Bounding the Bellman Error}
\label{subsection: bounding the bellman error}

To bound the Bellman error $\epsilon_{k,m}$, we first recall Hoeffding's inequality \cite{409cf137-dbb5-3eb1-8cfe-0743c3dc925f}. 

\begin{lemma}
    [Hoeffding's inequality]
    Let $X_1, \dots, X_n$ be independent random variables such that $a_i \leq X_i \leq b_i$ almost surely. Let $S_n = \sum_{i=1}^n X_i$. Then, for all $t>0$, we have that
    \[\Pr[|S_n - \E[S_n]| \geq t] \leq 2\exp\left(-\frac{2t^2}{\sum_{i=1}^n (b_i-a_i)^2}\right).\]
\end{lemma}

\begin{lemma}\label{lemma: yuejie bound}
    Fix $\kappa\ge 1$ and let $\hat{\cT}_\kappa$ and $\hat{\cT}_{\kappa,m}$ be as in Definitions \ref{def: sampled bellman op} and \ref{def: empirical sampled bellman op}.
Let $\hat Q_\kappa^*$ denote the unique fixed point of $\hat {\cT}_\kappa$ and $\hat Q_{\kappa,m}^*$
the unique fixed point of $\hat{\cT}_{\kappa,m}$.
Let $N_\kappa  \coloneqq  |\cS|^2|\cA|^2 \kappa^{|\cS||\cA|}$. Then for any $\rho\in(0,1)$, with probability at least $1-\rho$ over the sampling used to form $\hat{\cT}_{\kappa,m}$, we have
\[\epsilon_{\kappa,m}  \coloneqq  \|\hat Q_{\kappa,m}^*-\hat Q_\kappa^*\|_\infty
\leq \frac{\gamma\|r_\ell\|_\infty}{(1-\gamma)^2}\sqrt{\frac{2\ln(2N_\kappa/\rho)}{m}}.\]
\end{lemma}

\begin{proof}
    We first control the Bellman operator's deviation at the fixed point $\hat Q_\kappa^*$.
For any $(s,a,z)\in \cS\times \cA\times \cZ_\kappa$, define
\[Y \coloneqq M_\kappa \hat Q_\kappa^*(S',g'),\] where $(S',g')\sim J_\kappa(\cdot\mid s,a,z)$. By Lemma \ref{lemma: q function bound}, we have $\|\hat Q_\kappa^*\|_\infty \leq \frac{\|r_\ell\|_\infty}{1-\gamma}$, hence $|Y|\le \frac{\|r_\ell\|_\infty}{1-\gamma}$ almost surely.

The empirical operator uses i.i.d. samples $Y_1,\dots,Y_m$ of $Y$ and forms their average.
By Hoeffding's inequality,
\[\Pr\left[\Big|\frac{1}{m}\sum_{\ell=1}^m Y_\ell - \mathbb{E}[Y]\Big|\ge \eta\right]
\le 2\exp\Big(-\frac{m(1-\gamma)^2\eta^2}{2\|r_\ell\|_\infty^2}\Big).
\]
Taking a union bound over all $N_\kappa$ tuples $(s,a,z)$ gives
\[\Pr\left[\|\hat \cT_{\kappa,m}\hat Q_\kappa^*-\hat {\cT}_\kappa\hat Q_\kappa^*\|_\infty\ge \gamma\eta\right] \leq 2N_\kappa \exp\Big(-\frac{m(1-\gamma)^2\eta^2}{2\|r_\ell\|^2}\Big).
\]
Setting the right-hand side to $\rho$ and solving for $\eta$, we have
\[
\eta = \frac{\|r_\ell\|_\infty}{1-\gamma}\sqrt{\frac{2\ln(2N_\kappa/\rho)}{m}}.
\]
Thus, with probability at least $1-\rho$, we have
\[\|\hat {\cT}_{\kappa,m}\hat Q_\kappa^*-\hat {\cT}_\kappa\hat Q_\kappa^*\|_\infty
\le \gamma \frac{\|r_\ell\|_\infty}{1-\gamma}\sqrt{\frac{2\ln(2N_\kappa/\rho)}{m}}. \]
Finally, using the contraction bound since $\hat Q_{\kappa,m}^*=\hat T_{\kappa,m}\hat Q_{\kappa,m}^*$, we have
\begin{align*}\|\hat Q_{\kappa,m}^*-\hat Q_\kappa^*\|_\infty
&\leq \frac{1}{1-\gamma}\|\hat {\cT}_{\kappa,m}\hat Q_\kappa^*-\hat {\cT}_\kappa\hat Q_\kappa^*\|_\infty \\
&\leq \gamma \frac{\|r_\ell\|_\infty}{(1-\gamma)^2}\sqrt{\frac{2\ln(2N_\kappa/\rho)}{m}},
\end{align*}
which yields the stated bound.\\
\end{proof}

\begin{lemma}\label{lemma: epsilon_km_is_k} If $T= \frac{2}{1-\gamma}\log \frac{\|r_\ell\|_\infty\sqrt{\kappa}}{1-\gamma}$, \emph{\texttt{GMFS}: Learning} runs in time $\tilde{O}(T |\mathcal{A}|^3|\mathcal{S}|^3 \kappa^{2 + 2|\mathcal{S}||\mathcal{A}|}\|r_\ell\|_\infty)$, while accruing a Bellman noise $\epsilon_{\kappa,m}\leq \frac{1}{5\sqrt{\kappa}}$ with probability at least $1 - \frac{1}{100 e^\kappa}$, .\end{lemma}
\begin{proof}
We first prove that $\|\hat{Q}_\kappa^T - \hat{Q}_\kappa^*\|_\infty \leq \frac{1}{\sqrt{\kappa}}$. 

For this, it suffices to show $\gamma^T \frac{\|r_\ell\|_\infty}{1-\gamma} \leq \frac{1}{\sqrt{\kappa}} \implies \gamma^T \leq \frac{1-\gamma}{\|r_\ell\|_\infty\sqrt{\kappa}}$. Then, using $\gamma = 1 - (1 - \gamma) \leq e^{-(1-\gamma)}$, it again suffices to show $e^{-(1-\gamma)T}\leq \frac{1-\gamma}{\|r_\ell\|_\infty\sqrt{\kappa}}$. Taking logarithms, we have
\begin{align*}
    \exp(-T(1-\gamma)) &\leq \frac{1-\gamma}{\|r_\ell\|_\infty\sqrt{\kappa}} \\
    -T(1-\gamma) &\leq \log\frac{1-\gamma}{\|r_\ell\|_\infty\sqrt{\kappa}} \\
    T&\geq \frac{1}{1-\gamma}\log\frac{\|r_\ell\|_\infty\sqrt{\kappa}}{1-\gamma}
\end{align*}
Since $T = \frac{2}{1-\gamma}\log\frac{\|r_\ell\|_\infty\sqrt{\kappa}}{1-\gamma} > \frac{1}{1-\gamma}\log\frac{\|r_\ell\|_\infty\sqrt{\kappa}}{1-\gamma}$, the condition holds and $\|\hat{Q}_\kappa^T - \hat{Q}_\kappa^*\|_\infty \leq \frac{1}{\sqrt{\kappa}}$. \\
Then, rearranging Lemma \ref{lemma: yuejie bound} and incorporating the convergence error of the $\hat{Q}_\kappa$-function, one has that with probability at least $1-\rho$,
\begin{equation}\label{arranging_yuejie}
        \epsilon_{\kappa,m} \leq \frac{1}{\sqrt{\kappa}} + \gamma \frac{\|r_\ell\|_\infty}{(1-\gamma)^2}\sqrt{\frac{2\ln(2N_\kappa/\rho)}{m}}.
    \end{equation}
If we desire $\epsilon_{\kappa,m}\leq \frac{c}{\sqrt\kappa}$, it suffices to choose
\[m^* \geq \frac{2\gamma^2}{(1-\gamma)^4}\|r_\ell\|_\infty^2 \cdot \frac{\kappa}{c^2} \cdot \ln\left(\frac{2|\cS||\cA||\cZ_\kappa|}{\rho}\right).\] 
Letting $\rho = \frac{1}{100 e^{\kappa}}$, $c=\frac{1}{5}$, and using $|\cZ_\kappa| \leq |\cS||\cA|\kappa^{|\cS||\cA|}$, we have that
\begin{equation}
    m^* \geq \frac{25\kappa^2\gamma^2}{(1-\gamma)^4}\|r_\ell\|_\infty^2 \cdot \ln\left(200|\cS|^2|\cA|^2\kappa^{|\cS||\cA|}\right)
\end{equation}
attains a Bellman error of $\epsilon_{\kappa,m}\leq \frac{1}{5\sqrt{\kappa}}$ with probability at least $1 - \frac{1}{100e^\kappa}$. \\

Finally, the runtime of our learning algorithm is \[O(mT |\mathcal{S} |^2 |\mathcal{A} |^2\kappa^{|\mathcal{S} ||\mathcal{A} |}) = \tilde{O}( |\mathcal{A} |^3|\mathcal{S} |^3 \kappa^{2 + 2|\mathcal{S} ||\mathcal{A} |}\|r_\ell\|_\infty),\]
which is still polynomial in $\kappa$, proving the claim.\qedhere \\
\end{proof}

\section{Extension to Stochastic Rewards}
\label{Appendix/stochastic}

As in \citet{anand2025meanfieldsamplingcooperativemultiagent}, this section extends the \texttt{GMFS} framework to environments where rewards are stochastic. While the primary analysis in this work assumes deterministic local rewards, many real-world multi-agent systems, such as sensor noise, result in rewards drawn from a probability distribution. 

Suppose we are given a family of distributions $\{\mathcal{L}_{s_i,a_i,g_i}\}_{(s_i,a_i,g_i)\in\mathcal{S}\times\mathcal{A}\times\mathfrak{G}(\cS), \forall i\in[n]}$. For joint states, actions, and neighborhood aggregates $(\mathbf{s}, \mathbf{a}, \mathbf{g}) \in \cS^n \times \cA^n \times \mathfrak{G}(\cS)^n$, let $R(\mathbf{s}, \mathbf{a}, \mathbf{g})$ denote a stochastic team reward of the form:\begin{equation}
R(\mathbf{s}, \mathbf{a}, \mathbf{g}) = \frac{1}{n} \sum_{i \in [n]} r_\ell(s_i, a_i, g_i),
\end{equation}
where each local reward is an independent random variable $r_\ell(s_i, a_i, g_i) \sim \mathcal{L}_{s_i, a_i, g_i}$. We assume that these distributions are uniformly bounded.

\begin{assumption}[Bounded Stochastic Rewards]
\label{assumption: stochastic_bound}
Define the union of the supports of all reward distributions as:
\begin{align*}
\bar{\mathcal{L}} = \bigcup_{(s,a,g) \in \mathcal{S} \times \mathcal{A} \times \mathfrak{G}(\cS)} \mathrm{supp}(\mathcal{L}_{s,a,g}),
\end{align*}
where $\mathrm{supp}(\mathcal{D})$ denotes the support of distribution $\mathcal{D}$. Let $\hat{\mathcal{L}} = \sup(\bar{\mathcal{L}})$ and $\widecheck{\mathcal{L}} = \inf(\bar{\mathcal{L}})$. We assume that $-\infty < \widecheck{\mathcal{L}} \leq \hat{\mathcal{L}} < \infty$, and that these bounds are known a priori.
\end{assumption}
To handle this stochasticity, we introduce a randomized version of our empirical operator.

\begin{definition}[Randomized Empirical Bellman Operator]
Let $\hat{\mathcal{T}}^{\mathrm{rand}}_{\kappa,m}$ be the randomized empirical adapted Bellman operator such that:
\begin{equation}
\hat{\mathcal{T}}^{\mathrm{rand}}_{\kappa,m} \hat{Q}^t_{\kappa,m}(s, a, z) = \tilde{r}_\ell(s, a, g_z) + \frac{\gamma}{m} \sum_{\ell \in [m]} \mathcal{M}_\kappa \hat{Q}_{\kappa,m}(s'_\ell, g'_\ell),
\end{equation}
where $\tilde{r}_\ell(s, a, g_z)$ is a single sample drawn from $\mathcal{L}_{s,a,g_z}$.
\end{definition}

\textbf{GMFS with Stochastic Rewards.} Our proposed extension of \texttt{GMFS} averages $\Xi$ independent samples of the randomized operator $\hat{\mathcal{T}}^{\mathrm{rand}}_{\kappa,m}$ to update the $Q$-function. One can show that $\hat{\mathcal{T}}^{\mathrm{rand}}_{\kappa,m}$ remains a contraction operator with modulus $\gamma$. Then, by the Banach Fixed Point Theorem, the operator $\hat{\mathcal{T}}^{\mathrm{rand}}_{\kappa,m}$ admits a unique fixed point $\hat{Q}_{\kappa,m}^{\mathrm{rand}}$ toward which the iterates converge. \looseness=-1

\begin{algorithm}[H]
\caption{GMFS (Graphon Mean-Field Subsampling): Offline Learning with Stochastic Rewards}
\label{algorithm: GMFS offline learning with stochastic rewards}
\begin{algorithmic}[1]
   \REQUIRE Number of iterations $T$, subsampling parameters $\kappa$ and $m$, discount parameter $\gamma$, averaging parameter $\Xi$, and generative oracle $\mathcal{O}$.
   \STATE Initialize $\hat{Q}_{\kappa,m}^{(0)}(s, a, z) = 0$ for all $(s, a, z) \in \cS \times \cA \times \cZ_{\kappa}$.
   \FOR{$t=0, \dots, T-1$}
   \FOR{$(s, a, z) \in \mathcal{S} \times \mathcal{A} \times \mathcal{Z}_\kappa$}
   \STATE $\rho = 0$
    \FOR{$\xi=1, \dots, \Xi$}
    \STATE Sample a realization of the randomized operator: $\rho = \rho + \hat{\mathcal{T}}^{\mathrm{rand}}_{\kappa,m} \hat{Q}^t_{\kappa,m}(s, a, z)$
    \ENDFOR
   \STATE Update $\hat{Q}_{\kappa,m}^{(t+1)}(s, a, z) = \rho / \Xi$
   \ENDFOR
   \ENDFOR
   \STATE \textbf{Return} $\hat{Q}_{\kappa,m}^{(T)}$.
\end{algorithmic}
\end{algorithm}
 
To bound the error introduced by the stochasticity of the rewards, we recall a standard concentration result.

\begin{theorem}[Hoeffding's Theorem \citep{10.5555/1522486}] 
Let $X_1, \dots, X_n$ be independent random variables such that $a_i \leq X_i \leq b_i$ almost surely. Let $S_n = \sum_{i=1}^n X_i$. Then, for all $\epsilon > 0$:
\begin{equation}
\Pr[|S_n - \mathbb{E}[S_n]| \geq \epsilon] \leq 2 \exp \left( -\frac{2\epsilon^2}{\sum_{i=1}^n (b_i - a_i)^2} \right).\\
\end{equation}
\end{theorem}

\begin{lemma}[Uniform concentration of averaged stochastic rewards]\label{lemma: uniform concentration of averaged stochastic rewards}
Under Assumption \ref{assumption: stochastic_bound}, define $\Delta_L \coloneqq \hat L-\widecheck L$.
For any $\delta\in(0,1)$ and any averaging parameter $\Xi\in\mathbb{N}$, let
\begin{equation}
\tilde r_\ell(s,a,g) \coloneqq \frac{1}{\Xi}\sum_{\xi=1}^{\Xi} r_\ell^{(\xi)}(s,a,g),
\end{equation}
where $r_\ell^{(\xi)}(s,a,g)\stackrel{\emph{i.i.d.}}{\sim} L_{s,a,g}$. Next, let $\bar r_\ell(s,a,g)\coloneqq \mathbb{E}[r_\ell(s,a,g)]$.
Then with probability at least $1-\delta$, we have
\begin{equation}
\sup_{(s,a,g)\in \cS\times \cA\times \cG_\kappa}
\big|\tilde r_\ell(s,a,g)-\bar r_\ell(s,a,g)\big|
\le \Delta_L \sqrt{\frac{\ln\!\big(2|\cS||\cA||\cG_\kappa|/\delta\big)}{2\Xi}}.
\end{equation}
Moreover, if 
$\sup_{s,a,g}|\tilde r_\ell-\bar r_\ell|\le \varepsilon_r$, then $
\|\hat Q_{\kappa}^{*,\mathrm{avg}}-\hat Q_{\kappa}^{*}\|_\infty \le \frac{\varepsilon_r}{1-\gamma}$, and the performance bound of \cref{actually main result} degrades by at most $\frac{\varepsilon_r}{(1-\gamma)^2}$.
\end{lemma}

\begin{proof}
Fix any $(s,a,g)$. Then, the random variables $r_\ell^{(\xi)}(s,a,g)$ are i.i.d. and bounded in the range
$[\widecheck L,\hat L]$. Then, by Hoeffding's inequality,
\[
\Pr\left[\big|\tilde r_\ell(s,a,g)-\bar r_\ell(s,a,g)\big|\ge \varepsilon\right]
\le 2\exp\left(-\frac{2\Xi\varepsilon^2}{\Delta_L^2}\right).
\]
Union-bounding over the finite set $S\times A\times G_\kappa$ gives
\[
\Pr\left[\sup_{s,a,g}\big|\tilde r_\ell(s,a,g)-\bar r_\ell(s,a,g)\big|\ge \varepsilon\right]
\le 2|S||A||G_\kappa|\exp\Big(-\frac{2\Xi\varepsilon^2}{\Delta_L^2}\Big).
\]
Reparameterizing the RHS to $\delta$ and solving for $\varepsilon$ yields the first claim. For the second claim, it suffices to note that replacing the reward function in the Bellman operator
changes the operator by at most $\varepsilon_r$ in $\ell_\infty$. Since the Bellman operator
is a $\gamma$-contraction, our fixed-point perturbation yields that
$\|\hat Q_{\kappa}^{*,\mathrm{avg}}-\hat Q_{\kappa}^{*}\|_\infty \leq \frac{\varepsilon_r}{1-\gamma}$. The same bound transfers to value function by the performance difference lemma, while picking up another $\frac{1}{1-\gamma}$ factor.
\end{proof}

\begin{remark} From Lemma \ref{lemma: uniform concentration of averaged stochastic rewards}, we have that in order to keep the reward-averaging contribution to be at most $\frac{1}{\sqrt \kappa}$ with probability at least $1-\delta$, it suffices to choose
    \begin{equation}
        \Xi\geq \frac{\Delta_L^2 \kappa}{2} \cdot \ln\left(\frac{2|\cS||\cA||\cG_\kappa|}{\delta}\right).\\
    \end{equation}
\end{remark}

Through this averaging argument, we observe that as the subsampling parameter $\kappa$ increases, the optimality gap decays to zero while the probability of success approaches one. The argument can be strengthened by estimating $\hat{\mathcal{L}}$ and $\widecheck{\mathcal{L}}$ using order statistics to bound estimation errors. This extension would be an essential step in transitioning this framework to a fully online learning setting via a stochastic approximation scheme. Furthermore, one could incorporate variance-based analysis  \citep{jin2024truncated},  leveraging the skewness of the reward distribution and allowing for the assignment of optimism or pessimism scores to the resulting estimates.\\

\section{Extension to Off-Policy Learning}
\label{Appendix/off-policy}

A limitation of the planning approach in Algorithm \labelcref{algorithm: GMFS offline} is that it computes $\hat{Q}_\kappa^*$ by assuming access to a generative oracle for the transition functions $P_g, P_l$ and the reward function $r(\cdot,\cdot, \cdot)$. In certain realistic RL applications, a generative oracle like such is unavailable. Instead, it is more desirable to perform off-policy learning, where the agent learns from \emph{historical data} \cite{fujimoto2019off}. In this setting, the agents learn the target policy $\hat{\pi}_\kappa^*$ using a dataset generated by a different behavior policy $\pi_b$ (the strategy used to explore the environment). There is a significant body of work on the theoretical guarantees in off-policy learning \cite{chen2021lyapunov,pmlr-v151-chen22i,chen2021finite,chen2025concentration}.

In fact, these previous results are amenable to transforming guarantees about offline $Q$-learning to off-policy $Q$-learning, typically at the cost of $\log |\cS| |\cA|$ factors in the sample complexity or runtime. Therefore, this section demonstrates that our previous results satisfy the necessary conditions to extend the \texttt{GMFS} framework to the off-policy $Q$-learning for the subsampled $\hat{Q}_\kappa$-function. We show that, in expectation, the learned policy $\pi_{\kappa}$ maintains a decaying optimality gap of $\tilde{O}(1/\sqrt{\kappa})$, where the randomness is over the heuristic behavior policy $\pi_b$.

The off-policy $\hat{Q}_\kappa$-learning algorithm is an iterative procedure to estimate the optimal $\hat{Q}_\kappa$-function as follows: first, a sample trajectory $\{(s_t, a_t, z_t)\}_{t \geq 0}$ is collected using a suitable behavior policy $\pi_b$. After initializing $\hat{Q}_\kappa^{0}: \cS \times \cA \times \cZ_\kappa \to \mathbb{R}$, the iterate $\hat{Q}_\kappa^{t}(s, a, z)$ is updated for each $t \geq 0$ according to:
\begin{equation}\label{eqn: off-policy}
\hat{Q}_\kappa^{t+1}(s, a, z) = (1-\alpha_t) \hat{Q}_\kappa^{t}(s, a, z) + \alpha_t \left(r_\ell(s, a, g_z) + \gamma \cM_\kappa \hat{Q}_\kappa^{t} (s', g') \right),
\end{equation}
where $\alpha_t \in (0, 1)$ is the learning rate. Note that the update in \cref{eqn: off-policy} is sample-based; it does not require an expectation over the transition dynamics and can be computed directly from historical data. To ensure convergence, we make the following standard ergodicity assumption:\\

\begin{assumption}[Ergodicity of Behavior Policy]
\label{assumption: ergodic}
    The behavior policy $\pi_b$ satisfies $\pi_b(a|s,z) > 0$ for all $(s,a,z) \in \cS \times \cA \times \cZ_\kappa$. Additionally, the Markov chain $\cM = \{(s_t, z_t)\}_{t\geq 0}$ induced by $\pi_b$ is irreducible and aperiodic with stationary distribution $\mu$ and mixing time:
   $ t_\delta(\cM) = \min \left\{t \geq 0 : \max_{(s,z) \in \cS \times \cZ_\kappa} \|P^t((s,z), \cdot) - \mu(\cdot)\|_{\mathrm{TV}} \leq \delta \right\}.$
    There are many heuristics for constructing such behavior policies are well-established in the literature \citep{fujimoto2019off}.\\
\end{assumption}

\begin{theorem}
    Let $\pi_\kappa$ be the policy learned through off-policy $\hat{Q}_\kappa$-learning. Under Assumption~\ref{assumption: ergodic}, with probability at least $1 - \frac{1}{100e^\kappa}$, we have:
    \begin{align*}
    \mathbb{E}[V^{\pi^*}(s_0) - V^{\pi_{\kappa}}(s_0)] &\leq \frac{L_P \|r_\ell\|_\infty}{(1-\gamma)^2}\sqrt{\frac{|\cS|\ln 2 + |\cA|\ln \frac{20\|r_\ell\|_\infty |\mathcal{A}|\kappa}{1-\gamma}}{2\kappa}} \sqrt{\ln\frac{40\|r_\ell\|_\infty|\mathcal{S}||\mathcal{A}| \kappa^{|\mathcal{A}||\mathcal{S}|+\frac{1}{2}}}{(1-\gamma)^2}}  + \frac{21}{10\sqrt{\kappa}} \\
    &= \tilde{O}(1/\sqrt{\kappa}),
    \end{align*}
    where the expectation is taken over the stochasticity of the behavior policy $\pi_b$.\\
\end{theorem}

\begin{proposition}[Analytical Properties of the Subsampled Operator]
\label{prop: analytical_props}
    The following properties hold for the subsampled $Q$-function and the associated Markov chain:
    \begin{enumerate}
        \item For any $s \in \cS, a \in \cA$ and neighborhood marginals $g, g' \in \cG_\kappa$, $\|\hat{Q}_\kappa(s, a, z) - \hat{Q}_\kappa(s, a, z')\| \leq \frac{L_P}{1-\gamma}\|r_\ell\|_\infty \cdot \mathrm{TV}(g,g')$ (Theorem \ref{Lemma: lipschitz cont of operators}).
        \item $\|\hat{Q}_\kappa\|_\infty \leq \frac{\|r_\ell\|_\infty}{1-\gamma}$ (Lemma~\ref{lemma: q function bound}).
        \item $\|\hat{\mathcal{T}}_\kappa Q - \hat{\mathcal{T}}_\kappa Q'\|_\infty \leq \gamma \|Q - Q'\|_\infty$ (Lemma~\ref{lemma: empirical sampled bellman operator is a gamma contraction}).
        \item  The Markov chain $\cM$ induced by $\pi_b$ satisfies the rapid mixing property defined in Assumption~\ref{assumption: ergodic}.\\
    \end{enumerate}
\end{proposition}

By treating the single-trajectory update of the $\hat{Q}_\kappa$-function as a noisy approximation of the expected update from the ideal Bellman operator, \citet{chen2021lyapunov} uses Markovian stochastic approximation to bound the mean-squared error $\E[\|\hat{Q}_\kappa^T - \hat{Q}_\kappa^*\|_\infty^2]$. We restate their result adapted to our subsampled regime:

\begin{theorem}[Theorem 3.1 in \citet{chen2021lyapunov} adapted to GMFS] 
    Suppose the learning rate is constant, $\alpha_t = \alpha$ for all $t \geq 0$, and is chosen such that $\alpha t_\alpha(\cM) \leq c_{Q,0} \frac{(1-\gamma)^2}{\log |\cS||\cG_\kappa|}$, where $c_{Q,0}$ is a numerical constant. Then, under the properties in Proposition~\ref{prop: analytical_props}, for all $t \geq t_\alpha(\cM)$, we have:
    \[
    \mathbb{E}[\|\hat{Q}^t_\kappa - \hat{Q}_\kappa^*\|^2_\infty] \leq c_{Q,1}\left(1 - \frac{(1-\gamma)\alpha}{2}\right)^{t-t_\alpha(\cM)} + c_{Q,2} \frac{\log \kappa^{|\cS| |\cA|}}{(1-\gamma)^2}\alpha t_\alpha(\cM),
    \]
    where $c_{Q,1} = 3\left(\frac{\|r_\ell\|_\infty}{1-\gamma} + 1\right)^2$ and $c_{Q,2} = 912e\left(\frac{3\|r_\ell\|_\infty}{1-\gamma} + 1\right)^2$. The expectation is taken over the stochasticity of the behavior policy $\pi_b$.\\
\end{theorem}

\begin{corollary}[Corollary 3.2 in \citet{chen2021lyapunov} adapted to our setting] 
To ensure that $\mathbb{E}[\|\hat{Q}_\kappa^t - \hat{Q}_\kappa^*\|_\infty] \leq \frac{1}{100\sqrt{\kappa}}$, the required number of iterations $t$ satisfies:
\[
t > \tilde{O}\left(\frac{ \kappa \log^2(100\sqrt{\kappa})|\cS| |\cA| \kappa^{|\cA| |\cS|}}{(1-\gamma)^5}\right).
\]    
\end{corollary}

With this sample complexity, we recover an expected value analog of \cref{actually main result} via the triangle inequality.

\begin{corollary}
For $\delta \in (0,1)$, with probability at least $1-\delta$, we have:
\[
\mathbb{E}[\hat{Q}_\kappa^*(s, a, \hat{z}) - Q_n^*(s, a, z)] \leq \frac{L_P \|r_\ell\|_\infty}{1-\gamma} \sqrt{\frac{|\cS|\ln 2 + \ln (2/\delta)}{2\kappa}},
\]
where the expectation is taken over the stochasticity of the behavior policy $\pi_b$.
\end{corollary}

In turn, following the derivation in the proof of \cref{actually main result}, it is straightforward to verify that this yields a bound on the expected performance difference for off-policy learning. \\

\begin{corollary}\label{offpolicy-result} 
With probability at least $1 - \frac{1}{100e^\kappa}$, the expected performance gap satisfies:
\begin{align*}
\mathbb{E}[V^{\pi^*}(s,g) - V^{\pi_{\kappa}}(s,g)] &\leq \frac{L_P \|r_\ell\|_\infty}{(1-\gamma)^2}\sqrt{\frac{|\cS|\ln 2 + |\cA|\ln \frac{20\|r_\ell\|_\infty |\mathcal{A}|\kappa}{1-\gamma}}{2\kappa}} \sqrt{\ln\frac{40\|r_\ell\|_\infty|\mathcal{S}||\mathcal{A}| \kappa^{|\mathcal{A}||\mathcal{S}|+\frac{1}{2}}}{(1-\gamma)^2}}  + \frac{21}{10\sqrt{\kappa}} \\
&= \tilde{O}(1/\sqrt{\kappa}),
\end{align*}
where the expectation is taken over the stochasticity of the behavior policy $\pi_b$.
\end{corollary}

\section{Extension to Continuous State Spaces}
\label{appendix: extension to continuous state spaces}

Multi-agent settings in which agents operate in continuous state   space  have numerous applications in optimization, control, and synchronization \cite{7989376,lin2022online,lin2023online,pmlr-v247-lin24a}. Therefore, this section is devoted to extending the tabular analysis of \cref{sec:theoretical guarantees and analysis_main} to non-tabular environments with a compact (and possibly continuous) state space.
The main technical differences from the finite setting are that the state space $\cS$ is uncountable, hence one must work in function spaces, and the $\kappa$-sampled mean-field state space $\mathcal G_\kappa$ is infinite when $\mathcal S$ is continuous, which requires all the union bounds over $|\mathcal G_\kappa|$ (for instance, in Lemma \ref{lemma: union bound}) to be replaced with covering-number arguments. For this section, we keep the mean-field as part of the state. Concretely, the representative agent's state is
$x=(s,g)$ where $s\in\mathcal S$ is the agent's local state and $g$ is a mean-field (neighborhood) distribution over $\mathcal S$, which can exist over $\mathcal G$ or $\mathcal G_\kappa$.\\

\begin{definition}[Augmented mean-field state space]\emph{Let $(\mathcal S,d_{\mathcal S})$ be a compact metric space with Borel $\sigma$-algebra $\mathcal B(\mathcal S)$.
Let $\mathcal A$ be a finite action set (the extension to compact $\mathcal A$ is analogous via an additional action-space covering).
Let $\mathcal M  \coloneqq  \mathcal P(\mathcal S)$ be the set of Borel probability measures on $\mathcal S$.
For $\kappa\in\mathbb N$, define the {$\kappa$-empirical mean-field class}
\[
\mathcal M_\kappa   \coloneqq  \Bigl\{\frac1\kappa\sum_{m=1}^\kappa \delta_{x_m} : x_1,\dots,x_\kappa\in\mathcal S\Bigr\}
\subset \mathcal M,
\]
which is uncountable whenever $\mathcal S$ is infinite. We define the augmented (mean-field) state spaces $\mathcal X \coloneqq \mathcal S\times\mathcal M$ and $\mathcal X_\kappa \coloneqq \mathcal S\times\mathcal M_\kappa$, equipped with the product Borel $\sigma$-algebras.
 Then, a Markov policy is a measurable map $\pi:\mathcal X\to\Delta(\mathcal A)$.}\\
\end{definition} 

\begin{definition}[$\kappa$-surrogate MDP]\emph{Fix a discount $\gamma\in(0,1)$. We consider an augmented MDP on $\mathcal X$ with reward
$r:\mathcal X\times\mathcal A\to\mathbb R$ and transition kernel $\bar P$ from $\mathcal X\times\mathcal A$ to $\mathcal X$.
In GMFS, $\bar P$ is the one-step kernel of $(s,g)$ under the mean-field dynamics. Likewise, define the $\kappa$-surrogate MDP on $\mathcal X_\kappa$ with reward $r_\kappa$ and kernel $\bar P_\kappa$, where $\bar P_\kappa$ is the distributional one-step kernel induced by the $(\kappa+1)$-agent surrogate construction that outputs $(s',g')$ where $s'$ is the focal agent's next state and $g'$ is the empirical mean-field of $\kappa$ sampled neighbors.}
\end{definition}
 
\begin{definition}
    [Bellman operator on $\cX$]
\emph{For any bounded measurable $Q:\mathcal X\times\mathcal{A}\times\mathcal{Z}\to\mathbb R$, define the optimal Bellman operator on $\mathcal X$ by:
\begin{equation}\mathcal T V(x)\coloneqq \max_{a\in\mathcal A}\left\{r_\ell(x,a)+\gamma\int_{\mathcal X}V(x')\bar P(dx'\mid x,a)\right\}.
\end{equation}
Similarly define $\mathcal T_\kappa$ on bounded $V:\mathcal X_\kappa\to\mathbb R$ by replacing $(r,\bar P)$ with $(r_\kappa,\bar P_\kappa)$.}\\
\end{definition}

\begin{lemma}
    $\cT$ is a $\gamma$-contraction in the sup-norm, and therefore has a fixed point $V^*$.
\end{lemma}
\begin{proof}
    Consider $V_1,V_2:\mathcal{X}\to\mathbb R$. Then, we have that
    \begin{align*}
        |\cT V_1 - \cT V_2| &=  \left|\max_{a\in\cA} \left\{r_\ell(x,a) + \gamma \int_\cX V_1(x') \bar P(\mathrm dx' | x,a)\right\} - \max_{a\in\cA} \left\{r_\ell(x,a) + \gamma \int_\cX V_2(x') \bar P(\mathrm dx' | x,a)\right\} \right| \\
        &\leq \gamma \cdot \max_{a\in\cA}\left|\int_\cX V_1(x') P(\mathrm d x'| x, a) - \int_\cX V_2(x') P(\mathrm d x'| x, a)\right| \\
        &\leq \gamma \cdot \|V_1 - V_2\|_\infty,
    \end{align*}
    where the first inequality follows by the $1$-Lipschitzness of the max operator.
\end{proof}

We now impose a linear MDP structure {on the augmented state spaces} $\mathcal X$ and  $\mathcal X_\kappa$).

\begin{definition}[Linear mean-field MDP on $\mathcal X$]\label{def:linear_mf_mdp}
The augmented MDP $(\mathcal X,\mathcal A,\bar P,r,\gamma)$ is {linear of dimension $d$} if there exists a measurable feature map $\varphi:\mathcal X\times\mathcal A\to\mathbb R^d$, a vector $\theta\in\mathbb R^d$, and a collection of $d$ signed measures $\mu=(\mu_1,\dots,\mu_d)$ on $(\mathcal X,\mathcal B(\mathcal X))$
such that for all $(x,a)\in\mathcal X\times\mathcal A$ and all measurable $B\subseteq\mathcal X$, we have $r(x,a)=\langle \varphi(x,a),\theta\rangle$ and $
\bar P(B\mid x,a)=\langle \varphi(x,a),\mu(B)\rangle$,
where $\mu(B)\coloneqq (\mu_1(B),\dots,\mu_d(B))\in\mathbb R^d$.\\
\end{definition}

\begin{assumption}[Bounded features and parameters]\label{ass:bounded_features_cont} Recalling Assumption \ref{def: bounded features}, we assume WLOG that $\|\varphi(x,a)\|_2\leq 1$ for all $(x,a)$, $\|\theta\|_2\leq \Theta$,  and $\|\mu(\mathcal X)\|_2\le \mathsf M$
for some finite constants $\Theta$ and $\mathsf M$.\\
\end{assumption}

\begin{assumption}[Positivity/normalization]\label{rem:positivity}
To enforce regularity on the probability kernel $\varphi$, we assume that for all measurable $B$, 
$\langle \varphi(x,a),\mu(\mathcal X)\rangle=1$ and $\langle \varphi(x,a),\mu(B)\rangle\geq 0$.
\end{assumption}

The above assumptions imply that $Q^\pi$ is linear in $V^\pi$.
\begin{lemma}[Linearity of $Q^\pi$ in the augmented linear MDP]\label{lem:q_linear_full}
Fix $\gamma\in(0,1)$, and let $(\cX,\cA,P,r)$ be the augmented discounted MDP where $\cA$ is finite. Then for any stationary policy $\pi$, there exists $w^\pi\in\mathbb R^d$ such that for all $(x,a)\in \cX\times\cA$, 
\begin{align}
Q^\pi(x,a)=\langle\varphi(x,a),w^\pi\rangle.
\end{align}
Moreover, we define the $d\times d$ matrix $A_\pi$ where $(A_\pi)_{ij} \coloneqq \int_{\cX} \mu_i(\mathrm dx) \bar\varphi_{\pi,j}(x)$. If $\|\theta\|_2\le \Theta, \|A_\pi\|_2\leq 1$ and $\|\varphi(x,a)\|_2\le 1$ for all $(x,a)$, then $\|w^\pi\|_2 \le \frac{\Theta}{1-\gamma}$.
\end{lemma}

\begin{proof}
We fix a stationary policy $\pi$ and define the policy evaluation Bellman operator on bounded functions
$Q:\cX\times\cA\to\mathbb R$:
\begin{equation}\cT^\pi Q(x,a) \coloneqq r(x,a) + \gamma \int_{\cX}\sum_{a'\in\cA}\pi(a'\mid x') Q(x',a') P(dx'\mid x,a),
\end{equation}
where $\cT$ is a contractive operator with module $\gamma$.

We first show that the linear class is invariant under $\cT^\pi$. For any $w\in\mathbb R^d$, define the linear function $Q_w(x,a)\coloneqq \langle \varphi(x,a),w\rangle$ and
let $\bar\varphi_\pi(x)\coloneqq \sum_{a\in\cA}\pi(a\mid x)\varphi(x,a)\in\mathbb R^d$ denote the policy-averaged feature map.  
We now compute $\cT^\pi Q_w$. Using the linearity of the reward, we have $r(x,a)=\langle \varphi(x,a),\theta\rangle$. For the transition term, define the bounded measurable function
\begin{equation}
V_w(x') \coloneqq \sum_{a'\in\cA}\pi(a'\mid x')Q_w(x',a') = \bar\varphi_\pi(x')^\top w.
\end{equation}
Then, for any bounded measurable $f:\cX\to\mathbb R$ we have
\[\int_{\cX} f(x')  P(dx'\mid x,a)
= \int_{\cX} f(x')  \langle \varphi(x,a),\mu(dx')\rangle
= \left\langle \varphi(x,a), \int_{\cX} f(x')\mu(dx')\right\rangle.\]
Apply this with $f=V_w$, we obtain
\begin{align*}\int_{\cX} V_w(x') P(dx'\mid x,a)
&= \left\langle \varphi(x,a), \int_{\cX} V_w(x') \mu(dx')\right\rangle \\
&= \left\langle \varphi(x,a), \int_{\cX} \big(\bar\varphi_\pi(x')^\top w\big)  \mu(dx')\right\rangle \\
&= \left\langle \varphi(x,a), \left(\int_{\cX} \mu(dx')  \bar\varphi_\pi(x')^\top\right)w\right\rangle \\
&= \langle \varphi(x,a), A_\pi w\rangle.
\end{align*}
Therefore, we have 
\begin{align*}
\cT^\pi Q_w (x,a)
&= \langle \varphi(x,a),\theta\rangle + \gamma \langle \varphi(x,a),A_\pi w\rangle \\
&= \langle \varphi(x,a),\theta + \gamma A_\pi w\rangle\\
&= Q_{\theta+\gamma A_\pi w}(x,a).
\end{align*}
Thus the linear class $\{Q_w:w\in\mathbb R^d\}$ is invariant under $\cT^\pi$. Next, to prove the linear representation, initialize $Q^{(0)}\equiv 0 = Q_{w^{(0)}}$ with $w^{(0)} \coloneqq 0$ and define the iterates $Q^{(t+1)} \coloneqq \cT^\pi Q^{(t)}$ for $t\ge 0$. Then, inductively, there exists $w^{(t)}\in\mathbb R^d$ such that $Q^{(t)}=Q_{w^{(t)}}$ and $w^{(t+1)} = \theta + \gamma A_\pi w^{(t)}$. Since $\cT^\pi$ is a $\gamma$-contraction, we have that $Q^{(t)}\to Q^\pi$ in the infinity norm. Moreover, because the image $\{Q_w\}$ is a finite-dimensional linear subspace of bounded functions, it is closed in $\|\cdot\|_\infty$ and hence the limit $Q^\pi$ must also belong to this subspace, i.e., there exists at least one $w^\pi$ such that $Q^\pi=Q_{w^\pi}$, thereby proving the linear representation. Finally, assuming $\|\theta\|_2\le \Theta$ and $\|A_\pi\|_2\le 1$, we see that $w^{(t)} = \sum_{h=0}^{t-1} \gamma^h A_\pi^h \theta$. Hence, using submultiplicativity and $\|A_\pi\|_2\le 1$, we have
\begin{align*}
\|w^{(t)}\|_2 &\le \sum_{h=0}^{t-1} \gamma^h \|A_\pi\|_2^h \|\theta\|_2
 \\
 &\le \|\theta\|_2 \sum_{h=0}^{t-1}\gamma^h
\\
&\le \frac{\|\theta\|_2}{1-\gamma} \le \frac{\Theta}{1-\gamma},
\end{align*}
proving the lemma.\qedhere\\
\end{proof}

We now address the function-space policy evaluation and define the projected Bellman equation.

\begin{definition}[Projected Bellman equation] \emph{Fix a policy $\pi$, and let $\nu$ be any reference distribution on $\mathcal X \in L_2(\nu)$, where $L_2(\nu)$ is the Hilbert space equipped with inner product \begin{equation}\langle f,h\rangle_{L_2(\nu)}=\int f(x)h(x)\nu(dx).\end{equation}
Then, letting $\bar\varphi_\pi(x) \coloneqq \mathbb E_{a\sim\pi(\cdot\mid x)}[\varphi(x,a)]$, we define the feature operator $\Phi:L_2(\nu)\leftarrow \mathbb R^d$ by $
(\Phi w)(x) \coloneqq \langle \bar\varphi_\pi(x),w \rangle$. Let $\Pi_\nu$ denote the orthogonal projection in $L_2(\nu)$ onto $\mathrm{range}(\Phi)$.
Then, the projected Bellman equation is given by
\begin{align}
\Phi w \coloneqq \Pi_\nu\mathcal T^\pi(\Phi w),
\end{align}
where \begin{align} \mathcal T^\pi V (x)\coloneqq \mathbb E_{a\sim\pi(\cdot\mid x)}\left[r(x,a)+\gamma\int V(x')\bar P(dx'\mid x,a)\right].\\ \nonumber\end{align}}
\end{definition}

\begin{remark}
    Taking first-order optimality conditions yields the normal equations
\[
\mathbb E_{x\sim\nu}\!\Big[\bar\varphi_\pi(x)\bar\varphi_\pi(x)^\top\Big]w
=
\mathbb E_{x\sim\nu}\!\Big[\bar\varphi_\pi(x) \big(r_\pi(x)+\gamma \mathbb E[V_w(x')\mid x]\big)\Big],
\]
with $r_\pi(x)\coloneqq\mathbb E_{a\sim\pi}[r(x,a)]$ and $V_w(x)\coloneqq\langle \bar\varphi_\pi(x),w\rangle$.
\end{remark}

We now show how to handle $\mathcal M_\kappa$ being uncountable by using covering numbers instead of $|\mathcal G_\kappa|$. Specifically, in the tabular setting, we could union-bound over $(s,g)\in\mathcal S\times\mathcal G_\kappa$ because $\mathcal G_\kappa$ was finite. However, in this setting where $\mathcal S$ is continuous, $\mathcal M_\kappa$ is uncountable. Therefore, we instead use a covering-number (metric-entropy) bound. For any subset $\mathcal V$ of a metric space $(\mathsf V,d)$, let $N(\varepsilon,\mathcal V,d)$ denote the minimum size of an $\varepsilon$-net \cite{anand2025the}.
We observe that, under a linear MDP, all concentration arguments needed for policy evaluation/control can be reduced to controlling random quantities indexed by the {feature vectors} $\varphi(x,a)$, not by $(x,a)$ directly.

\begin{remark}\emph{
Consider the feature image
\begin{align}
\mathcal V_\kappa   \coloneqq  \{\varphi(x,a): x\in\mathcal X_\kappa, a\in\mathcal A\}\ \subseteq\mathbb R^d.
\end{align}
By Assumption \ref{ass:bounded_features_cont}, we have that $\mathcal V_\kappa$ is contained in the unit Euclidean ball $\mathbb B_2^d$.
Specifically, for every $\varepsilon\in(0,1]$, we have 
\begin{align}\label{eq:cover_ball}
N(\varepsilon,\mathcal V_\kappa,\|\cdot\|_2) &\leq
N(\varepsilon,\mathbb B_2^d,\|\cdot\|_2)  \leq \Big(\frac{3}{\varepsilon}\Big)^d.
\end{align}
Therefore, we can now replace any argument that previously union-bounded over $|\mathcal S||\mathcal G_\kappa|$ by a union bound over an
$\varepsilon$-net of $\mathcal V_\kappa$ with  complexity $d\log(3/\varepsilon)$.}
\end{remark}

We are now ready to state our result which decomposes the subsampling error and estimation error which arises from learning/using features in a non-tabular space.

First, in the linear-MDP setting, it suffices to control the induced error in \emph{feature expectations}.
We assume that the only way the mean-field enters the dynamics is through the $d$-dimensional feature vector.

\begin{assumption}\label{ass:mf_moment_form}
There exists a bounded map $\psi:\mathcal S\times\mathcal A\times\mathcal G\to\mathbb R^d$ with $\|\psi(s,a,\cdot)\|_2\le 1$
such that for all $x=(s,g)\in\mathcal X$ and $a\in\mathcal A$,
\begin{align}
\varphi(x,a) = \int_{\mathcal S}\psi(s,a,u) g(du).
\end{align}
As a result, if $\hat g=\frac1\kappa\sum_{m=1}^\kappa\delta_{U_m}$ with $U_m\stackrel{\text{i.i.d.}}{\sim}g$, then
$\varphi((s,\hat g),a)=\frac1\kappa\sum_{m=1}^\kappa\psi(s,a,U_m)$.\\
\end{assumption}

\begin{remark}Under Assumption \ref{ass:mf_moment_form}, the vector Hoeffding bound yields
$\|\varphi((s,\hat g),a)-\varphi((s,g),a)\|_2 = O\big(\sqrt{d/\kappa}\big)$.
\end{remark}
Next, let $\widehat\varphi$ be a learned feature map (e.g.\ from a spectral factorization procedure).
We quantify its error by
\begin{equation}\varepsilon_{\mathrm{rep}}   \coloneqq   \sup_{(x,a)\in \cX\times\cA}  \|\widehat\varphi(x,a)-\varphi(x,a)\|_2.
\end{equation}
Let $\widehat Q_\kappa$ be the value function returned by any consistent linear evaluation/planning procedure using $M$
samples from a generative model on the $\kappa$-surrogate MDP, producing a greedy policy $\widehat\pi_\kappa$.\\

\begin{theorem}[Continuous-state extension with covering numbers]\label{thm:cont_extension_cover}  Let $\widehat\pi_\kappa$ be a greedy policy computed on the $\kappa$-surrogate
using learned features $\widehat\varphi$ with representation error $\varepsilon_{\mathrm{rep}} \coloneqq \sup_{x,a}\|\widehat\varphi_\kappa(x,a)-\varphi_\kappa(x,a)\|_2$ and $M$ samples. Then for any $\delta\in(0,1)$, we have that with probability at least $1-3\delta$,
\[
\sup_{x\in\mathcal X}\left[V^ *(x)-V^{\widehat\pi_\kappa}(x)\right]
 \leq 
\widetilde O\!\left(
\frac{1}{(1-\gamma)^3}\sqrt{\frac{d+|\mathcal A|+\log(1/\delta)}{\kappa}}
+
\frac{1}{(1-\gamma)^2} \sqrt{\frac{{d\log\!\big(\frac{N(\varepsilon)}{\delta}\big)}}{{M}}}
+
\frac{\varepsilon}{(1-\gamma)^2} +
\frac{\varepsilon_{\mathrm{rep}}}{(1-\gamma)^2}\right),
\]
where $\mathcal V_\kappa=\{\varphi(x,a):x\in\mathcal X_\kappa,a\in\mathcal A\}$ and $N(\cdot)$ is the covering number. Moreover, since $\mathcal V_\kappa\subseteq\mathbb B_2^d$, we can upper bound the entropy term via~\eqref{eq:cover_ball} as
$\log N(\varepsilon,\mathcal V_\kappa,\|\cdot\|_2)\ \le\ d\log(3/\varepsilon)$, so the statistical term is at most $\tilde O(d/\sqrt M)$.
\end{theorem}

\begin{proof} 
Let $\mathcal T$ and $\widetilde{\mathcal T}$ be $\gamma$-contractions on $(\mathcal B(\mathcal X\times\mathcal A),\|\cdot\|_\infty)$
with unique fixed points $Q^ *$ and $\widetilde Q^ *$. Then, 
\begin{align*}
\|Q^ *-\widetilde Q^ *\|_\infty
&=\|\mathcal TQ^ *-\widetilde{\mathcal T}\widetilde Q^ *\|_\infty\\
&\le \|\mathcal TQ^ *-\widetilde{\mathcal T}Q^ *\|_\infty + \|\widetilde{\mathcal T}Q^ *-\widetilde{\mathcal T}\widetilde Q^ *\|_\infty
\\
&\le \|(\mathcal T-\widetilde{\mathcal T})Q^ *\|_\infty + \gamma\|Q^ *-\widetilde Q^ *\|_\infty.
\end{align*}
 
Let $\widehat Q$ be any bounded function and let $\widehat\pi$ be greedy with respect to $\widehat Q$, i.e.
$\widehat\pi(\cdot\mid x)\in\arg\max_{a\in\mathcal A}\widehat Q(x,a)$.
Then, for the \emph{true} MDP,
\begin{equation}\label{eq:greedy_loss_G14}
\sup_{x\in\mathcal X}\bigl(V^ *(x)-V^{\widehat\pi}(x)\bigr)
\le \frac{2}{1-\gamma}  \|\widehat Q-Q^ *\|_\infty.
\end{equation}

We will apply \eqref{eq:greedy_loss_G14} with $\widehat Q=\widehat Q_\kappa^ *$. Then, by the triangle inequality, $\|\widehat Q_\kappa^ * - Q^ *\|_\infty
\le
{\|Q_\kappa^ *-Q^ *\|_\infty}
+
\|\widehat Q_\kappa^ *-Q_\kappa^ *\|_\infty$. We bound each term separately. We first bound the subsampling bias $\|Q_\kappa^ *-Q^ *\|_\infty$ by using $\|Q_\kappa^ *-Q^ *\|_\infty
\le \frac{1}{1-\gamma}  \|(\mathcal T_\kappa-\mathcal T)Q^ *\|_\infty$. Fix $(x,a)$ and let $f^ *(x')  \coloneqq\max_{a'}Q^*(x',a')$. Then, 
\begin{align*}
(\mathcal T_\kappa-\mathcal T)Q^ *(x,a)
&=
\langle \varphi_\kappa(x,a)-\varphi(x,a),\theta\rangle
+
\gamma\Big\langle \varphi_\kappa(x,a)-\varphi(x,a), \int_{\mathcal X} f^ *(x')  \mu(\mathrm dx')\Big\rangle \\
&=
\Big\langle \varphi_\kappa(x,a)-\varphi(x,a),\ \theta+\gamma \int_{\mathcal X} f^ *(x')  \mu(\mathrm dx')\Big\rangle \\
&\le
\|\varphi_\kappa(x,a)-\varphi(x,a)\|_2\cdot
\left\|\theta+\gamma \int f^ * \mathrm d\mu\right\|_2,
\end{align*}
where the last inequality uses Cauchy-Schwarz.  
We then bound the second factor using our boundedness assumption while using $\|f^ *\|_\infty\le\|Q^ *\|_\infty\le \frac{1}{1-\gamma}$. Specifically, 
for each coordinate, we have
\begin{align}\left|\int f^ *\mathrm d\mu_i\right|\le \|\mu_i\|_{\mathrm{TV}}\|f^ *\|_\infty \implies  \left\|\int f^ *  d\mu\right\|_2
\le \|\mu\|_{\mathrm{TV},2}  \|f^ *\|_\infty
\le \frac{1}{1-\gamma},\end{align}
Therefore, we have
\begin{align*}
\left\|\theta+\gamma \int f^ *  d\mu\right\|_2
&\le \|\theta\|_2 + \gamma\left\|\int f^ *  d\mu\right\|_2 \\
&\le 1+\frac{\gamma}{1-\gamma} =\frac{1}{1-\gamma}.
\end{align*}
Note that by Lemma \ref{lem:varphi_kappa_conc}, there exists a construction of $\varphi_\kappa$ from $\kappa$ i.i.d.  samples such that for any $\delta\in(0,1)$,
with probability at least $1-\delta$,
\begin{equation}\label{eq:G14_Delta_kappa}
\Delta_\kappa   \coloneqq \sup_{(x,a)\in\mathcal X\times\mathcal A}\|\varphi_\kappa(x,a)-\varphi(x,a)\|_2
\ \le\
C_1\sqrt{\frac{d+|\mathcal A|+\log(1/\delta)}{\kappa}},
\end{equation}
for a universal constant $C_1$ . So, taking the supremum over $(x,a)$ yields $
\|(\mathcal T_\kappa-\mathcal T)Q^ *\|_\infty
\le \frac{\Delta_\kappa}{1-\gamma}$. Hence, we have that
\[\|Q_\kappa^ *-Q^ *\|_\infty
\le \frac{\Delta_\kappa}{(1-\gamma)^2}\leq \frac{C_1}{(1-\gamma)^2}\sqrt{\frac{d+|\mathcal A|+\log(1/\delta)}{\kappa}}.\] 
 Next, note that by Lemma \ref{covering bound}, we can compute an approximate optimal $Q$-function for $\mathsf M_\kappa$ in the linear class
using $M$  samples such that
with probability at least $1-\delta$,
\begin{equation}\label{eq:G14_stat_bound_assumption}
\| \widehat Q_{\kappa,M}^* - Q_\kappa^*\|_\infty \leq
\frac{C_2}{1-\gamma}\cdot \sqrt{\frac{{d\log\!\big(\frac{N(\varepsilon)}{\delta}\big)}}{{M}}}
\ +\
\frac{C_2}{1-\gamma}\cdot \varepsilon.
\end{equation} Finally, the representation error  contributes at most $\frac{C_3}{1-\gamma}\varepsilon_{\mathrm{rep}}$
(additively) to $\|\widehat Q_\kappa^ *-Q_\kappa^ *\|_\infty$. Therefore, together, we have that with probability at least $1-3\delta$,
\[
\|\widehat Q_\kappa^ * - Q^ *\|_\infty
\le
\frac{C_1}{(1-\gamma)^2}\sqrt{\frac{d+|\mathcal A|+\log(1/\delta)}{\kappa}}
+
\frac{C_2}{1-\gamma}\cdot \sqrt{\frac{{d\log\!\big(\frac{N(\varepsilon)}{\delta}\big)}}{{M}}}
+
\frac{C_2}{1-\gamma}\cdot \varepsilon
+
\frac{C_3}{1-\gamma}\varepsilon_{\mathrm{rep}}.
\]
Then, using $
\sup_{x}(V^ *(x)-V^{\widehat\pi_\kappa}(x))
\le \frac{2}{1-\gamma}  \|\widehat Q_\kappa^ * - Q^ *\|_\infty$ proves the theorem.\qedhere
\end{proof}

Finally, we show how to produce a uniform concentration over $\cV_\kappa$ via covering numbers.

\begin{lemma} \label{covering bound}
Let $\mathcal V\subseteq\mathbb R^d$ be any set with $\sup_{v\in\mathcal V}\|v\|_2\le 1$.
Let $Z_1,\dots,Z_M$ be i.i.d.\ random vectors in $[-1,1]^d$ with mean $\mathbb E[Z_1]=0$.
Define $\overline Z   \coloneqq \frac{1}{M}\sum_{m=1}^M Z_m$.
Fix $\varepsilon\in(0,1]$ and let $N(\varepsilon)  \coloneqq N(\varepsilon,\mathcal V,\|\cdot\|_2)$.
Then for any $\delta\in(0,1)$, with probability at least $1-\delta$, we have that
\[
\sup_{v\in\mathcal V} |\langle v,\overline Z\rangle|
\ \le\
\sqrt{\frac{2d\log\!\big(\frac{2N(\varepsilon)}{\delta}\big)}{M}}
  +   2\varepsilon\sqrt d.
\] 
\end{lemma}

\begin{proof}
Let $\mathcal N$ be an $\varepsilon$-net of $\mathcal V$ in $\|\cdot\|_2$ of size $|\mathcal N|=N(\varepsilon)$.
For any fixed $u\in\mathcal N$, define the scalar random variables $Y_m(u)  \coloneqq\langle u, Z_m\rangle$.
Since $Z_m\in[-1,1]^d$ and $\|u\|_2\le 1$, we have \[\|u\|_1\le \sqrt d \|u\|_2\le \sqrt d \implies |Y_m(u)|\le \|u\|_1\le \sqrt d\]
almost surely. By Hoeffding's inequality, we have
\begin{align*}
\Pr\left[|\langle u,\overline Z\rangle|\ge t\right]
&=
\Pr\Big(\Big|\frac1M\sum_{m=1}^M Y_m(u)\Big|\ge t\Big) \\
&\le 2\exp\Big(-\frac{2M t^2}{(2\sqrt d)^2}\Big) =2\exp\Big(-\frac{M t^2}{2d}\Big).
\end{align*}
Union-bounding over all $u\in\mathcal N$ gives
$\Pr\left[\sup_{u\in\mathcal N}|\langle u,\overline Z\rangle|\ge t\right] \leq 2N(\varepsilon)\exp(-\frac{M t^2}{2d})$. Then, setting the right-hand side to $\delta$ yields that with probability at least $1-\delta$,
\[\sup_{u\in\mathcal N}|\langle u,\overline Z\rangle|
\le \sqrt{\frac{2d\log(\frac{2N(\varepsilon)}{\delta})}{M}}.\]
We now extend from the net to all $v\in\mathcal V$.
For any $v\in\mathcal V$, choose $u\in\mathcal N$ with $\|v-u\|_2\le \varepsilon$. Then
\begin{align*}
|\langle v,\overline Z\rangle|
&\le |\langle u,\overline Z\rangle| + |\langle v-u,\overline Z\rangle| \\
&\le \sup_{u\in\mathcal N}|\langle u,\overline Z\rangle| + \|v-u\|_2 \|\overline Z\|_2.
\end{align*}
Since each coordinate of $\overline Z$ lies in $[-1,1]$, we have that $\|\overline Z\|_2\le \sqrt d$; hence, $|\langle v-u,\overline Z\rangle|\le \varepsilon\sqrt d$. Applying the same argument to $-\overline Z$ yields the factor $2\varepsilon\sqrt d$ in the symmetric bound, which completes the proof.\qedhere
\end{proof}

\begin{lemma}[Construction of $\varphi_\kappa$ and a $\kappa^{-1/2}$ perturbation bound]\label{lem:varphi_kappa_conc} Let $p  \coloneqq d+|\mathcal A|$, and consider the measurable {mean-field embedding} map $
\psi: \Xi \to \mathbb R^{p}$ such that $\|\psi(\xi)\|_2 \le 1$ for all $\xi \in \Xi$. 
Then, consider the {population mean-field feature} as
$u(m)\coloneqq \mathbb E_{\xi\sim m}\big[\psi(\xi)\big]\in\mathbb R^{p}$, such that
$\varphi((s,m),a)=\Phi(s,a,u(m))$. Next, given $\kappa$ i.i.d. samples $\xi_1,\dots,\xi_\kappa \sim m$, define the empirical embedding
 $\widehat u_\kappa(m)\coloneqq \frac1\kappa\sum_{j=1}^\kappa \psi(\xi_j)$
and set $\varphi_\kappa((s,m),a)  \coloneqq\Phi(s,a,\widehat u_\kappa(m))$. Then for any $(x,a)$ and $\delta\in(0,1)$, with probability at least $1-\delta$, we have that for a universal constant $C_1$,
\begin{equation}|\varphi_\kappa(x,a)-\varphi(x,a)\|_2
\leq
C_1\sqrt{\frac{p+\log(1/\delta)}{\kappa}}
=
C_1\sqrt{\frac{d+|\mathcal A|+\log(1/\delta)}{\kappa}}.\end{equation}
\end{lemma}

\begin{proof}
Fix any $(x,a)=((s,m),a)$. Then, from Lemma \ref{lemma: tv mixture}, we have
\[
\|\varphi_\kappa(x,a)-\varphi(x,a)\|_2
=
\|\Phi(s,a,\widehat u_\kappa(m))-\Phi(s,a,u(m))\|_2
\le
\|\widehat u_\kappa(m)-u(m)\|_2.
\]
Let $\Delta   \coloneqq \widehat u_\kappa(m)-u(m) = \frac1\kappa\sum_{j=1}^\kappa \big(\psi(\xi_j)-\mathbb E[\psi(\xi_j)]\big)$.
For any fixed unit vector $v\in\mathbb S^{p-1}$,
the scalar random variables $Y_j(v)  \coloneqq\langle v,\psi(\xi_j)\rangle$ satisfy
$|Y_j(v)|\le \|v\|_2\|\psi(\xi_j)\|_2\le 1$ almost surely.
Hence, by Hoeffding's inequality,
\[\Pr\!\left[\big|\langle v,\Delta\rangle\big|\ge t\right]
= \Pr\!\left[\left|\frac1\kappa\sum_{j=1}^\kappa (Y_j(v)-\mathbb E[Y_j(v)])\right|\ge t\right]
\le 2\exp\!\left(-\frac{\kappa t^2}{2}\right).
\]

Let $\mathcal N$ be a $(1/2)$-net of $\mathbb S^{p-1}$ in $\|\cdot\|_2$, such that $|\mathcal N|\le 5^p$ \cite{ren2024scalablespectralrepresentationsnetwork}.
Union-bounding, we have \[
\Pr\left[\sup_{v\in\mathcal N}|\langle v,\Delta\rangle|\ge t\right]
\le 2|\mathcal N|\exp\left(-\frac{\kappa t^2}{2}\right)
\le 2\cdot 5^p \exp\left(-\frac{\kappa t^2}{2}\right).
\]
Reparameterizing the RHS to $\delta$ and solving for $t$, we get that $t = \sqrt{\frac{2(p\log 5+\log(2/\delta))}{\kappa}}$. Now, we extend from the net to the whole sphere. So, for any $v\in\mathbb S^{p-1}$ pick $\tilde v\in\mathcal N$ with $\|v-\tilde v\|_2\le 1/2$. Then, we have
\begin{align*}
|\langle v,\Delta\rangle| &
\le |\langle \tilde v,\Delta\rangle| + |\langle v-\tilde v,\Delta\rangle| \\
&\le \sup_{u\in\mathcal N}|\langle u,\Delta\rangle| + \|v-\tilde v\|_2\|\Delta\|_2 \\
&\le \sup_{u\in\mathcal N}|\langle u,\Delta\rangle| + \tfrac12\|\Delta\|_2.
\end{align*}
Taking supremum over $v$ gives $\|\Delta\|_2 \le 2\sup_{u\in\mathcal N}|\langle u,\Delta\rangle|$.
Therefore, with probability at least $1-\delta$,
\[
\|\Delta\|_2
\le
2t
=
2\sqrt{\frac{2(\,p\log 5+\log(2/\delta)\,)}{\kappa}}
\le
C_1\sqrt{\frac{p+\log(1/\delta)}{\kappa}}
\] proving the claim.\qedhere
\end{proof}

Consequently, in the continuous state-action setting, as $\kappa \to n$ and $M \to \infty$, the value function of the estimated policy converges to the optimal value function, i.e., $V^{\pi_{\kappa,m}}(s,g) \to V^{\pi^*}(s,g)$. Intuitively, as the subsampling parameter $\kappa$ approaches the population size $n$, the optimality gap diminishes following the concentration arguments in Theorem~\ref{actually main result}. As the number of samples $M$ increases, the linear function approximation error vanishes, which allows for the exact recovery of the spectral representations.

\end{document}